\documentclass[11pt]{article}

\usepackage[margin=1in]{geometry}
\usepackage[T1]{fontenc}
\usepackage[utf8]{inputenc} % harmless; optional on modern LaTeX
\usepackage{microtype}
\usepackage{authblk}
\usepackage{amsmath,amssymb,amsfonts,amsthm}
\usepackage{bm}

\usepackage{graphicx}
\usepackage{wrapfig}
\usepackage{subcaption}
\usepackage{booktabs}
\usepackage{multirow}

\usepackage{algorithm}
\usepackage{algorithmic}

\usepackage{enumitem}

\usepackage[numbers,sort&compress]{natbib}

\usepackage{xcolor}
\usepackage[
    colorlinks=true,
    linkcolor=blue,
    citecolor=blue,
    urlcolor=blue
]{hyperref}

\usepackage{tikz}
\usetikzlibrary{
    arrows.meta,
    positioning,
    calc,
    decorations.pathreplacing,
    shapes.geometric,
    fit,
    backgrounds
}

\usepackage{pgfplots}
\pgfplotsset{compat=1.18}

\newtheorem{theorem}{Theorem}
\newtheorem{proposition}[theorem]{Proposition}
\newtheorem{lemma}[theorem]{Lemma}

\newtheorem{remark}[theorem]{Remark}
\newtheorem{definition}[theorem]{Definition}
\newtheorem{assumption}[theorem]{Assumption}

\DeclareMathOperator*{\argmin}{arg\,min}

\usepackage{xcolor}         % colors

\title{Graph Representation via Elements of Discrete Morse and Cobordism Theories}

\author[1]{Jennifer Rozenblit}
\author[3]{Chenguang Yang}
\author[3]{Yuxin Liu}
\author[3]{Yuzhou Chen}
\author[2]{Yulia R. Gel}

\affil[1]{Department of Mathematics; University of Texas, Austin}
\affil[2]{Department of Statistics; Virginia Tech}
\affil[3]{Department of Statistics; University of California, Riverside}

\begin{document}
\maketitle

\begin{abstract}
Topology is, by its nature and design, suited to structure that is nonlinear, multiscale, and nonstationary - however, within machine learning, its use remains largely confined to topological data analysis. We advocate that tools from low-dimensional topology which have remained almost exclusively contained within the domain of pure mathematics (such as Morse theory and cobordism theory) offer a strong, complementary, and yet virtually unexplored perspective on the hidden structure of data-generating processes and learning tasks built upon them. Here we introduce concepts from cobordism theory and harness tools from discrete Morse theory to improve the performance of graph diffusion models. Further, we derive theoretical guarantees and  
sufficient conditions so that under a positive decision-gap, the Morse-theoretic tools and their application for induced diffusion guidance are stable under small perturbations. Finally, we illustrate the utility of cobordism and discrete Morse theory in application to graph diffusion models for spatio-temporal graph forecasting and graph regeneration, and argue that these applications are only a small window into the part of what low-dimensional topology can offer to the field of machine learning.

\end{abstract}

% =========================
% 1. Introduction
% =========================
\section{Introduction}
\label{sec:intro}
%\YGL{Need to start from topology not diffusion as even title is now on representation learning}
In recent years, the use of theory from topology in machine learning has seen rapid growth, yet remains largely focused on tools from topological data analysis~\cite{carlson, cosk}. We argue that cobordism and discrete Morse theory (which, with the exception of few studies on topological quantum field theory~\cite{tqft,marciano2024deepneuralnetworkssemiclassical}) remain almost exclusively confined within pure mathematics, but offer strong new complementary insights into learning latent data representations. To substantiate our claim, we show the utility of discrete Morse and cobordism theories via their application to graph diffusion.

Graph diffusion has recently emerged as a new %models are a strong 
powerful paradigm for graph generation, with applications ranging from new material design to traffic forecasting~\cite{cao2024survey, liu2023generative, shou2026graph, yang2024survey_ts_st_diffusion}.
%\citep{vignac2022digress, jo2022gdss, huang2022graphdf}.
However, prevailing graph diffusion approaches often struggle to preserve {\it global structural organization} such as cycles and higher-order motifs, particularly when trained with purely local objectives~\cite{wang2025graphdiffusionmodelsaccurately}. In turn, these often overlooked fundamental (or independent)
features convey critical information on the underlying latent organization and functionality of the complex system, and %potentially 
play %a particularly 
an important role in both spatiotemporal graph learning and graph regeneration. 
For example, in transportation systems, fundamental cycles can describe independent alternative traffic loops or determine robustness of traffic flow against disruptions. %Similarly, fundamental cycles in power grid networks can describe how electricity can be transmitted along the power lines, preventing unintended power flow onto another cycle. Understanding these cycles can therefore critically reveal redundant power lines and a network’s resilience to attacks and cascading failures.
Moreover, because graphs are discrete combinatorial structures whose global topology can shift abruptly with even minor local modifications, denoising trajectories are often susceptible to what we call ``{\it topological drift}''.
Recent approaches on higher-order guided diffusion such as HOG-Diff~\cite{huang2025hogdiffhigherorderguideddiffusion} aim to address this gap by considering a coarse-to-fine curriculum based on cell-complex filtering. Although compelling, the resulting guidance is still effectively heuristic, as the diffusion process itself is still driven by 1-skeleton information, and higher-order structure is still used transiently as an ordering signal. 

%Motivated by the ideas of the higher-order guided diffusion, 
Instead, we choose to look at graph diffusion %generation 
%process 
through a prism of
intermediate structures derived from higher-order representations. Adopting a coarse-to-fine viewpoint, we ask a complementary question: \emph{what intermediate object most directly encodes topological events and guides future steps?} To answer this question, we propose a discrete Morse skeleton that takes a graph $G = (V,E)$ and {\it explicitly} identifies critical cells (i.e., cells that characterize the fundamental topological structure or changes) and the gradient-respecting connectivity structure on $G$ between them. We then use these two structures, together which form a \textit{topological skeleton}, as the conditioning signal for diffusion.

%\paragraph{Relation to higher-order guided diffusion.}
% Recent higher-order guided diffusion methods motivate graph generation through intermediate structures derived from higher-order representations. We adopt this coarse-to-fine viewpoint and ask a complementary question: \emph{what intermediate object most directly encodes topological events?} We propose a discrete Morse skeleton that explicitly identifies critical cells (cells that characterize the fundamental topological structure or changes) and gradient-respecting connectivity (i.e., how structure can be refined without backtracking). We then use this topological skeleton as the conditioning signal for diffusion.

%\paragraph{Cobordism Viewpoint: Generation as Topology-Changing Trajectories}
To track fundamental structures and their associated topological changes, we invoke \emph{cobordism theory} from low-dimensional topology (see \cite{milnor1963morse} for background). Rather than studying spaces in isolation, cobordisms study the \emph{process by which one space transforms into another}. This topological framing provides a pathway to address a key limitation of existing GDMs: their current treatment of generation as a sequence of purely \textit{local} denoising steps, without explicitly modeling how global structure emerges, merges, or dissipates over time.

Intuitively, a cobordism can be viewed as a long-exposure photograph of a person walking across a room. Rather than capturing a single frozen moment, the image records the entire journey as a blurred light trail from start to finish. In this sense, a cobordism encodes an object \emph{changing through time}; not a static snapshot, but the full trajectory itself (see Fig.~\ref{fig:human_cobord}). Where classical geometry studies objects at fixed moments, cobordisms study the transformation—the morphism—connecting one state to another. Crucially, not every point along this trajectory is equally significant, as certain moments, such as changes in direction or pauses before motion, mark meaningful transitions. Formally, a cobordism is a higher-dimensional object whose boundary consists of two lower-dimensional spaces, encoding an entire topology-changing trajectory between them (see Fig.~\ref{fig:triv-cob} for a nontrivial cobordism between two trivial spaces).

\begin{figure}[t]
    \centering
    \begin{subfigure}[t]{0.48\linewidth}
        \centering
        \includegraphics[width=\linewidth]{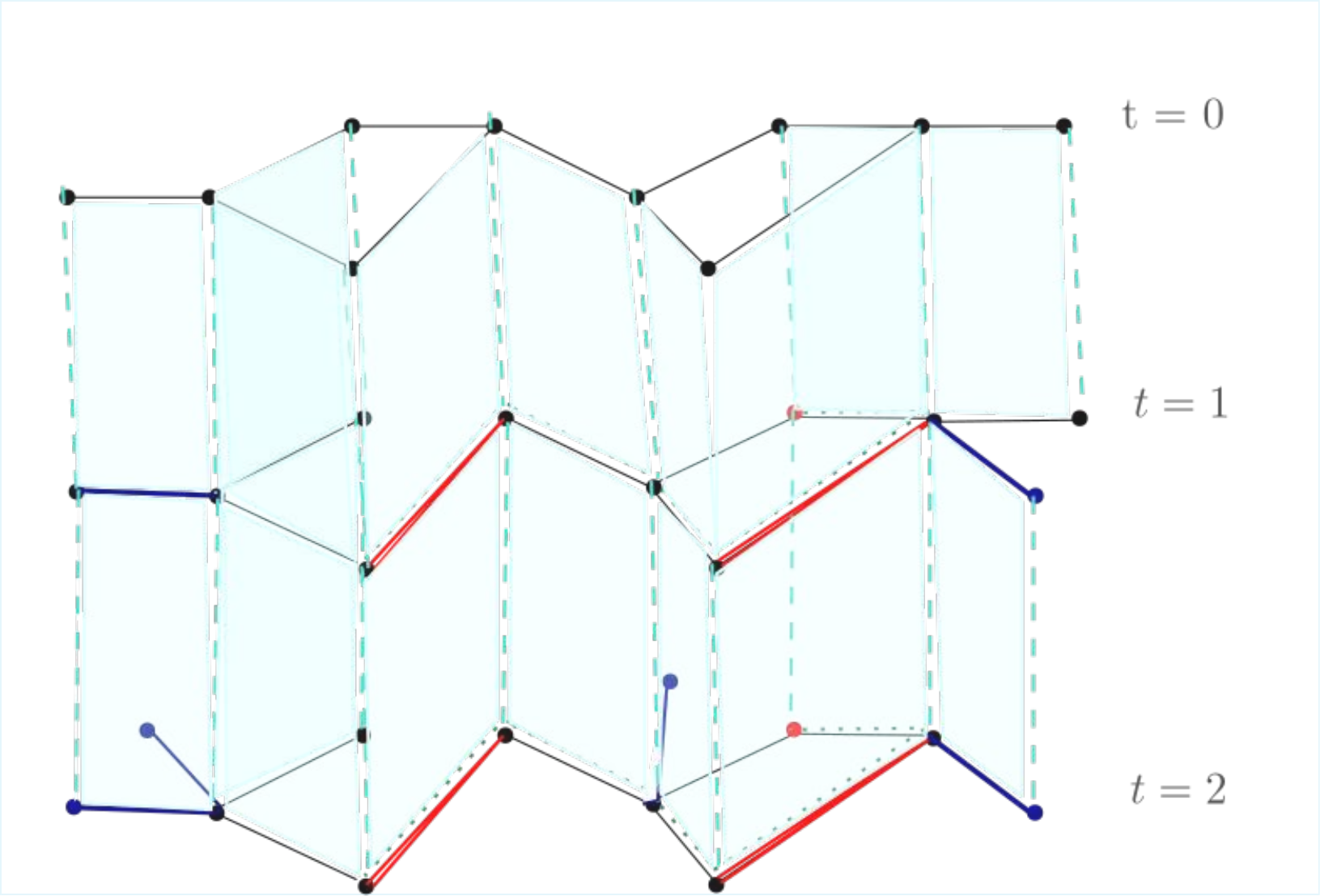} 
        \caption{\textbf{Iterative perturbation as discrete cobordism sampling.} 
        Once the Morse skeleton (critical cells and their descending manifolds) is fixed, successive perturbations generate a sequence of graphs $G_0 \to G_1 \to G_2 \to \cdots$ sharing identical Morse-theoretic structure. 
        Each perturbation modifies only edges outside the skeleton (red and dotted), preserving the topological backbone. We can think of the process as sampling from time slices of a discrete cobordism $W$: all $G_i$ are cobordant via the identity on their shared Morse complex, realizing topology-preserving graph generation as traversal through a cobordism cylinder.}
        \label{fig:cobordism-sampling}
    \end{subfigure}%
    \hfill
    \begin{subfigure}[t]{0.48\linewidth}
        \centering
        \includegraphics[width=.8\linewidth]{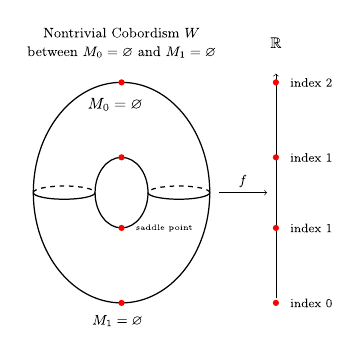}
        \caption{\textbf{Cobordism depiction.} The torus $T^2$ represents a nontrivial cobordism $W$ between $M_0 = \varnothing$ and $M_1 = \varnothing$. As we sweep through the Morse function $f$, the topology of level sets evolves: passing the index-0 critical point births a circle from the void, the first index-1 saddle splits this circle into two disjoint loops, the second saddle re-merges into one, and finally the index-2 maximum collapses the circle back to empty. This sequence of topological transitions (birth, splitting, merging, death) demonstrates how a cobordism can interpolate between boundary conditions while traversing rich intermediate topology.}
        \label{fig:triv-cob}
    \end{subfigure}
    \label{fig:cobordism-pair}
\end{figure}

We introduce this perspective into graph generation and graph diffusion; rather than viewing diffusion as a sequence of unrelated noise-removal steps, we interpret generation as a \emph{discrete cobordism} (see Fig.~\ref{fig:cobordism-sampling}), where critical cells $\mathcal{C}$ represent fundamental topological events and gradient-respecting paths $\mathcal{D}$ between those cells represent the graph's fundamental skeleton $\mathcal{S}(G) = (\mathcal{C}, \mathcal{D})$. Discrete Morse theory makes this viewpoint fully combinatorial, where the cobordism is encoded by critical simplices and the directed structure induced by a discrete gradient vector field.
% \begin{figure}[t]
%     \centering
%     \includegraphics[width=0.45\linewidth]{png pls_cropped.pdf}
%     \caption{\textbf{Iterative perturbation as discrete cobordism sampling.} 
%     Once the Morse skeleton (critical cells and their descending manifolds) is fixed, successive perturbations generate a sequence of graphs $G_0 \to G_1 \to G_2 \to \cdots$ sharing identical Morse-theoretic structure. 
%     Each perturbation modifies only edges outside the skeleton (red and dotted), preserving the topological backbone. 
%     We can think of process as sampling  from time slices of a discrete cobordism $W$: all $G_i$ are cobordant via the identity on their shared Morse complex, realizing topology-preserving graph generation as traversal through a cobordism cylinder.}
%     \label{fig:cobordism-sampling}
% \end{figure}
This cobordism-based view explains why our intermediate object must encode more than persistence or filtration order in topological data analysis~\citep{carlsson2009topology}. Persistence records \emph{when} topological features appear and disappear, but not \emph{how} they are connected by gradient flow. In particular, multiple discrete gradient vector fields may be compatible with the same filtration and induce identical persistence diagrams, while
exhibiting distinct descending-manifold structures (further shown in  Fig.~\ref{fig:ph_vs_flow} and Appx.~\ref{sec:ph_counterexamples}). This motivates the need for an intermediate object that captures both event timing and event connectivity; cobordism-aware structure retains precisely this missing information. Conditioning diffusion on such a structure therefore aligns the generative process with the true topology-changing process, rather than treating topology as an emergent side effect of local denoising.
Armed with the cobordism perspective, we introduce \textbf{Morse-Guided Diffusion (MG-Diff)}, which extracts a \emph{discrete Morse skeleton} from each graph and conditions a diffusion model on this skeleton. MG-Diff then yields an interpretable, sparse diffusion guide that explicitly encodes intra- and inter-graph topological events while also creating both constraints and freedom on admissible structural refinement steps.
% Each perturbation step to the fundamental structures of $G$ can be viewed as an elementary cobordism 
% that either preserves or changes graph topology. This perspective is analogous to how Khovanov homology \textbf{REFERNECE }tracks topology-changing 
% saddle cobordisms between link resolutions, though we do not construct an 
% explicit chain complex over the space of trajectories. Rather, we use the 
% skeleton as a conditioning signal to bias the distribution toward trajectories 
% that realize the correct critical structure. We note that diffusion trajectories 
% need not be monotone in any topological invariant; edges may appear and disappear 
% non-monotonically during denoising. The Morse skeleton nonetheless provides a 
% useful conditioning signal because it identifies the \emph{stable} topological 
% structure that persists in the final graph, even if intermediate states exhibit 
% transient features.
As Fig.~\ref{fig:cobordism-sampling} shows, each perturbation step to $G$ can be viewed as an elementary cobordism that either preserves or changes the graph's topology \cite{miller2023singularity}. %Though we do not construct an explicit chain complex over the space of trajectories, 
We use the skeleton as a conditioning signal to bias the distribution toward trajectories that realize correct critical connectivity structure. We note that diffusion trajectories need not be monotone in any topological invariant, as edges may appear and disappear non-monotonically during denoising~\cite{jo2022gdss}. %Still, 
The main utility of the Morse skeleton is providing a useful conditioning signal by identifying points both within and between the changing topology of the graph that should remain stable within the topological structure of the final graph.
In short, cobordism theory provides a mathematically principled answer to 
the question that diffusion models implicitly pose but do not explicitly 
address: \emph{what is the right way to think about a trajectory through 
state space that changes the topology of the object being generated?}
%While we do not fully exploit these structure here, 
%We conclude by noting that despite the important insights the cobordism theory can offer to graph representation, generative models and other, in general, it remains
%We view MG-Diff as a first step toward cobordism-aware generative models, including holistic approaches to topological uncertainty quantification, and we believe the broader program of importing low-dimensional topology into ML (machine learning) holds substantial promise.

\medskip

\textbf{Novelty of our contributions can be outlined as follows:}
\begin{itemize}[leftmargin=*]        
\item We introduce a cobordism theory to ML, bringing a largely untapped body of low-dimensional topology from pure mathematics to graph diffusion.

%into M, particularly for graph diffusion. 

%We introduce a cobordism theory and discrete Morse-guided approach to graph diffusion, bringing largely untapped low-dimensional topology machinery from mathematical theory  into ML, particularly for graph diffusion. 
% the machinery of cobordism theory and, more generally, elements of low-dimensional topology which remain largely unexplored outside pure mathematics, to machine learning and graph diffusion, in particular.
\item We provide an end-to-end implementation of Morse scaffold extraction and show how it integrates into a coarse-to-fine diffusion %pipeline 
for topology-aware graph generation.
% \item \item Among other theoretical results, we show that skeleton extraction is locally stable: if the relevant decision gaps are separated by a positive margin, then the resulting skeleton (and the guidance it induces) unchanged under small perturbations.
\item Among other theoretical frameworks, we model skeleton extraction as a finite-comparison deterministic \emph{extractor} yields a local stability result: when the relevant decision gaps are separated by a positive margin, the extracted skeleton (and its induced guidance) remains unchanged under small perturbations.
% \item Among other theoretical results, we derive stability guarantees by modeling skeleton extraction as a finite-comparison deterministic \emph{extractor}: under a positive decision-gap (margin) condition, the skeleton and its induced guidance are invariant to small perturbations.
\item We illustrate the introduced concepts from cobordism and discrete Morse theories in application to graph diffusion for spatio-temporal forecasting and regeneration. Our results suggest that MG-Diff  does not only yield competitive performance but shows a high promise  in terms of uncertainty quantification on all but one dataset, achieving relative CRPS gains of 6\%-16\%.  %point forecasts on all datasets but substantially outperforms all other DGMs in terms of uncertainty quantification on all but one dataset, achieving relative CRPS gains of 6\%-16\%. 

%the utility of Morse-Guided Diffusion (MG-Diff) in application to spatio-temporal graph forecasting and demonstrate that MG-Diff does not only yield competitive point forecasts on all datasets but substantially outperforms all other DGMs in terms of uncertainty quantification on all but one dataset, achieving relative CRPS gains of 6\%-16\%. 

\end{itemize}
\section{Elements of the Discrete Morse Theory}
\label{sec:prelim}
% \begin{wrapfigure}[18]{r}[0pt]{7cm}
% \vspace*{-5mm}
%   %\begin{center}
% %\begin{figure}%[t]
%     \centering
%     \includegraphics[width=1\linewidth]{Figures/testing0.pdf}
%     \caption{\textbf{MG-Diff %pipeline 
%     (left) vs.\ higher-order guided diffusion (right).} 
%     {\bf MG-Diff}: (1)~identify critical cells (red), (2)~compute descending manifolds (dotted), (3)~perturb outside these regions (blue). 
%     {\bf Higher-order diffusion} may preserve only 2-cell boundaries (teal triangles, right), missing Morse-critical structure (e.g., critical 1-cells not bounding 2-cells) that governs topological connectivity.}
%     \label{fig:pipeline-comparison}
% %\end{figure}
% \end{wrapfigure}

\begin{figure}[t]
    \centering
    \begin{minipage}{0.45\linewidth}
        \centering
        \includegraphics[width=\linewidth]{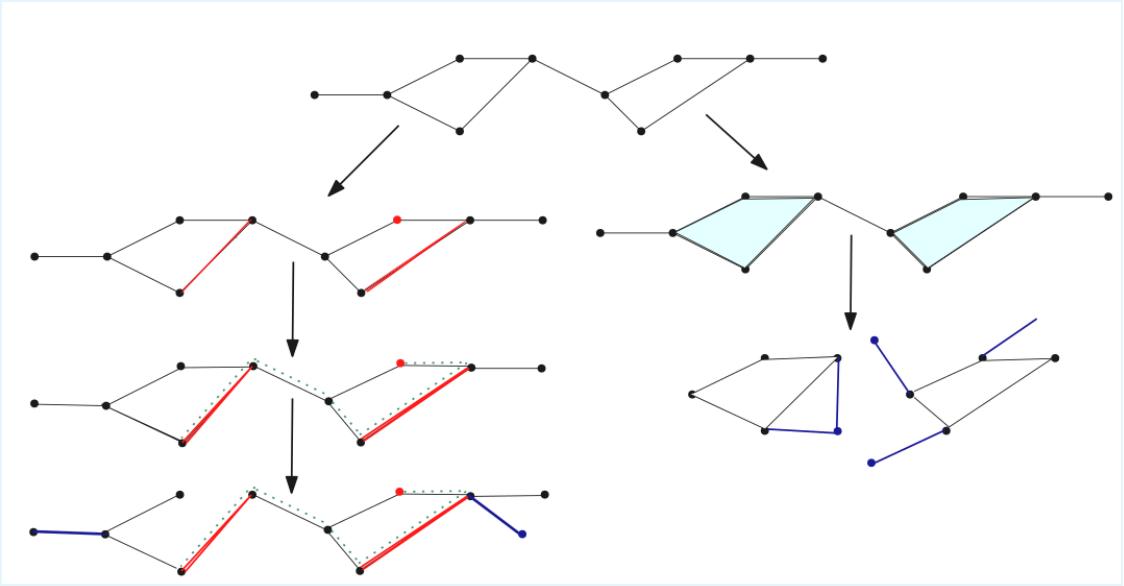}
    \end{minipage}\hfill
    \begin{minipage}{0.52\linewidth}
        \caption{\textbf{MG-Diff (left) vs.\ higher-order guided diffusion (right).} 
        MG-Diff identifies critical cells (red), computes descending manifolds (dotted), and perturbs outside these regions (blue). Higher-order diffusion preserves only 2-cell boundaries (teal triangles), missing critical 1-cells that govern topological connectivity.}
        \label{fig:pipeline-comparison}
    \end{minipage}
\end{figure}
\textbf{From Classical to Discrete Morse Theory.} Classical Morse theory studies smooth real-valued functions on manifolds and has been %successfully 
used in TDA to compute persistent homology in applications such as statistical sciences~\cite{bubenik2015statisticaltopologicaldataanalysis} and neuroscience signal analysis~\cite{giusti2015neuro}. 
% However, it fundamentally relies on smoothness assumptions that typically fail for simplicial complexes arising from graphs. In particular, the clique complex of a weighted or unweighted graph generally admits no smooth structure, preventing the direct application of classical Morse theory.
Discrete Morse theory, introduced by Forman~\cite{forman1998morse,forman2002users}, provides a combinatorial non-smooth analogue applicable to cell and simplicial complexes. Beyond encoding higher-dimensional topology, it enables substantial simplification via Morse matchings while preserving homology, often yielding significant computational savings in persistent homology for dense or large-scale complexes. \cite{mischaikow2013morse}formalized these ideas for filtered complexes, showing that persistence is preserved under filtration-compatible discrete Morse reductions.

\textbf{Morse-Theoretic Filtrations for Networks. }Persistent homology depends on a choice of filtration, since the ordering of simplices determines the birth and death of topological features~\cite{zomorodian2005computing}. For network clique complexes, naive filtrations (e.g., dimension-only) can obscure higher-order structure --- while they may reproduce heavy-tailed degree distributions, they often fail to capture heterogeneity in higher-order clique organization observed in real scale-free networks~\cite{xiao2007cliquedegree}.

We therefore adopt a \emph{discrete Morse filtration}~\cite{forman1998morse,forman2002users}, assigning real-valued weights to simplices subject to Morse consistency. Building on~\cite{kannan2019persistent}, a discrete Morse function is defined on the clique complex 
% of an 
% unweighted, undirected network
and induces a filtration via sublevel sets, with critical simplices governing topological events. Compared to early dimensional filtrations~\cite{horak2009persistent}, Morse-guided filtrations reveal finer-scale structure, provide theoretical lower bounds on the number of critical cells needed to represent homology, and achieve near-optimal reductions in practice. Although discrete Morse-based persistence has been applied in imaging~\cite{guenther2012efficient,mischaikow2013morse} and weighted simplicial complexes~\cite{Wu_2020}, its systematic use for (un)weighted complex networks remains limited. We leverage these tools within graph diffusion models to provide topology-aware guidance. 

We now briefly recall the core definitions due to Forman~\cite{forman2002users} (see Appx.~\ref{sec:discrete_morse_appendix}).

\begin{definition}[Discrete Morse function]
\label{def:MorseFun}
 Let $\Lambda$ be a simplicial complex and let $f : \Lambda \to \mathbb{R}$ assign a real value to each simplex. For a $p$-simplex $\alpha \in \Lambda^p$, define
\begin{eqnarray}
    U_f(\alpha) &=&
\{\beta \in \Lambda^{p+1} \mid \alpha < \beta,\ f(\beta) \le f(\alpha)\}\\
V_f(\alpha)
&=&
\{\gamma \in \Lambda^{p-1} \mid \gamma < \alpha,\ f(\gamma) \ge f(\alpha)\}.\nonumber
\end{eqnarray}

The function $f$ is a \emph{discrete Morse function} if
$
|U_f(\alpha)| \le 1$ and
$|V_f(\alpha)| \le 1$ for all $\alpha \in \Lambda$.
%for all $\alpha \in K$. 
\end{definition}

\begin{definition}[Critical Simplex]
\label{def:CritSimp}
A simplex $\alpha$ is \emph{critical} if
% \begin{equation}
% \label{critical_cond}
$|U_f(\alpha)| = |V_f(\alpha)| = 0.$
% \end{equation}
Equivalently, all cofaces of $\alpha$ have strictly larger function values and all faces have strictly smaller values.
\end{definition}

Critical simplices play the role of critical points in smooth Morse theory~\cite{forman1998morse}, encoding the irreducible topological features of $\Lambda$ and forming the basis for Morse reductions and topology-aware filtrations in persistent homology~\cite{kannan2019persistent}. We write that the set of all critical cells of dimension $p$ for a choice of Morse function $f$ as $C_p(f) = \{\alpha^p \in \Lambda\text{ s.t. } \alpha \text{ is critical}\}.$

\section{Methodology: Morse-Guided Diffusion}
\label{sec:method}

Higher-order guided diffusion provides a useful coarse-to-fine factorization of graph generation; for example, in protein structure modeling, it first constructs rigid secondary structure motifs ($\alpha$-helices, $\beta$-sheets as higher-order simplices), then adds flexible loops and sidechains while preserving fold topology~\cite{ingraham2019generative}.
However, in existing formulations the intermediate guidance signals do not explicitly encode \emph{intra-topological shifts}~\cite{huang2025hogdiffhigherorderguideddiffusion}, which leads to a number of fundamental limitations. First, prevailing methods often under-specify which cycles are promoted to higher-dimensional cells and how the corresponding attaching maps are chosen, leading to loss of valuable Morse-theoretic information (see for example Fig.~\ref{fig:pipeline-comparison}). Moreover, higher-order information is frequently reduced only to an ordering over edges despite diffusion itself proceeding on $1$-skeleton quantities. Finally, the presumed hierarchy can be difficult to justify on sparse or near-tree graphs, where cycle structure is weak or absent and the induced ordering may be unstable~\cite{moutsinas2021graph}. For example, in sparse transportation networks like rural road systems, minor perturbations (e.g., road closures) can drastically reorder persistent features due to the lack of redundant cycles.

To address these fundamental questions, we harness the machinery of low-dimensional topology
and introduce the notion of the {\it Morse-guided diffusion} (MG-Diff). The key idea is to replace heuristic intermediate objects with a principled Morse-theoretic skeleton $\mathcal{S}$ of the clique complex $Cl(G)$ that {\it explicitly} identifies topological junctions within and between graph states as critical cells, then to encode the influence of these junctions via a fixed representative path obtained by following a filtration-compatible discrete gradient along incidence relations.
We can then condition our graph diffusion on the resulting Morse skeleton of $G$:
$p(G) \approx p(\mathcal{S})\, p(G\mid \mathcal{S})$. While discrete Morse theory was defined on a simplicial complex, we note that the clique complex of a graph is a simplicial complex whose simplices are exactly the cliques.

\subsection{Morse Skeleton Representation of Graphs}
%\textbf{Morse skeleton representation of graphs}
To obtain a systematic \emph{compressed} representation of discrete Morse flow structure, we start from introducing the notion of \emph{Morse skeleton} on graphs. 

\begin{definition}[Morse Skeleton Representation of Graphs]
Let $\mathcal{C}(f)$ be the set of critical cells for a graph $G$ with a set of vertices $V$ ($N = |V|$), set of edges $E$ ($M = |E|$), with clique complex $Cl(G)$ and some discrete Morse
function $f$. Rather than retaining the full descending manifold of each critical cell,
which may be large and unstable, %under benign perturbations,
we associate to each $\alpha \in \mathcal{C}(f)$ a \emph{single representative
descending path} $\pi(\alpha)$ 
obtained by following a deterministic,
gradient-respecting walk in the directed Hasse diagram (see Appx.~\ref{sec:hasse} for more details.) %definition and construction
    Then, the %resulting 
    Morse skeleton (MS) is the finite set $S(G):=(\mathcal{C,D})= (\mathcal{C},\{ (\alpha, \pi(\alpha)) : \alpha \in \mathcal{C}(f) \}),
$
which encodes the essential information on both the locations of topological events (critical cells)
and a minimal, canonical description of how combinatorial flow propagates
from them.
\end{definition}
We will sometimes view the skeleton as the output of an extractor map. 
Let $\mathbb{S}$ denote the space of inputs the extractor consumes 
(the graph $G$ together with vertex scores $g$ and any auxiliary flags 
used by Algorithms~\ref{alg:discrete_morse_function}--\ref{alg:bfs}); 
write $\mathrm{Skel}: \mathbb{S} \to \mathfrak{S}$ for the deterministic 
map sending a state $s \in \mathbb{S}$ to its Morse skeleton 
$\mathcal{S}(G) \in \mathfrak{S}$, where $\mathfrak{S}$ is the (finite) 
set of skeletons of $\mathrm{Cl}(G)$. We now have machinery to explore how, given $G$, we can define a \emph{suitable discrete Morse function} $f$ and the associated sets of \emph{critical cells} $\mathcal{C}$ and an \emph{descending paths} $\mathcal{D}$. This task can be approached using the elements of Forman's discrete Morse theory~\cite{forman1998morse} (see Appx.~\ref{sec:discrete_morse_appendix} for discussion). 

%Armed with the discrete Morse theory, we describe the process and  definitions established by \cite{forman1998morse} (defined in Appx.~\ref{sec:discrete_morse_appendix}) to extract a small set of \emph{topological  events} \emph{critical cells} $\mathcal{C}$ and an associated \emph{descending path} $\mathcal{D}$, which together form what we call a \emph{Morse skeleton} $\mathcal{S}(G) = (\mathcal{C,D})$. We then subsequently use $\mathcal{S}(G)$ to guide the generative diffusion process.

%\YGL{Jen, I see that the detailed skeleton discussion is below but we need something on skeleton here (at least, basic)} \Jen{Fixed above}
%We now condition our graph diffusion on the resulting Morse skeleton of $G$: $p(G) \approx p(\mathcal{S})\, p(G\mid \mathcal{S})$.

%We now condition our graph diffusion on the resulting Morse skeleton of $G$: $p(G) \approx p(\mathcal{S})\, p(G\mid \mathcal{S})$.
% ============================================================
% Methodology: MG-Diff (Step 1: Morse-guided event extraction)
% ============================================================

\subsection{Construction of the Discrete Morse Function  
%Morse-Guided Diffusion Model
}
\label{sec:method}

%\YGL{Jen, this paragraph below seems to be somewhat repetitive. Morse skeleton is defined above and here we start to redefining it. Actually, this whole subsection is on definition of appropriate critical cells, paths and their properties.}
 
%In particular, let $G$ be a (possibly weighted) graph and  its clique complex $K=\mathrm{Cl}(G)$. Armed with the discrete Morse theory, we describe the process and  definitions established by \cite{forman1998morse} (defined in Appx.~\ref{sec:discrete_morse_appendix}) to extract a small set of \emph{topological  events} \emph{critical cells} $\mathcal{C}$ and an associated \emph{descending path} $\mathcal{D}$, which together form what we call a \emph{Morse skeleton} $\mathcal{S}(G) = (\mathcal{C,D})$. We then subsequently use $\mathcal{S}(G)$ to guide the generative diffusion process.

%\YGL{Again, skeleton is not defined yet.}
%\textbf{ADD DIFFUSION STUFF.}%\YC{For 4.1, I think we may not mention diffusion?} 
%\Jen{Fixed}

%\textbf{Constructing the Discrete Morse Function.   }
%We begin from a vertex-level seed $g:V\to\mathbb{R}$.

Note that in a clique complex, high-degree vertices participate in a combinatorially large number
of incident simplices across many dimensions. Hence, prioritizing them strongly influences the induced filtration and the subsequent
collapse structure (see~\cite{kannan2019persistent} for %motivation 
and Appx.~\ref{sec:maximality_counterexample} for more discussion). Motivated by this observation, we first choose a degree-based scoring function ~\cite{kannan2019persistent}:
\begin{equation}
g(v) \;=\; \deg_{\max}-\deg(v)+\varepsilon_v,
\label{eq:degree_seed}
\end{equation}
where $g:V\to\mathbb{R}$ is a vertex-level seed and $\deg_{\max}=\max_{u\in V}\deg(u)$.
As a randomized tie-breaker, we also consider $\epsilon_v\sim \mathrm{Unif}(0, \delta)$ (in our studies we set $\delta=0.5$); alternatively, any deterministic tie-break scheme can be employed. %is a small tie-break noise (or any deterministic tie-break scheme). \YGL{I'd generalize as 0.5 looks too peculiar.} 
We then extend the vertex score function $g$ to all higher-dimensional simplices (here, edges only) by constructing a function
$f : Cl(G) \to \mathbb{R}$ together with a discrete gradient vector field (Morse matching) on the Hasse diagram of $\Lambda$, and set $f(v) = g(v)$ for all vertices $v \in V$.

The extension is performed by a deterministic routine that assigns function values to edges while pairing incident cells so as to produce a valid discrete gradient vector field.
An order is fixed on the edge set $E = \{e_1,\dots,e_m\}$ (in practice, the deterministic order returned by the graph library).
Each simplex $\alpha \in K$ is assigned an integer flag variable $\mathrm{Flag}(\alpha)$, which counts the number of cofaces of $\alpha$ in
$U_f(\alpha) = \{\beta \supset \alpha : f(\beta) \le f(\alpha)\}$
during construction. We restrict $\mathrm{Flag}(v) \in \{0,1\}$ for vertices and initialize $\mathrm{Flag}(\alpha)=0$ for all simplices. For a given edge $e=(u,w)$, let $\gamma_0(e)$ and $\gamma_1(e)$ denote its endpoints ordered so that
$f(\gamma_0(e)) > f(\gamma_1(e))$.
If $\mathrm{Flag}(\gamma_0(e)) = 0$, we assign
$f(e) = {1}/{2}\bigl(f(\gamma_0(e)) + f(\gamma_1(e))\bigr)$,
and update $\mathrm{Flag}(\gamma_0(e)) \leftarrow 1$ (the ``average branch'').
Otherwise, we assign
$f(e) = f(\gamma_0(e)) + \eta(e)$,
where $\eta(e) \sim \mathrm{Unif}(0,0.5)$ is a small edge-noise term (the ``raise branch''). In either case, the assignment sets the flag of exactly one face of $e$ to $1$, ensuring that the discrete Morse inequalities
$|U_f(\alpha)| \le 1$ and $|V_f(\alpha)| \le 1$ hold throughout the construction. Consequently, the resulting function $f$ is a valid discrete Morse function, and the induced pairings define a discrete gradient vector field on $\Lambda$.

We enforce \emph{weak filtration-compatibility}, meaning that a pairing of cells is compatible 
with our $f$ in the sense that paired 
cofaces do not appear strictly later in the filtration, so if $\sigma\in F_t$, then 
its paired coface (if any) also lies in $F_t$. Note this is a weaker condition than strict filtration-compatibility, but is sufficient for the construction used in MG-Diff and matches our experimental  implementations.
\paragraph{What are the Suitable Choices of $g$ and $f$?}
Choosing a Morse function amounts to choosing an interpretation of what constitutes a topological event in graph $G$. Different scoring functions $g: V \to \mathbb{R}$ correspond to different critical cell sets $\mathcal{C}$, not just re-orderings of the same underlying structure, but different identifications of which cells are responsible for the graph's topological events. Importantly, critical cells resulting from a discrete Morse function are not an intrinsic property of the graph's topology alone. Which specific cells are designated critical depends entirely on the gradient flow structure induced by $f$. Two Morse functions on the same graph can agree on the total count of critical cells while disagreeing completely on their locations, in the same way that a vector space has a fixed dimension but infinitely many choices of basis. 

The choice of $f$ determines which cells constitute intra-topological shifts for a particular downstream task.
% the critical expert specializes in, and which edges the non-critical expert treats as a cancellable structure
The right way to make this choice is to ask which topological events are most functionally relevant for the task at hand. The degree-based scoring  $g$ directs gradient flow toward high-degree vertices, making the critical edges those that form cycles around the graph's major hubs. For road networks and sensor arrays, where high-degree nodes correspond to major intersections or monitoring hubs that govern macroscopic flow patterns, this alignment between topological events and functional importance is natural and well-motivated, making the degree-based function a preferred choice for sparser sensor networks. A betweenness-centrality-based scoring would instead direct flow toward bridge nodes and articulation points, defining the critical edges as those that form cycles in the periphery independent of the major bridge structure, more appropriate for robustness analysis where bridge failures are the relevant topological events. An eigenvector or PageRank-based scoring would identify as critical those cycles connecting globally influential nodes, capturing long-range diffusion dynamics rather than local hub density.

\textbf{Descending Manifold \& its Descending Path and Path restriction}
Given a critical cell $\alpha \in \mathcal{C}(f)$, its (discrete) descending manifold consists of all simplices reachable by following gradient arrows ``downward'' from $\alpha$ in the directed Hasse diagram induced by the discrete gradient vector field $\mathcal{V}$. Intuitively, it captures the region of the complex whose combinatorial flow terminates at $\alpha$. The gradient field $\mathcal{V}$ is represented as a pairing on the Hasse diagram such that $(\sigma,\tau)\in\mathcal{V}$ implies $\sigma\prec\tau$ and $f(\sigma)>f(\tau)$, with each cell appearing in at most one pair (Alg.~\ref{alg:neighbors},~\ref{alg:bfs}).
MG-Diff does \emph{not} use the full descending manifold. Instead, for each critical cell $\alpha$ that participates in diffusion guidance (notably, critical edges in the $1$-skeleton), we select a \emph{single representative path} we call the \emph{descending path}
$\pi(\alpha) = (\alpha = \sigma_0 \to \sigma_1 \to \cdots \to \sigma_m)$
where each step follows a gradient arrow or induced directed adjacency toward strictly lower $f$-value.

In the $1$-skeleton setting, the discrete gradient vector field induces a unique outgoing move at each nonterminal simplex. As a result, each critical edge deterministically follows a finite $\mathcal{V}$-path to a terminal critical cell, typically a vertex. That is, the walk alternates between (i) ``up'' moves along a gradient pairing (when present) and (ii) ``down'' moves to an unpaired face, enforcing strict decrease in $f$ at each descent. The resulting path provides MG-Diff a compact encoding of the descending manifold (see Appx.~\ref{sec:discrete_morse_appendix} for discussion on sufficiency). %\textbf{Path Selection $\&$ Stability} 
Finally, we select path $\pi(\alpha)$ using a deterministic greedy rule: among eligible faces $\sigma \subset \alpha$ with $f(\alpha) \leq f(\sigma)$, choose $\pi(\alpha) = \argmin_\sigma f(\sigma)$
(i.e., maximum drop in $f$) (Alg. \ref{alg:discrete_morse_function}), with flow graph outdegree $\leq 1$ by construction guaranteeing a unique path.  The resulting Morse skeleton is hence the event set 
$\mathcal{E}(G) = \{(\alpha, \pi(\alpha)) : \alpha \in \mathcal{C}(f)\}$.

% \textbf{Path selection and stability.}
% We select $\pi(\alpha)$ using a deterministic rule (e.g., greedy descent by maximum drop in $f$, or shortest gradient-consistent path in the $1$-skeleton), with tie-breaking to ensure reproducibility. The choice of rule affects only the representative encoding, not the underlying event. The resulting Morse skeleton is therefore the event set
% $\mathcal{E}(G) = \{(\alpha, \pi(\alpha)) : \alpha \in \mathcal{C}(f)\}$.
% In \missing{Appendix~\textbf{[REF]}}, we prove a margin stability result: if vertex scores are perturbed by less than half the minimum separation between adjacent scores, the skeleton $\mathcal{E}(G)$ remains unchanged. We also characterize failure modes under qualitatively different scoring rules \missing{(Appendix~\textbf{[REF]})}.

\textbf{Why are critical cells indicative of intra-topological change?}
Being critical is equivalent to remaining unpaired in the discrete gradient vector field $\mathcal{V}$ induced by the Morse function%, i.e., it participates in no gradient arrow 
(see Appx.~\ref{sec:dis_gvf}). Critical cells are then the irreducible obstructions to collapse, as \textit{every non-critical cell is paired and can be eliminated by a sequence of elementary cancellations that preserve homotopy type}~\cite{forman1998morse,forman2002users}. Hence, $\mathcal{C}(f)$ provides a principled, topology-aware ``skeleton'' of $Cl(G)$ that concentrates the meaningful combinatorial
changes across the filtration.

\subsection{Theoretical Guarantees of Morse Skeleton Guidance}
\label{sec:mainstablemma}

We now prove margin stability guarantees, namely, if vertex scores are perturbed by less than half the minimum separation between adjacent scores, the skeleton $\mathcal{S}(G)$ remains unchanged. We also characterize failure modes under qualitatively different scoring rules (see Appx.\ref{sec:stability_skeleton_guidance}.) 

% Formally, let $\mathbb{S}$ be the space of \emph{states} read by the extractor (graph structure plus any scores/flags it consumes), and let $s\in\mathbb{S}$ denote a state. Let $\mathrm{Skel}:\mathbb{S}\to Cl(G)$ be  a deterministic skeleton extractor. \YGL{These notations need explanations/intution} %$\Psi:\mathcal{K}\times\mathcal{X}\to\mathcal{H}$ be a deterministic guidance map, $X^h\in\mathcal{X}$ be a fixed historical input, $\bar H(s):=\Psi(\mathrm{Skel}(s),X^h)$, $\epsilon_\theta(X_k,\bar H,k)$ be a denoiser, and $p_\theta(\cdot\mid X_k,\bar H,k)$ be reverse-step kernel with covariance independent of $s$.

\begin{proposition}[Skeletal Informativeness]
\label{prop:skel-strict-refinement}
Let $G = (V, E)$ be a graph equipped with a discrete Morse function $f$
and let $\mathcal{S}(G) = (E^{\mathrm{crit}}, E^{\mathrm{pair}})$ be its Morse skeleton with its associated gradient vector field $\mathcal{V}$.  Let
$\mathrm{Dgm}(G)$ be persistence diagram $\mathrm{Dgm}(G)$ of $G$ under sublevel
filtration of~$f$.  Then:
\begin{enumerate}[label=\textup{(\alph*)},leftmargin=*]
  \item \textbf{Strict refinement.}
        $\mathcal{S}(G)$ determines $\mathrm{Dgm}(G)$, but the
        converse fails: there exist graphs $G, G'$ with
        $\mathrm{Dgm}(G) = \mathrm{Dgm}(G')$ and $\mathcal{S}(G) \neq \mathcal{S}(G')$.  In
        particular, the descending-manifold structure that defines
        our edge partition is not recoverable from $\mathrm{Dgm}(G)$.
  \item \textbf{Margin-based stability.}
        Under the no-ties condition, $\mathcal{S}(G)$ enjoys a
        \emph{computable} stability radius
          $r(s) = \frac{m(s)}{L_\Delta} > 0,$
        where $m(s)$ is the minimum decision-gap margin and
        $L_\Delta$ is the Lipschitz constant of the decision gaps:
        every data perturbation of magnitude strictly less than
        $r(s)$ leaves $\mathcal{S}(G)$ unchanged
        (See Appx. \ref{sec:stability_skeleton_guidance}).
\end{enumerate}
\end{proposition}

Consequently, $\mathcal{S}(G)$ is a strict refinement of $\mathrm{Dgm}(G)$
that inherits the same form of margin-based stability classical
persistence enjoys. Proposition~\ref{prop:skel-strict-refinement}\,(b) implies every input that admits PD-stability with
margin $m(s)$ admits skeleton-stability at the same radius scaled
by the kernel-determined constant $L_\Delta$, and the partition
information that PD discards remains protected within that ball.

\begin{definition}[Decision margin]
\label{def:decision_margin}
For a state $s$ in our extractor map $\mathbb{S}$, under no-ties condition, the \emph{decision margin} is $m(s) := \min_{j \in \mathcal{J}} |\Delta_j(s)|$, where $\mathcal{J}$ is a finite index set.
\end{definition}
Intuitively, we can think of the decision margin as the smallest gap across all comparisons in the extractor. To change the skeleton, at least one comparison must flip, which requires pushing some $\Delta_j$ across zero. The margin $m(s)$ quantifies the minimum perturbation required to reach the nearest decision boundary, so a large margin indicates that all decisions are comfortable, while a small margin indicates at least one comparison is nearly tied and could flip under small perturbations.

\begin{theorem}[Margin implies skeleton preservation]
\label{lem:margin_preserves_skel}
   Let $\mathrm{Skel} : \mathcal{S} \to \mathbb{S}$ be the skeleton extractor,
where $\mathbb{S}$ denotes the (finite) space of skeletons of $\mathrm{Cl}(G)$.
Assume $\mathrm{Skel}$ depends on the input only through a feature map
$\varphi : \mathcal{S} \to \mathbb{R}^d$ and a finite set of strict scalar
comparisons: namely, there exist a finite index set $\mathcal{J}$, functions
$u_j, v_j : \mathbb{R}^d \to \mathbb{R}$, and a deterministic map
$F : \{-1, +1\}^{\mathcal{J}} \to \mathbb{S}$ such that for every state
$s \in \mathcal{S}$ satisfying the no-ties condition $\Delta_j(s) \neq 0$ for
all $j \in \mathcal{J}$, we have
\begin{eqnarray*}
    \mathrm{Skel}(s)=F\bigl((\mathrm{sign} (\Delta_j(s)))_{j\in\mathcal{J}}\bigr)  \quad \rm{and} \quad
\Delta_j(s):=u_j(\varphi(s))-v_j(\varphi(s)).
\end{eqnarray*}
Further assume the decision gaps are $\ell_\infty$-Lipschitz:  %$\exists L_\Delta<\infty$ such that %for all $j\in\mathcal{J}$ and $s,s'\in\mathbb{S}$
$$\exists L_\Delta<\infty, \quad \rm{such \,\ that} \quad |\Delta_j(s')-\Delta_j(s)|\le L_\Delta\|\varphi(s')-\varphi(s)\|_\infty, \qquad \forall j\in\mathcal{J}, \forall s,s'\in\mathbb{S}.$$  %Define the margin $m(s):=\min_{j\in\mathcal{J}}|\Delta_j(s)|$. 
Then, if decision margin $m(s)>0$ and
$\|\varphi(s')-\varphi(s)\|_\infty<m(s)/L_\Delta$, %we have 
$\mathrm{Skel}(s')=\mathrm{Skel}(s)$. 
\end{theorem}
Proof is in Appx.~\ref{app:proof_marstab}.
This result provides important practical guarantees that the topological skeleton is
expected to be robust to minor perturbations.
Table~\ref{robustness_analysis} in Appx.~\ref{app_addexp} illustrates this idea in application to spatio-temporal graph learning. We emphasize that Theorem~\ref{lem:margin_preserves_skel} is a local stability statement about the extractor Skel, rather than a global claim about discrete Morse theory or the full diffusion model; below $r(s)$ the conditioning signal is unchanged, so the topological prior adds no new source of sensitivity. We further show this in Table \ref{tab:missing_events} in Appendix \ref{app_addexp} confirming our phenomena under 5\%-20\% of missing events. 
% \YGL{Everything is stated in terms of $\mathrm{Skel}$ not $\mathcal{E}(G)$ so explicit connection to $\mathcal{E}(G)$ is needed in the formulation somewhere.}

%\textbf{Stability under non-critical perturbations.}
%MG-Diff treats $\mathcal{S}(G)$ as a stable topological scaffold and permits the remainder of the graph or complex to vary. Perturbations affecting non-critical simplices, such as local rewiring or noise, do not alter the event set, enabling generative diversity while preserving global topological structure. Conditioning diffusion on $\mathcal{S}(G)$ thus stabilizes topological events without over-constraining non-critical geometry.

% need to condense
\subsection{Morse-Guided Diffusion Model}
We start from introducing a Morse-oriented mixture-of-experts conditioning mechanism, %for graph diffusion, 
where a discrete Morse gradient induces a critical/noncritical decomposition of the graph and routes denoising capacity to topology-informative vs. cancelable structure.

% \begin{figure}[ht]
%   \centering
%   \includegraphics[width=0.78\linewidth]{Figures/MG_Diff_vis.pdf}
%   \caption{The MG-Diff pipeline (nodes \& edges in red color are critical cells; nodes \& edges in blue color are non-critical cells).}
%   \label{fig:morse_diff}
% \end{figure}

%\newcolumntype{g}{>{\columncolor{white}}c}
\begin{table*}[t]
\centering
\caption{Prediction comparison over spatio-temporal graphs based on PEMS-03~\citep{chen2001freeway}, PEMS-BAY~\citep{li2018diffusion}, AIR-BJ~\citep{yi2018deep}, and AIR-GZ~\citep{yi2018deep}. }
%\vspace{-1.em}
  \scalebox{0.6}{
  \centering
  \small
  \begin{tabular}[width=1.\linewidth]{lc|ccc|ccc|ccc|ccc}
  \toprule
      \multirow{2}*{Method} & Data & \multicolumn{3}{c|}{PEMS-03} & \multicolumn{3}{c|}{PEMS-BAY} & \multicolumn{3}{c|}{AIR-BJ} & \multicolumn{3}{c}{AIR-GZ}\\
      \cline{3-14}
      & Metric & MAE & RMSE & CRPS & MAE & RMSE & CRPS & MAE & RMSE & CRPS & MAE & RMSE & CRPS \\
      %\hline
      \hline
      %STGCN & \multirow{7}*{\makecell{Deterministic\\ Method}} & 17.49 & 30.12 & \_ & 1.88 & 4.30 & \_ & 31.09 & 49.16 & \_ & 12.80 & 18.26 & \_ \\
      %DCRNN & & 18.03 & 30.01 & \_ & 1.68 & 3.83 & \_ & 30.98 & 49.05 & \_ & 9.98 & 15.20 & \_ \\
      %GWN & & 16.21 & 27.66 & \_ & {1.67} & {3.60} & \_ & 30.60 & {48.46} & \_ & {9.90} & {15.31} & \_ \\
      %STSGCN & & 17.48 & 29.21 & \_ & 1.79 & 3.91 & \_ & 31.00 & 49.46 & \_ & 12.02 & 18.59 & \_ \\
      %STGODE &  & 16.50 & 27.84 & \_ & 1.77 & {3.33} & \_ & 30.28 & {48.26} & \_ & 10.89 & 15.73 & \_ \\
      %STGNCDE &  & {15.57} & 27.09 & \_ & {1.68} & 3.66 & \_ & {30.45} & 49.17 & \_ & {10.10} & 15.90 & \_ \\
      %GMSDR &  & 15.78 & {26.82} & \_ & 1.69 & 3.80 & \_ & 32.15 & 51.08 & \_ & 11.78 & 17.72 & \_ \\
      %\hline
      TimeGrad  &  & 21.55 & 36.57 & 0.101 & 2.62 & 5.30 & 0.034 & {33.40} & 54.93 & {0.363} & 15.45 & 21.93 & 0.376 \\
      MC Dropout & & 18.87 & 29.81 & 0.093 & 3.50 & 5.43 & 0.040 & 37.92 & 55.49 & 0.391 & 13.10 & 19.26 & {0.290} \\
      CSDI & & 23.46 & 39.60 & 0.098 & 2.67 & {4.10} & 0.031 & 38.94 & 57.81 & 0.417 & 14.78 & 22.24 & 0.361 \\
      DiffSTG & & {17.58} & {28.75} & 0.095 & {2.03} & 4.22 & {0.025} & 38.03 & 56.87 & 0.373 & {13.06} & {18.25} & 0.319  \\
      PriSTI & & 22.30 & 37.58 & {0.092} & 2.51 & 3.99 & 0.026 & 36.81 & {54.34} & 0.388 & 14.04 & 21.03 & 0.352 \\
      %\rowcolor{Gray}
      USTD & & \underline{16.26} & 26.91 & \underline{0.082} & 1.86 & 3.81 & 0.025  & \underline{30.09} & \underline{47.65} & 0.348 & \underline{9.99} & \underline{15.41} & {\bf 0.235}\\ 
      DSTD & & 17.10 & 26.83 & 0.096 & \underline{1.72} & \underline{3.94} & \underline{0.024} & 31.86 & 48.79 & 0.356 & 10.75 &16.40&0.312\\
      ZS-DM & & 16.57 & \underline{26.46} & 0.093 & 1.78 & 3.76   & 0.025 &  31.72 & 48.88 & \underline{0.343} & 10.30& 16.10 & 0.315\\
    \hline
    %{\bf MG-Diff - morse\_v1} & & {\bf 15.63} & {\bf 26.82} & {\bf 0.077} \\
    {\bf MG-Diff (Ours)} & & {\bf 15.67} & {\bf 26.06} & {\bf 0.077} & {\bf 1.63} & {\bf 3.52} & {\bf 0.022} &  {\bf 29.91} & {\bf 47.10} & {\bf 0.290} & {\bf 9.80} & {\bf 15.07} & \underline{0.254}
    \\
     %{\bf MG-Diff - morse\_v2 + CRPS} & &  \\
  \bottomrule
\end{tabular}
}
\label{tab:forecasting}
\end{table*}

\begin{table}[h]
\caption{Graph generation performance on community-small data in MMD along with the average
Wasserstein distance $W_2$ and mean absolute error in Betti $\beta_0$ and $\beta_1$ numbers between the observed and generated graphs. 
\label{graph_gen_res}}
\centering
\begin{minipage}{0.48\linewidth}
\centering
\scalebox{0.75}{
\begin{tabular}{lcc}
\hline
Method & Degree $\downarrow$ & Clus. $\downarrow$ \\
\hline
GDSS & 0.0373 & \underline{0.0723} \\
DiGress & 0.0899 & 0.1920 \\
GraphVAE-MM & 0.0587 & 0.3560 \\
HOG-Diff & 0.4700 & 0.3650 \\
CoPHo & 0.0523 & 0.0911 \\
SwinGNN & \underline{0.00366} & 0.0862 \\ \hline
{\bf MG-Diff (Ours)} & {\bf 0.00249} & {\bf 0.0734}\\
\hline
\end{tabular}
}
%\caption{Comparison of methods based on Wasserstein distance and Betti number differences.}
\end{minipage}
\hfill
\begin{minipage}{0.48\linewidth}
\centering
\scalebox{0.8}{
\begin{tabular}{lccc}
\hline
Metric & GDSS& SwinGNN & {\bf MG-Diff (Ours)} \\
\hline
$W_2$ & 14.33 & \underline{11.52} & {\bf 5.87} \\
$\rm{MAE}_{\beta_0}$ & 6 & \underline{5} & {\bf 3} \\
$\rm{MAE}_{\beta_1}$ & 22 & 22 & {\bf 19} \\
\hline
\end{tabular}
}
\end{minipage}

\end{table}

%{\bf Forward process.} We adopt the standard forward diffusion process to progressively transform the input data into a tractable noise distribution. That is, the forward diffusion process incrementally corrupts the original input data $X_0$ through the injection of Gaussian noise over $\Lambda$ steps: %which can be represented as follows:
%\begin{equation}
%     q(X_{1:K} | X_0)  = \prod_{k=1}^{K}q(X_k | X_{k-1}),\quad
%     q(X_k|X_{k-1}) = \mathcal{N}(X_k; \sqrt{1-\beta_k}X_{k-1}, \beta_k {\boldsymbol{I}}) \nonumber,
%\end{equation}
%where $\alpha_k = 1-\beta_k$ and $\bar{\alpha}_k = \prod_{s=1}^{k} \alpha_s$.

{\bf Morse-oriented mixture of experts.} We then introduce a Morse-oriented Mixture-of-Experts (MoMoE) architecture to explicitly disentangle and selectively encode heterogeneous topological signals arising from the underlying complex. In particular, the MoMoE framework enables different experts to specialize in capturing distinct structural regimes, i.e., critical cells which (with respect to a discrete gradient vector field $\mathcal{V}$) correspond to irreducible topological events, and non-critical cells which encode redundant structures removable via homotopy-preserving cancellations. Specifically, we first split the original graph $G = (X^h, A)$ into critical cell (i.e., $G_{\alpha}$) and non-critical cell (i.e., $G_{\alpha^\prime}$) parts which are described as $G_{\alpha} = (X^h, A_{\alpha}), \ G_{\alpha^\prime} = (X^h, A_{\alpha^\prime})$, where $X^h$ denotes the historical input, and $A_{\alpha}$ and $A_{\alpha^\prime}$ are the critical cell-based and non-critical cell-based adjacency matrix respectively. Through input-dependent routing, the MoMoE adaptively allocates representational capacity to these complementary topological components and is defined by $\bar{H} = \theta_{\alpha}H_{\alpha} + \theta_{\alpha^\prime}H_{\alpha^\prime}$, with:
\vspace{-.27em}
 \begin{align}
    H_{\alpha} = \text{MLP}_{\alpha}([X^h,\phi_{\text{pre}}(G_\alpha)]), \quad
    H_{\alpha^\prime} = &\text{MLP}_{\alpha^\prime}([X^h, \phi_{\text{pre}}(G_{\alpha^\prime})]), \nonumber
%\end{split}
%\end{align}
 \end{align}
where ${H}_{\alpha}$ and ${H}_{\alpha^\prime}$ are expert-aligned critical and non-critical cell features respectively, $\theta_{\alpha}$ and $\theta_{\alpha^\prime}$ are the routing weights for two experts, and $\bar{H}$ denotes the Morse-aware condition obtained from the MoMoE. Each expert internal structure comprises a multi-layer perceptron (MLP), i.e., $\mathrm{MLP}_{\alpha}$ for critical cells or $\mathrm{MLP}_{\alpha^\prime}$ for non-critical cells, $[\cdot,\cdot]$ stands for the concatenation function, and the pre-trained graph neural network (GNN)~\citep{hu2024unifyingdiffusionmodelsprobabilistic} is denoted as $\phi_{\text{pre}}(\cdot)$. In other words, MoMoE and the denoiser below are generic neural components conditioned on a Morse-derived partition; under $H_0$-faithfulness (Appx. \ref{sec:H0_faithful_semantics}) $|E^\text{crit}| = \beta_1(G)$, so the split into a cycle-space basis and a spanning forest backbone fixes both the number of experts and the routing.

{\bf Morse-guided dual-attention denoising network.} The denoising network is designed to recover stochastic perturbations injected during the forward diffusion process. %In our study design, 
We generate the noisy data $X_k$ by %generated by 
corrupting the clean input $X_0$ with Gaussian noise $\epsilon$. The denoising network aims to predict this noise realization, conditioned on both the diffusion timestep and the MoMoE's guidance. Formally, given a noisy input 
$X_k$ and Morse-aware condition $\bar{H}$, we introduce a Morse-guided Dual-attention Denoising Network (MDDN) %. The formulation of MDDN is 
as follows: %\missing{transpose?}
\begin{eqnarray}
Z =\upsilon_1 Z_1 + \upsilon_2 Z_2, \quad
    Z_1  = \text{softmax}(\frac{Q_{sq}K^{\top}_{sk}}{\sqrt{d_h}})V_{sv}, \quad
    Z_2  = \text{softmax}(\frac{Q_{cq}K^{\top}_{ck}}{\sqrt{d_h}})V_{cv}
%\end{split}
\end{eqnarray}
where $Q_{sq} = W_{sq} X_k$, $K_{sk} = W_{sk} X_k$, and $V_{sv} = W_{sv} X_k$, $W_{sq}, W_{sk}, W_{sv}$ are the head-specific query, key, and value matrices of the self-attention module, $Q_{cq} = W_{cq} X_k$, $K_{ck} = W_{ck} \bar{H}$, and $V_{cv} = W_{cv} \bar{H}$, $W_{cq}, W_{ck}, W_{cv}$ are weights of linear projections of the cross-attention module, and $\frac{1}{\sqrt{d_h}}$ is a scaling factor ($d_h$ is the embedding dimension); $\upsilon_1, \upsilon_2$ are attention scores which are calculated by $\upsilon_i = \frac{\exp{(Z_i)}}{\sum^2_{j=1}\exp{(Z_j)}}$, and $Z$ denotes the final MDDN's output.

{\bf Forward and Diffusion Processes.} We adopt the standard forward diffusion process to progressively transform the input data into a tractable noise distribution, and integrate our MoMoE into a conditional diffusion process.
The reverse process of MG-Diff is conditioned on the critical cell and non-critical cell graph structures ${G}_{\alpha}$ and $G_{\alpha^\prime}$, and the Morse-aware condition $\bar{H}$, %which is 
formulated as:
\begin{align}
%\begin{split}
 %\begin{multline}
        p_{\theta}({X}_{0:K} | \bar{H}, {G}_{\alpha}, {G}_{\alpha^\prime})  %\\
 = p({X}_K)
\prod_{k=1}^{K}
p_{\theta}\!(
{X}_{k-1} \mid {X}_k, \bar{H}, {G}_{\alpha}, {G}_{\alpha^\prime}).
\label{eq:reverse_process}
%\end{split}
\end{align}
%\end{multline}
The training objective minimizes the mean squared error between the true noise $\epsilon$ and the noise predicted by the denoising network $\epsilon_{\theta}$ is $\mathcal{L}
= \mathbb{E}_{{X}_0, {\epsilon}, k}
\Big[
\big\|
{\epsilon}
- {\epsilon}_{\theta}
({X}_k, \bar{H}, {G}, k)
\big\|^2
\Big]$.
% =========================
% 7. Experiments (must be concrete)
% =========================

\section{Experiments}
\label{sec:experiments}
%{\bf Datasets, baselines, and experimental setup.}
%We perform extensive experiments on real-world datasets which are obtained from open sources including (i) PEMS-03~\cite{chen2001freeway}; (ii) PEMS-BAY~\cite{chen2001freeway,li2018diffusion}; (iii) AIR-BJ~\cite{yi2018deep}; (iv) AIR-GZ~\cite{yi2018deep}. 
\vspace{-.1em}
%need to fix the spacing
We evaluate MG-Diff with respect to graph diffusion for spatio-temporal forecasting and graph regeneration. %(see Appx. \TODO{need citation}).
%To evaluate MG-Diff, we compare it with 8 state-of-the-art baselines (see Appx. \TODO{need citation}). %, i.e., (i) TimeGrad~\cite{rasul2021autoregressive}, (ii) MC Dropout~\cite{wu2021quantifying}, (iii) CSDI~\cite{tashiro2021csdi}, (iv) DiffSTG~\cite{wen2024diffstgprobabilisticspatiotemporalgraph}, (v) PriSTI~\cite{liu2023pristiconditionaldiffusionframework}, (vi) USTD~\cite{hu2024unifyingdiffusionmodelsprobabilistic} , (vii) Decoupled spatio-temporal diffusion model (DSTD) (anonymous ICLR2026 submission), and (viii) ZS-DM~\cite{chen2025topological}. 
%We use %three 
%standard metrics: Mean Absolute Error (MAE), Root Mean Square Error (RMSE), and Continuous Ranked Probability Score (CRPS) as measure of forecasting performance, and \YGL{need to add metrics depending on the selected results} as measure of regeneration performance. 
All experiments are conducted on 4 NVIDIA RTX A5000 GPUs. 
%(For more details on experiments %data, parameter settings, and code 
%see Appx.~\ref{app_data}). The {\bf bold} is the overall %best result, and the \underline{underline} indicates the %second-best baseline methods.
PEMS-03, PEMS-BAY, AIR-BJ, and AIR-GZ contain 358, 325, 36, and 42 nodes, with chronological 70\%/10\%/20\% splits and 24-step windows (Table~\ref{tab:dataset_stats}). 
(For more details on experiments %data, parameter settings, and code 
see Appx.~\ref{app_data}). The {\bf bold} is the overall best result, and the \underline{underline} indicates the second-best baseline methods.

{\bf Results.} Table~\ref{tab:forecasting} presents
results on graph diffusion for spatio-temporal forecasting.
%the overall prediction performances on all 4 spatio-temporal graph data. 
We find that: (i) Our MG-Diff consistently yields the strongest %forecasting 
performance across all scenarios. On PEMS-03, MG-Diff achieves the lowest MAE, RMSE, and CRPS, improving upon the runner-up (i.e., USTD) by 3.6\% in MAE, 3.2\% in RMSE, and 6.1\% in CRPS, providing both more accurate predictions and better-calibrated uncertainty. On PEMS-BAY, MG-Diff %further reduces MAE to 1.63 and RMSE to 3.52, 
outperforms the best competing methods by 6.5\% - 12.4\% across metrics (again also attaining the lowest CRPS) which highlights its superior probabilistic forecasting capability; (ii) On AIR-BJ, MG-Diff achieves the lowest MAE, RMSE, and CRPS. Compared to the second-best baselines (i.e., USTD and ZS-DM), this corresponds to approximately 0.6 – 1.0 %absolute error 
reduction in MAE and RMSE and a 15.7\% relative improvement in CRPS, demonstrating enhanced robustness under challenging urban air-quality dynamics. On AIR-GZ, MG-Diff attains the best MAE and RMSE, surpassing runner-ups such as USTD and ZS-DM, while achieving competitive CRPS, second only to USTD. Table~\ref{graph_gen_res} presents the results of graph generation on 
community-small ({\it left}) data along with the measures on how well the global topology is preserved ({\it right}) by the next two best contenders. %\YGL{Yuzhou, can we add GDSS?}. 
%Notably, MG-Diff outperforms the baselines across different metrics. More specifically, on the tree dataset, MG-Diff achieves the lowest error on Degree, Wavelet, Spectral, and Orbit statistics, with particularly notable gains on spectral metrics, which suggests improved modeling of global connectivity patterns. 
%On the community-small dataset, 
MG-Diff achieves the best performance on both Degree and clustering coefficient. Compared to strong baselines such as SwinGNN and GDSS, improvements are most pronounced on Degree (up to 46.98\% relative gain), indicating more accurate control over node-level distributions, while maintaining competitive clustering structure. 
MG-Diff also yields substantially lower Wasserstein distance $W_2$ and smaller 
discrepancies in Betti $\beta_0$ and $\beta_1$ numbers than SwinGNN and GDSS, %\YGL{Yuzhou, can we add results on GDSS?}, 
reflecting improved preservation of connectivity and cycle structures. See also Appx.~\ref{app_addexp} for more experiments. %graph regeneration experiments on tree data.
%and exceeds the performance of runner-ups by up to 71.08\% on Degree

{\bf Topological Uncertainty Quantification, or Why Cobordisms in GDMs?} 
In GDM the denoiser may produce signals that are locally consistent at each node but globally inconsistent. Without skeleton conditioning, nothing in the training objective explicitly penalizes this kind of global incoherence. Also, MSE is a point forecast measure and does not directly assess whether forecasts over the cycle-connected nodes exhibit the appropriate interdependence, but it has a major impact on probabilistic forecast and UQ measured via CRPS~\citep{gneiting2007strictly}. The skeleton conditioning prevents this by anchoring each denoising step to the same global cyclic context. During the early high-noise steps of the reverse process where the denoiser has least signal to work with, the skeleton is the dominant source of structural guidance. Conditioning the denoiser on the Morse skeleton yields a predictive distribution 
that better respects the structural organization imposed by the graph spanning connectivity, resulting in the MG-Diff CRPS gains of 6.1\%-15.7\%. Accordingly, the CRPS gain is largest on AIR-BJ (ring-like roads, high $\beta_1$), intermediate on PEMS, and smallest on AIR-GZ (dense, opaque skeleton). On tree-like graphs $(\beta_1 \approx0)$ $E^\text{crit}$ is empty and MG-Diff reduces to standard conditional diffusion (Appx.\ref{app_addexp}).

{\bf Ablation Studies.}
Table~\ref{ablation_analysis} in Appx.~\ref{app_abl} shows the results of ablation studies on PEMS-03 and AIR-BJ in MAE and CRPS,
%indicate that: (i) removing the MoMoE module leads to a consistent degradation in CRPS, i.e., 6.5\% and 20.0\% relative increase on PEMS-03 and AIR-BJ respectively; (ii) removing MDDN module leads to a more substantial performance drop, with CRPS rising sharply to 0.234 and 0.534 on PEMS-03 and AIR-BJ respectively. These findings 
demonstrating that both MoMoE and MDDN play critical roles in model performance and especially robust uncertainty modeling, with relative degradation up to 20\%. %\YC{we can move table 3 to appendix?}
To demonstrate the contribution of Morse skeleton guidance, we provide additional ablation results by using simpler structural prior, i.e., degree-based and betweenness-based scores. Table~\ref{tab:ablation_different_morse_values} in Appx.~\ref{app_abl} shows that, replacing Morse skeleton by degree or betweenness centrality scores results in an average 4.53\% and 2.06\% degradation on PEMS03 and AIR-BJ respectively, which highlights that Morse skeleton makes a substantial contribution to forecasting performance. To isolate whether the gain stems from the Morse topology rather than from extra capacity, Table \ref{tab:ablation_different_morse_values} replaces the Morse skeleton by a random skeleton of matched size under an identical architecture, training configuration, and checkpoint. We see that the Morse skeleton wins by 6.0\%, 3.5\%, 7.2\%, and 8.4\% in MAE, RMSE, MAPE, and CRPS on AIR-GZ (highly significant over 5 seeds), and replacing the scoring rule \eqref{eq:degree_seed} by $deg(v) + \epsilon_v$ or $\deg_{\mathrm{med}}+\deg(v)+\varepsilon_v$ changes performance only modestly, i.e., the skeleton is robust as long as the scoring preserves the degree-based ordering.

{\bf Complexity.}\; %Let $N=|V|$, $M=|E|$, 
Let $T_\triangle$ be the number of triangles, $K$ be the number of critical cells, $S$ be the diffusion steps, $\ell$ and $\Lambda$ be the counts of WaveNet and transformer layers, $C$ be the embedding dimension, $L$ be the sequence length, $B$ be the batch size, and $\Delta$ be the maximum degree.
The Morse skeleton is built once on the fixed sensor graph in $\mathcal{O}(M\Delta)$. 
Triangle sparsity %and small $\Delta$ 
keep this efficient. 
The discrete Morse function is assigned via a degree-based scan in $\mathcal{O}(N+M+T_\triangle)$; critical simplices are identified in $\mathcal{O}(N M + M T_\triangle)$. Gradient-respecting paths between critical cells are found by BFS over the incidence graph with $\mathcal{O}(K^2 (N+M+T_\triangle))$ complexity. With a $0.9$ pruning rate, roughly $K\!\approx\!0.1N$, giving a worst-case of $\mathcal{O}(N^3)$. See Appx.~\ref{app_abl} for more detail. Empirically, the one-time Morse construction is cheap; on ogbn-arxiv (169K nodes, 1.16M edges) it runs in $2.2$ seconds at about $25$MB, and adds no per-epoch cost (Table \ref{tab:scalability}, Appx. \ref{app_abl}). %Three optimizations improve scalability: replace finding cliques with triangle enumeration over common neighborhoods ($\mathcal{O}(M\Delta)$); precompute edge--triangle coface indices to cut identification to $\mathcal{O}(M+T_\triangle)$; and restrict path search to local critical pairs, reducing cost from $\mathcal{O}(K^2)$ to $\mathcal{O}(Kk)$ for small $k$. Each iteration runs the GWaveNet encoder and denoiser with backprop. Graph convolution dominates at $\mathcal{O}(B C N^2 L)$ per hop and per layer, giving total $\mathcal{O}((\ell+\Lambda)BCN^2L)$, where $N^2$ arises from dense adjacency and quadratic attention. Inference performs $S$ denoising steps after a single encoding pass, costing $\mathcal{O}(\ell BCN^2L + S\Lambda BN^2C)$. For probabilistic evaluation with $Q$ stochastic samples, inference scales as $\mathcal{O}(\ell BCN^2L + QS\LambdaBN^2C)$, roughly $QS$ times slower than one denoiser call per training step.

\section{Conclusion}
\label{sec:conclusion}
We have introduced the powerful apparatus of cobordism and discrete Morse theory into graph diffusion. In particular, we have developed a new mathematically rigorous methodology MG-Diff for spatio-temporal graph generation that explicitly guides a diffusion generator on which global graph structures are to be intact and which local substructures can be randomized. MG-Diff has shown competitive gains both in forecasting and uncertainty quantification.
In the future, we plan to advance the MG-Diff to diffusion on manifolds and probabilistic object alignment, and to dynamic graph, where a graph trajectory is interpolated by a non-trivial cobordism whose Betti-number increments index predictive uncertainty (see also Appx. \ref{app_addexp}).

\section*{Acknowledgement}
This work has been supported in part by by NSF DMS-2533984/2533985 and OAC-2530469/253047, and by the U.S. Department of Energy (DOE), Office of Science, Advanced Scientific Computing Research (ASCR) program under the Scientific Discovery through Advanced Computing (SciDAC) Institute ``LEADS: LEarning-Accelerated Domain Science''.

\clearpage
%\section*{Impact Statement} This paper presents work whose goal is to advance the field of machine learning and bring untapped uses of theoretical topology into practice. There are many potential societal consequences of our work, none of which we feel must be specifically highlighted here.

%\section*{Acknowledgments}
%\newpage

% NeurIPS-style \bibitem entries
% Use with: \begin{thebibliography}{99} ... \end{thebibliography}
% Requires natbib (NeurIPS style file loads it automatically).

\bibliographystyle{plain}
\bibliography{references}

\newpage
\appendix
\section{Mathematical Background}
\label{sec:morse_appendix}

We start from  providing the following Table~\ref{tab:notation} of notation to ease readability.

\begin{table}[t!]
\centering
\caption{Table of Notations}
\label{tab:notation}
\small
\resizebox{0.8\linewidth}{!}{
\begin{tabular}{@{}ll@{}}
\toprule
\textbf{Symbol} & \textbf{Meaning} \\
\midrule
\multicolumn{2}{l}{\textit{Graphs and complexes}} \\
\midrule
$G=(V,E,X)$ & Simple graph: vertices $V$, edges $E$, node features $X$ \\
$N=|V|,\ M=|E|$ & Number of vertices and edges \\
$\deg(v)$ & Degree of vertex $v$; $\deg_{\max}=\max_{u\in V}\deg(u)$ \\
$\Lambda$ & Finite simplicial / regular CW complex \\
$\mathrm{Cl}(G)$ & Clique complex of $G$ (simplices $=$ cliques of $G$) \\
$\Lambda^{(p)},\ \alpha^{(p)}$ & Set of $p$-cells; a $p$-cell of $\Lambda$ \\
$\alpha\prec\beta$ & $\alpha$ is a (proper) face of $\beta$ \\
$\beta_p(K)$ & $p$-th Betti number; $\chi(K)$: Euler characteristic \\
\midrule
\multicolumn{2}{l}{\textit{Discrete Morse theory}} \\
\midrule
$g:V\to\mathbb{R}$ & Vertex score (seed); we use $g(v)=\deg_{\max}-\deg(v)+\varepsilon_v$ \\
$\varepsilon_v$ & Tie-breaking noise, $\varepsilon_v\sim\mathrm{Unif}(0,\delta)$ \\
$f:K\to\mathbb{R}$ & Discrete Morse function on $\Lambda$ \\
$U_f(\alpha),\ V_f(\alpha)$ & Cofaces with $f(\beta)\le f(\alpha)$; faces with $f(\gamma)\ge f(\alpha)$ \\
$\mathrm{Flag}(\alpha)$ & Construction-time flag tracking pairing of $\alpha$ \\
$\eta(e)$ & Edge-noise term in raise branch, $\eta(e)\sim\mathrm{Unif}(0,0.5)$ \\
$\mathcal{V}$ & Discrete gradient vector field (acyclic matching on $\Lambda$) \\
$(\sigma,\tau)\in\mathcal{V}$ & Matched pair: $\sigma\prec\tau$, $f(\tau)\le f(\sigma)$ \\
$G_{\mathcal{V}}$ & Directed Hasse diagram induced by $\mathcal{V}$ \\
$\mathcal{V}$-path & Alternating directed path in $G_{\mathcal{V}}$ \\
$E^{\mathrm{pair}},\ E^{\mathrm{crit}}$ & Paired vs.\ critical edges; $E=E^{\mathrm{pair}}\sqcup E^{\mathrm{crit}}$ \\
$\Gamma(\mathcal{V})$ & Paired-edge digraph on $V$ \\
$r:V\to V$ & Anchor map: $r(v)=$ unique sink reached in $\Gamma(\mathcal{V})$ \\
\midrule
\multicolumn{2}{l}{\textit{Morse skeleton}} \\
\midrule
$\mathcal{C}(f)$ & Set of critical cells; $\mathcal{C}_p(f)$: critical $p$-cells \\
$\mathcal{D}_{\mathcal{V}}(\sigma)$ & Descending set of critical cell $\sigma$ \\
$\pi(\alpha)$ & Single representative descending path from critical cell $\alpha$ \\
$\mathcal{D}=\{(\alpha,\pi(\alpha)):\alpha\in\mathcal{C}(f)\}$ & Descending paths used by MG-Diff \\
$\mathcal{S}(G)=(\mathcal{C},\mathcal{D})$ & Morse skeleton of $G$ \\
$\mathfrak{S}$ & Space of possible Morse skeletons of $\mathrm{Cl}(G)$ \\
\midrule
\multicolumn{2}{l}{\textit{Extractor and stability}} \\
\midrule
$\mathbb{S}$ & State space (inputs read by the extractor: $G$, scores, flags) \\
$s\in\mathbb{S}$ & A state \\
$\mathrm{Skel}:\mathbb{S}\to\mathfrak{S}$ & Deterministic skeleton extractor; $\mathrm{Skel}(s)=\mathcal{S}(G)$ \\
$\varphi:\mathbb{S}\to\mathbb{R}^d$ & Feature map on states \\
$\mathcal{J}$ & Finite index set of comparisons performed by extractor \\
$u_j,v_j:\mathbb{R}^d\to\mathbb{R}$ & Left/right sides of comparison $j\in\mathcal{J}$ \\
$\Delta_j(s)=u_j(\varphi(s))-v_j(\varphi(s))$ & Decision gap for comparison $j$ \\
$F:\{-1,+1\}^{\mathcal{J}}\to\mathfrak{S}$ & Sign-pattern $\to$ skeleton map \\
$m(s)=\min_{j}|\Delta_j(s)|$ & Decision margin at state $s$ \\
$L_\Delta$ & Lipschitz constant of decision gaps \\
$r(s)=m(s)/L_\Delta$ & Stability radius \\
\midrule
\multicolumn{2}{l}{\textit{Diffusion model}} \\
\midrule
$X_0,\ X_k$ & Clean input; noised state at diffusion step $k\in\{1,\ldots,K\}$ \\
$X^h$ & Historical input \\
$\beta_k,\ \alpha_k=1-\beta_k,\ \bar\alpha_k$ & Noise schedule \\
$\epsilon,\ \epsilon_\theta$ & Injected noise; denoiser network \\
$p_\theta(\cdot\mid X_k,\bar H,k)$ & Reverse-step kernel \\
$G_\alpha,\ G_{\alpha'}$ & Critical / non-critical cell subgraphs \\
$A_\alpha,\ A_{\alpha'}$ & Critical / non-critical adjacency matrices \\
$\Psi:\mathfrak{S}\times\mathcal{X}\to\mathcal{H}$ & Guidance map \\
$\bar H(s)=\Psi(\mathrm{Skel}(s),X^h)$ & Morse-aware guidance signal \\
$H_\alpha,\ H_{\alpha'}$ & Expert features (critical / non-critical) \\
$\theta_\alpha,\ \theta_{\alpha'}$ & MoMoE routing weights \\
$Q,K,V$ (with subscripts) & Self-/cross-attention query, key, value matrices \\
$d_h$ & Embedding dimension \\
$\upsilon_1,\upsilon_2$ & MDDN attention scores \\
\bottomrule
\end{tabular}}
\end{table}

\paragraph{Organization.}
Section \ref{sec:morse_appendix} gives brief definitions on graphs and their associated clique complexes.
Section~\ref{sec:smooth_morse} reviews smooth Morse theory on manifolds.
Section~\ref{sec:cobordism_appendix} introduces cobordisms and their relationship to Morse functions.
Section~\ref{sec:discrete_morse_appendix} presents discrete Morse theory for cell complexes.
Section~\ref{sec:discrete_flow_descending} develops the theory of $\mathcal{V}$-paths and descending sets.
Section~\ref{sec:graph_acyclic_orientation} specializes these constructions to graphs.
Section~\ref{sec:H0_faithful_semantics} establishes semantic guarantees under $H_0$-faithfulness.
Section~\ref{sec:stability_skeleton_guidance} proves stability results for Morse-skeleton guidance, and specifically Section~\ref{sec:margin_stability} establishes margin stability of the 1D Morse skeleton. Section~\ref{sec:worked_ex} presents detailed worked examples.
Section~\ref{sec:ph_counterexamples} provides counterexamples showing persistence does not determine descending manifolds.
Section~\ref{sec:algorithms_appendix} contains the complete algorithmic specifications.

This appendix provides comprehensive mathematical background on smooth Morse theory, cobordisms, discrete Morse theory, and their applications to graph diffusion. While these ideas are classical in topology, we include them for completeness and to make precise the objects manipulated by our algorithmic constructions. For deeper theory, we refer readers to classical texts in Morse theory and algebraic topology~\cite{milnor1963morse, stong1968cobordism, forman1998morse, hatcher2002algebraic}.
%========================================%===================

\subsection{Positioning}

Persistent homology has emerged as the central algebraic tool in TDA,
capturing the birth and death of homological features across a filtration of spaces.
Given a real-valued function $f:M\to\mathbb{R}$ on a manifold (or a filtered complex $\Lambda$),
the persistence diagram $\mathrm{Dgm}_*(f)$ records \emph{when} homological features are born
and \emph{when} they are paired (killed by other critical points), yielding a multiset of
birth--death intervals that is stable under $ ||\cdot ||_\infty$ perturbations of $f$ \cite{cohen2007stability}.

In parallel, classical Morse theory provides a fundamentally \emph{geometric} description of
how gradient flow organizes a manifold into cells: the stable and unstable (ascending and descending)
manifolds of critical points, together with their intersection patterns, form the
\textit{Morse-Smale complex} \cite{ smale1961generalized, milnor1963morse}.
This flow-induced decomposition captures not merely the fact that critical points exist and pair,
but \emph{how} gradient trajectories connect them, which informs essential understanding
the global dynamics and homotopy type of the space.

\paragraph{The gap: homology vs.\ homotopy.}
It is folklore in the computational topology community that persistent homology and Morse--Smale
structure encode \emph{complementary} rather than equivalent information.
As Bauer \cite{bauer2011persistence} states precisely:
\begin{quote}\itshape
``Whereas (discrete) Morse theory makes statements about the homotopy type of the sublevel sets
of a function, persistence is concerned with their homology. While homology is an invariant of
homotopy equivalences, the converse is not true.''
\end{quote}
The Hurewicz homomorphism $\pi_1(X)\to H_1(X)$ further abelianizes the fundamental group (\cite{hatcher2002algebraic}),
discarding all non-abelian structure that may be critical for understanding gradient flow connectivity.
Thus persistence diagrams encode \emph{critical values and their pairing structure} but not
\emph{gradient flow geometry}.

\paragraph{Existing evidence: counterexamples on surfaces.}
The most significant evidence for this gap comes from Catanzaro et al. \cite{Catanzaro_2020}, who construct explicit examples of two Morse
functions on $S^2$ with \emph{identical persistence barcodes and identical Reeb graphs} that
are nonetheless not gradient-equivalent---i.e., they have different Morse--Smale complexes.
Their paper introduces the \textit{level set invariant} (a poset of nested disks and annuli)
to distinguish what persistence and Reeb graphs cannot, and they explore ``the moduli space of
Morse functions sharing the same persistence barcode.''

Similarly, recent work on computational accuracy of Morse--Smale complexes \cite{thanh2024revisitingaccurategeometrymorsesmale}
observes that ``the connectivity of the separatrices differs'' between different discretizations,
and crucially, ``topological differences have a high persistence'' two Morse-Smale decompositions
can yield identical persistence diagrams.

However, \emph{no paper contains a formal theorem} stating that persistent homology \emph{cannot} determine descending manifold structure. In this paper, we provide explicit counterexamples in Section~\ref{sec:ph_counterexamples} showing that there exists no well-defined map from persistence diagrams to descending-manifold structures.

\subsection{Topology in Modern Machine Learning}

Over the last several years, ``topology in machine learning'' has matured from a niche feature-engineering tool into a set of practical modeling primitives that (i) encode higher-order structure beyond edges, (ii) impose global consistency constraints, and (iii) stabilize learning when local heuristic descriptors are insufficient. One major thrust comes from \emph{persistent homology}, which provides multiscale summaries of data shape that can be robustly incorporated into learning algorithms. Seminal work in network topology showed how discrete Morse theory can be used to dramatically reduce the complexity of persistent homology computations on unweighted graphs by identifying critical simplices and constructing efficient filtrations~\cite{kannan2019persistent}.

In parallel, \emph{topological deep learning} has pushed topology from a post-hoc descriptor to an architectural design principle: rather than restricting message passing to 1-skeleton neighborhoods, models lift data to simplicial or cell complexes and propagate information over higher-order incidences, improving expressivity and capturing multiway interactions. A canonical example is Message Passing Simplicial Networks (MPSN), which extends Weisfeiler--Lehman-style reasoning from graphs to simplicial complexes and formalizes how higher-order message passing increases representational power~\cite{bodnar2021wltopological}.

Most relevant to this paper's goal (topology-aware diffusion/generation) is the emerging line of work that treats topology not merely as an intrinsic characteristic to match, but as a \emph{essential signal} structuring generative trajectories. On the applied side, TopoGen demonstrates that diffusion-style generation can be steered by topological constraints to produce topology-preserving outputs~\cite{hu2025topogen}. Additionally, recent studies have begun to connect persistent homology to diffusion dynamics through \emph{coarsening and multiscale structure}: persistent-homology-derived summaries are used to define meaningful intermediate objectives that guide diffusion at different resolutions, improving global structure retention during sampling.

Taken together, these developments motivate our central design choice: if a diffusion model is to avoid \emph{topological drift}, the intermediate object used for coarse-to-fine guidance should encode \emph{actual topological events and admissible refinement routes}, rather than only an ordering heuristic on local motifs.

Discrete Morse theory (DMT) is established in computational topology and machine learning, e.g. for image segmentation \citep{hu2023learning} and graph reconstruction\citep{dey2018graph}. Persistent homology (PH) is used throughout topological deep learning\citep{horn2022topological,yan2022neural,verma2024topological}, and topology-aware diffusion has been proposed for graphs and images \citep{park_topology-aware_2025, gupta2024topodiffusionnet}. Our contribution is therefore not the sole introduction of DMT itself, but its integration into graph diffusion as a topology-aware conditioning signal. We note that a diagram-matching loss $\beta_0/\beta_1$ attention \citep{park_topology-aware_2025}, a Betti-number loss \citep{gupta2024topodiffusionnet}, injected PH features\citep{horn2022topological,verma2024topological}, or a learned diagram carrier. If we write $P: (K,f) \rightarrow \text{Dgm}_*(K)$ for the persistence map, each signal factors as $\psi \circ P$ whereas our conditioner $\text{Skel}(K) = (C(f), \mathcal{V}, \mathcal{D}_\mathcal{V}(\alpha))$ does not. We note that by Proposition \ref{prop:skel-strict-refinement} and Section~\ref{sec:ph_counterexamples}, there exist $(K,f)$ and $(K,f')$ with $P(K,f)=P(K,f')$, but result in different descending paths, so the persistence diagram is not a sufficient statistic for the descending-path structured and no PH signal recovers what MoMoE routes on. We can compare directly to \citep{park_topology-aware_2025}, as it is conditioned on exactly the insufficient PH object. The DMT-based methods \citep{hu2023learning,dey2018graph} do build a Forman gradient field and a critical-edge spanning forest partition (or sample persistence thresholded $1$-stable manifolds), but we emphasize that MG-Diff differs in three respects. First, in the \textit{domain} sense, as they require an ambient scalar field on a cubical or simplicial complex, whereas we act on an abstract graph via a degree score, as $g$ can absorb such a field when one is available, though degree is the intrinsic default. Second, in its \textit{use}, we they perform deterministic reconstruction or segmentation, whereas we use the skeleton as a generative diffusion condition. Thirdly, in \textit{persistence}, as persistence is used to prune the Morse complex, whereas our skeleton is persistence-free; wile persistence may simplify a complex, the diagram does not determine the descending paths. We summarize these distinctions in \ref{tab:positioning}.

\begin{table}[H]
\centering
\caption{Positioning of MG-Diff among topology-aware diffusion and discrete-Morse (DMT) methods in ML.}
\label{tab:positioning}
\resizebox{\textwidth}{!}{
\begin{tabular}{llcccc}
\toprule
Method & Topological signal & DMT & Persistence & Domain & Task \\
\midrule
TAGG~\citep{park_topology-aware_2025} & PH (diagram-matching loss, $\beta_i$ attention) & No & full & abstract graph & graph generation (diffusion) \\
TopoDiffusionNet~\citep{gupta2024topodiffusionnet} & PH (Betti number) & No & full & image/cubical & image generation (diffusion) \\
Prob.\ Topo.\ Rep.~\citep{hu2023learning} & DMT+PH & Yes & prunes Morse complex & image likelihood field & segmentation (reconstruction) \\
Dey--Wang--Wang~\citep{dey2018graph} & DMT+PH & Yes & prunes Morse complex & density field $\rho$ & graph/curve reconstruction \\
TOGL~\citep{horn2022topological} & PH (learned filtration) & No & full & abstract graph & classification \\
PDGNN~\citep{yan2022neural} & PH (EPD surrogate) & No & full & abstract graph & classification \\
TopNets~\citep{verma2024topological} & PH (equivariant) & No & full & simplicial/geometric & classification/regression \\
\midrule
{\bf MG-Diff (Ours)} & DMT field with descending paths & Yes & none & abstract graph & graph generation (diffusion) \\
\bottomrule
\end{tabular}}
\end{table}

\label{sec:topology_ml}

Discrete curvature has also been used extensively to identify bottlenecks in message-passing, and is also viewed as the balanced Forman curvature of . Derived from Forman's combinatorial differential geometry, as Forman authored both discrete Morse theory and combinatorial Forman-Ricci curvature; the discrete Morse function we build follows, which derives from a similar area that develops Forman-Ricci curvature. The skeleton and curvature-based bottleneck analysis are thus two branches of one combinatorial framework. We note that both separte bride-liek edges from cycle-supporting edges, which is the dichotomy that governs the information flow. Note that in, a negatively curved edge is a bridge with little cycle support such as an oversquashing bottleneck, and in our decomposition such bridges lie in $E^\text{pair}$ as a bridge closes no cycle and therefore cannot be critical. On the other hand, cycle-supporting, positively curved regions are precisely where hte critical edges $E^\text{crit}$ concentrate.

\subsection{Graphs as Clique Complexes}
\label{sec:basic_defs}
Let $G=(V,E, X)$ be a %finite 
simple graph (no self-loops or multi-edges), where $V=\{v_0,\dots,v_n\}$ is the vertex set ($n = |V|$ for the number of vertices), %\cite{Harary}, %\YC{Do we need to highlight no self-loop here?} \Jen{@Yuzhou, not necessarily that G cannot, but a simplicial complex by definition cannot. So although the original G does not have to be restricted, we do need to highlight that in the final simplicial complex we do not. In my current code implementation we have that for $u, v\in E \Rightarrow$ complex$_$dict[1].append(frozenset([u,v]). It would make it so that dim(frozenset([v]) = len(v) -1 if u and v were the same. In other words downstream V path would treat it as a vertex not an edge}, 
$E \subseteq V \times V$ is the edge set ($m = |E|$ for the number of edges), and $X$ is the node feature matrix. We associate clique complex $\Lambda$ with $G$ which is the abstract simplicial complex whose simplices are the vertex sets of cliques in $G$: a $p$-simplex is a set of $p{+}1$ mutually adjacent vertices. Hence, vertices, edges, and triangles correspond to $0$-, $1$-, and $2$-simplices, respectively. Formally, $\Lambda$ is downward closed: if $\alpha\in \Lambda$, then every non-empty subset of $\alpha$ lies in $\Lambda$; $\dim(\alpha)=|\alpha|-1$, $\dim(\Lambda)=\max_{\alpha\in \Lambda}\dim(\alpha)$, and $\gamma<\alpha$ denotes a proper face. (See~\cite{hatcher2002algebraic,edelsbrunner2010computational} for background on simplicial complexes and homology.)

\section{Smooth Morse Theory on Manifolds}
\label{sec:smooth_morse}

\paragraph{Smooth Manifolds.}
Let $M$ be a smooth, compact $d$-dimensional manifold (possibly with boundary). We assume $M$ is equipped with a smooth Riemannian metric $\gamma$, which allows us to define gradients and gradient flows.

\begin{definition}[Critical point]
Let $M$ be a smooth $n$-dimensional manifold and let $f : M \to \mathbb{R}$ be a smooth function. A point $p \in M$ is called a \emph{critical point} of $f$ if the differential of $f$ vanishes at $p$, i.e., $df_p = 0$. Equivalently, in any local coordinate chart $(x_1,\dots,x_n)$ around $p$, all first partial derivatives vanish at $p$:
\[
\frac{\partial f}{\partial x_i}(p) = 0 \quad \text{for all } i=1,\dots,n.
\]
A point $p$ is called a \emph{regular point} if $df_p \neq 0$.
\end{definition}

\paragraph{Morse Functions.}
\begin{definition}[Morse Function]
A smooth function $f : M \to \mathbb{R}$ is a \emph{Morse function} if all of its critical points are non-degenerate. That is, for every $p \in M$ such that $\nabla f(p) = 0$, the Hessian $\mathrm{Hess}_p(f)$ is non-singular.
\end{definition}

Non-degeneracy implies that critical points are isolated and finite in number for compact $M$.

\paragraph{Index of a Critical Point.}
\begin{definition}[Morse Index]
Let $p$ be a critical point of a Morse function $f$. The \emph{Morse index} of $p$, denoted $\operatorname{ind}(p)$, is the number of negative eigenvalues of $\mathrm{Hess}_p(f)$.
\end{definition}

Geometrically, the index counts the number of independent directions in which $f$ decreases near $p$.

\paragraph{Gradient Flow.}
Given the Riemannian metric $\gamma$, the gradient vector field $\nabla f$ defines a flow $\varphi_t : M \to M$ by
\[
\frac{d}{dt} \varphi_t(x) = -\nabla f(\varphi_t(x)).
\]
Along flow lines, the function value strictly decreases except at critical points.

\paragraph{Descending and Ascending Manifolds.}
\begin{definition}[Descending (Stable) Manifold]
The \emph{descending manifold} of a critical point $p$ is
\[
W^s(p) = \{ x \in M \mid \lim_{t \to +\infty} \varphi_t(x) = p \}.
\]
\end{definition}

\begin{definition}[Ascending (Unstable) Manifold]
The \emph{ascending manifold} of $p$ is
\[
W^u(p) = \{ x \in M \mid \lim_{t \to -\infty} \varphi_t(x) = p \}.
\]
\end{definition}

For a Morse function, $W^s(p)$ is a smooth submanifold of dimension $d - \operatorname{ind}(p)$, while $W^u(p)$ has dimension $\operatorname{ind}(p)$.

\paragraph{Cell Decomposition and Topology.}
A fundamental theorem of Morse theory states that $M$ admits a CW decomposition with one $\operatorname{ind}(p)$-cell for each critical point $p$. Thus, Morse functions encode the topology of $M$ through their critical points.

\section{Cobordisms and Morse Functions}
\label{sec:cobordism_appendix}

\paragraph{Cobordism.}
\begin{definition}[Cobordism]
Two closed $(d-1)$-manifolds $M_0$ and $M_1$ are \emph{cobordant} if there exists a $d$-manifold $W$ such that
\[
\partial W = M_0 \sqcup M_1.
\]
\end{definition}

    One can view a cobordism as the entire process of the manifold $M_0$ as it goes through transitions to become $M_1$. Expanding on our analogy in Section \ref{sec:intro}, we can see the analogy in the transitions between a graph and a walking person in Figure \ref{fig:human_cobord}. This perspective can be thought of analogously to how Khovanov homology tracks topology-changing saddle cobordisms between link resolutions \cite{khovanov2000categorification, bar-natan2002khovanov}. 

\begin{figure}[t]
    \centering
    \includegraphics[width=1\linewidth]{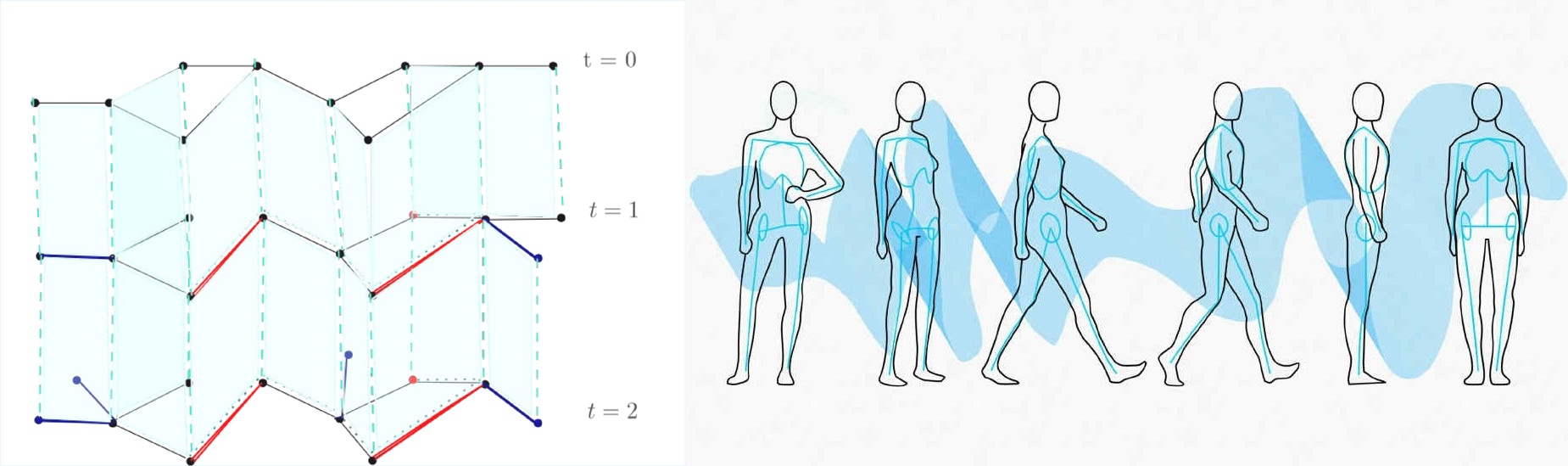}
    \caption{Side-by-side comparison of the Morse-guided graph evolution (left) and the corresponding human motion phases (right).}
    \label{fig:human_cobord}
\end{figure}

\paragraph{Morse Functions on Cobordisms.}
A Morse function $f : W \to [0,1]$ whose critical points lie in the interior of $W$ describes how $M_0$ is transformed into $M_1$ by attaching handles. Each critical point corresponds to a handle attachment of index equal to the Morse index.

This perspective is central in topology: \emph{changing topology occurs precisely at critical points}.

\paragraph{Deforming Morse Functions.}
Given two Morse functions $f_0, f_1 : M \to \mathbb{R}$, one may consider a generic homotopy (a smooth transition of one function to another) $f_t$. Such a homotopy may introduce or cancel pairs of critical points of adjacent indices, corresponding to handle creation or cancellation. This underlies the idea of moving between different Morse descriptions of the same space.

\section{Discrete Morse Theory}
\label{sec:discrete_morse_appendix}

We now turn to the discrete analogue of smooth Morse theory introduced by Forman~\cite{forman1998morse, forman2002users}, which replaces smooth manifolds with cell complexes.

\paragraph{Cell Complexes.}
Let $\Lambda$ be a finite regular CW complex (e.g., a simplicial complex or cubical complex). Denote by $\Lambda^{(p)}$ the set of $p$-cells.

\subsection{Discrete Morse Functions and Critical Simplices}

Let $\Lambda$ be a simplicial complex and let $f : \Lambda \to \mathbb{R}$ be a real-valued function defined on its simplices. For a $p$-simplex $\alpha \in \Lambda^p$, define the sets
\begin{equation}
U_f(\alpha) = \left\{ \beta \in \Lambda^{p+1} \,\middle|\, \alpha < \beta \text{ and } f(\beta) \leq f(\alpha) \right\},
\end{equation}
\begin{equation}
V_f(\alpha) = \left\{ \gamma \in \Lambda^{p-1} \,\middle|\, \gamma < \alpha \text{ and } f(\gamma) \geq f(\alpha) \right\}.
\end{equation}
Intuitively, $U_f(\alpha)$ contains cofaces of $\alpha$ with function value less than or equal to that of $\alpha$, while $V_f(\alpha)$ contains faces of $\alpha$ whose function value is greater than or equal to that of $\alpha$.

\begin{definition}[Discrete Morse Function]
A function $f : \Lambda \to \mathbb{R}$ is called a \emph{discrete Morse function} if for every simplex $\alpha \in \Lambda^p$,
\begin{equation}
|U_f(\alpha)| \leq 1 \quad \text{and} \quad |V_f(\alpha)| \leq 1.
\end{equation}
\end{definition}

\begin{definition}[Critical Cell]
A simplex $\alpha \in K^p$ is said to be \emph{critical} if
\begin{equation}
|U_f(\alpha)| = |V_f(\alpha)| = 0.
\end{equation}
Equivalently, $\alpha$ is critical if all cofaces $\beta$ of $\alpha$ satisfy $f(\beta) > f(\alpha)$ and all faces $\gamma$ of $\alpha$ satisfy $f(\gamma) < f(\alpha)$.
\end{definition}

Critical $p$-cells play the role of critical points of index $p$. They encode the essential topological features of the complex and form the backbone of the Morse-based filtration used to compute persistent homology.

\paragraph{Why a single path is sufficient.}
The full descending manifold can be large and unstable under benign perturbations of non-critical structure (e.g., local edge rewiring that does not alter $\mathcal{C}(f)$). A single descending path yields a \emph{compact and stable representative} that (i) anchors the critical cell to a concrete location in the $1$-skeleton, (ii) preserves the directionality of gradient-induced topological flow, and (iii) produces a fixed-size object suitable for diffusion guidance. Crucially, the event semantics reside in the existence and identity of critical cells and their induced flow direction; selecting a representative path constitutes a controlled compression that preserves these semantics while allowing variability in non-critical regions.

\paragraph{Discrete Gradient Vector Field.}
\label{sec:dis_gvf}
\begin{definition}[Discrete gradient vector field]
\label{def:dgvf}
A \emph{discrete gradient vector field} $\mathcal{V}$ is a collection of pairs $(\alpha^{(p)},\beta^{(p+1)})$ with $\alpha\subset\beta$, such that each cell appears in at most one pair. We require $\mathcal{V}$ to be \emph{acyclic}, i.e., there are no nontrivial closed $\mathcal{V}$-paths.
\end{definition}

% \begin{definition}[Filtration-compatibility]
% \label{def:filtration_compat_appendix}
% A discrete gradient vector field $\mathcal{V}$ is \emph{filtration-compatible} with $f$ if for every $(v,e)\in\mathcal{V}$ we have $f(v)=f(e)$.
% \end{definition}

\begin{definition}[Compatibility notions]
\label{def:compatibility_notion}
Let $\Lambda$ be a finite cell complex, $f:\Lambda\to\mathbb{R}$ a discrete Morse function, and $\mathcal{V}$ a discrete gradient vector field.

\begin{enumerate}
\item \textbf{$f$-admissible (induced-filtration compatible).}
We say $\mathcal{V}$ is \emph{$f$-admissible} if for every matched pair $(\sigma^{(p)},\tau^{(p+1)})\in\mathcal{V}$ we have
\[
f(\tau)\le f(\sigma).
\]
Equivalently, if $F_t:=\{\alpha\in \Lambda: f(\alpha)\le t\}$ is the sublevel-set filtration induced by $f$, then $\sigma\in F_t$ implies its paired coface (if any) also lies in $F_t$.

\item \textbf{Strict filtered matching (external filtration).}
Given an externally specified filtration $F:\Lambda_0\subset\cdots\subset \Lambda_T$ and entry time $t_F(\alpha)$, we say $\mathcal{V}$ is \emph{entry-time compatible} if every matched pair satisfies $t_F(\sigma)=t_F(\tau)$.
\end{enumerate}
\end{definition}

\paragraph{Directed Hasse diagram and $\mathcal{V}$-paths.} \label{sec:hasse}
Let $\Lambda$ be a finite regular CW complex, and let $\mathcal{V}$ be a discrete gradient vector field consisting of pairs $(\alpha^{(p)},\beta^{(p+1)})$ with $\alpha\prec\beta$ and each cell appearing in at most one pair. Define the \emph{directed Hasse diagram} $G_\mathcal{V}$ on the cells of $\Lambda$ by orienting each incidence $\alpha^{(p)} \prec \beta^{(p+1)}$ as
\[
\beta \to \alpha
\quad\text{unless}\quad
(\alpha,\beta)\in \mathcal{V},
\ \text{in which case}\ 
\alpha\to \beta.
\]
A \emph{$\mathcal{V}$-path} is a directed path in $G_\mathcal{V}$ alternating dimensions:
\[
\alpha_0^{(p)} \prec \beta_0^{(p+1)} \succ \alpha_1^{(p)} \prec \beta_1^{(p+1)} \succ \cdots.
\]
A cell is \emph{critical} if it does not appear in any pair of $\mathcal{V}$.

\paragraph{Forman's Morse Inequalities.}
Let $c_p$ be the number of critical $p$-cells. For all $p$,
\[
c_p \ge \beta_p(\Lambda)
\]
and the alternating sum equality
\[
\sum_p (-1)^p c_p = \chi(\Lambda)
\]
holds, where $\chi(\Lambda)$ is the Euler characteristic.

\paragraph{Topology Preservation.}
Forman showed that a complex $\Lambda$ admits a homotopy-equivalent CW complex with exactly one $p$-cell for each critical $p$-cell. Thus, discrete Morse theory provides a principled way to simplify complexes while preserving topology: the critical cells determine a minimal CW model up to homotopy.

\paragraph{Avoiding a degenerate choice of Morse function $f$.}
A cautionary example is the \emph{dimension function} $f(\sigma)=\dim(\sigma)$: by direct inspection of Forman's conditions, this choice makes $U_\alpha=V_\alpha=\varnothing$ for every simplex $\alpha$, so \emph{every simplex is critical} (i.e., the skeleton is uselessly dense). This pathology is well-known and explicitly noted in prior work on discrete Morse filtrations for clique complexes~\cite{kannan2019persistent}.

\section{Discrete Flow and Descending Sets}
\label{sec:discrete_flow_descending}

\begin{definition}[$\mathcal{V}$-paths (Forman)]
\label{def:Vpath_general}
Let $\Lambda$ be a finite regular CW complex and let $\mathcal{V}$ be a discrete gradient vector field on $\Lambda$.
A \emph{$\mathcal{V}$-path} is a sequence of cells
\[
\sigma_0^{(p)} \prec \tau_0^{(p+1)} \succ \sigma_1^{(p)} \prec \tau_1^{(p+1)} \succ \cdots
\prec \tau_{m-1}^{(p+1)} \succ \sigma_m^{(p)}
\]
such that $(\sigma_i,\tau_i)\in\mathcal{V}$ for each $i$ and $\sigma_{i+1}$ is a codimension-$1$ face of $\tau_i$ with $\sigma_{i+1}\neq \sigma_i$.
\end{definition}

\begin{definition}[Descending set / discrete unstable manifold]
\label{def:descending_set}
Fix $\mathcal{V}$ on $\Lambda$. For a critical $p$-cell $\sigma$, define its \emph{descending set}
\[
\mathcal{D}_{\mathcal{V}}(\sigma)
:= \{\, \tau \in \Lambda \;:\; \text{ exists $\mathcal{V}$-path at $\sigma$ reaching $\tau$}\,\}.
\]
Equivalently, $\tau\in\mathcal{D}_{\mathcal{V}}(\sigma)$ if $\tau$ is reachable from $\sigma$ by repeatedly following the directed incidences induced by $\mathcal{V}$ in the Hasse diagram.
\end{definition}

\begin{remark}[Terminology]
Many authors refer to $\mathcal{D}_{\mathcal{V}}(\sigma)$ as the \emph{discrete descending region} or \emph{descending manifold} (discrete unstable manifold). We use \emph{descending set} to emphasize that this is a combinatorial reachability object rather than a smooth submanifold.
\end{remark}

\paragraph{Why descending sets appear here.}
Descending sets are the discrete analogue of unstable manifolds: they describe which parts of the complex ``flow out of'' a critical cell under $\mathcal{V}$. In our context, the Morse skeleton (critical cells plus the incidence structure induced by $\mathcal{V}$) serves as a compact structural summary used to guide diffusion.

\begin{figure}[ht]
  \centering
  \includegraphics[width=0.6\linewidth]{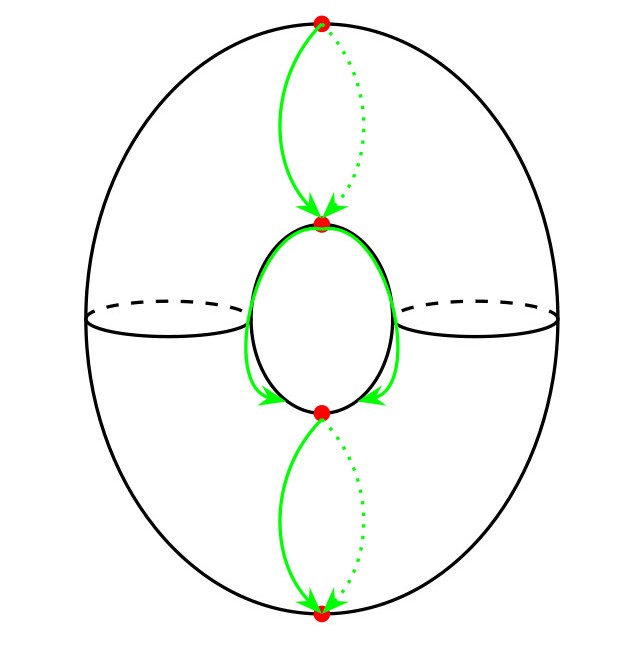}
  \caption{The torus viewed as a cobordism between two empty sets. The green marks indicate the gradient flow lines that constitute the connections between critical points along the gradient vector field.}
  \label{fig:torus_descending}
\end{figure}

\section{Specialization to Graphs: Acyclicity and Paired-Edge Orientation}
\label{sec:graph_acyclic_orientation}

Henceforth let $G=(V,E)$ be a finite graph, viewed as a $1$-dimensional CW complex, and let $\mathcal{V}$ be a discrete gradient vector field on $G$, i.e., a collection of disjoint vertex--edge pairs $(v,e)$ with $v\in e$.

\begin{definition}[Acyclicity (Forman $\mathcal{V}$-path condition on graphs)]
\label{def:acyclic_graph}
The vector field $\mathcal{V}$ is \emph{acyclic} if there is no nontrivial closed $\mathcal{V}$-path. Equivalently, there do not exist vertices $v_0,\dots,v_{k-1}$ and edges $e_0,\dots,e_{k-1}$ with $k\ge 1$ such that for each $i$ (indices mod $k$),
\[
(v_i,e_i)\in\mathcal{V}, \qquad e_i=\{v_i,v_{i+1}\},
\]
and $v_k=v_0$.
\end{definition}

\begin{definition}[Paired-edge orientation and the induced directed graph]
\label{def:Gamma_fullgraph}
Define the paired-edge set
\[
E^{\mathrm{pair}} := \{\, e\in E : \exists v\in V \text{ such that } (v,e)\in\mathcal{V}\,\}.
\]
For each paired edge $e=\{v,w\}\in E^{\mathrm{pair}}$, exactly one endpoint is paired to it; write $(v,e)\in\mathcal{V}$. Define a directed graph $\Gamma(\mathcal{V})$ on vertex set $V$ by adding the directed edge $v\to w$ for each such pair $(v,e)$.
\end{definition}

\begin{lemma}[Outdegree bound]
\label{lem:outdegree_bound}
Every vertex in $\Gamma(\mathcal{V})$ has outdegree at most $1$.
\end{lemma}

\begin{proof}
Each vertex appears in at most one pair $(v,e)\in\mathcal{V}$ by definition of a discrete gradient vector field. Thus, from any vertex $v$ there is at most one paired edge $e$ with $(v,e)\in\mathcal{V}$, and hence at most one outgoing edge in $\Gamma(\mathcal{V})$.
\end{proof}

\begin{lemma}[Acyclicity implies no directed cycles]
\label{lem:acyclic_implies_DAG}
If $\mathcal{V}$ is acyclic in the sense of Definition~\ref{def:acyclic_graph}, then $\Gamma(\mathcal{V})$ contains no directed cycles.
\end{lemma}

\begin{proof}
A directed cycle $v_0\to v_1\to \cdots \to v_{k-1}\to v_0$ in $\Gamma(\mathcal{V})$ corresponds exactly to a closed $\mathcal{V}$-path $(v_i,e_i)\in\mathcal{V}$ with $e_i=\{v_i,v_{i+1}\}$, contradicting Definition~\ref{def:acyclic_graph}.
\end{proof}

\begin{definition}[Anchor map on the full graph]
\label{def:anchor_fullgraph}
Assume $\mathcal{V}$ is acyclic. Since $\Gamma(\mathcal{V})$ is a finite directed graph with outdegree $\le 1$ and no directed cycles (Lemmas~\ref{lem:outdegree_bound}--\ref{lem:acyclic_implies_DAG}), every vertex reaches a terminal sink by iterating outgoing edges. Define the \emph{anchor map} $r:V\to V$ by letting $r(v)$ be the unique sink reached from $v$ in $\Gamma(\mathcal{V})$.
\end{definition}

\begin{remark}[When anchors are useful]
Anchors are only needed if one later uses a ``component representative'' tokenization (e.g., anchor tokens). Our main MoMoE split can be implemented using only $(E^{\mathrm{pair}},E^{\mathrm{crit}})$ and does not require anchors.
\end{remark}

\section{$H_0$-Faithfulness and Semantic Guarantees for the MoMoE Split}
\label{sec:H0_faithful_semantics}

\begin{definition}[$H_0$-faithfulness (component agreement)]
\label{def:H0_faithful}
We say $\mathcal{V}$ is \emph{$H_0$-faithful} on $G$ if the paired-edge subgraph
\[
P^- := (V,E^{\mathrm{pair}})
\]
induces the same connected-component partition on $V$ as the full graph $G=(V,E)$. Equivalently, for all $u,w\in V$,
\[
u \sim_{P^-} w \quad \Longleftrightarrow \quad u \sim_G w.
\]
\end{definition}

\paragraph{Scope of the assumption.}
$H_0$-faithfulness is a \emph{sufficient} condition to derive %precise semantic 
the following guarantees on paired edges, namely, ensuring that paired edges preserve the component structure of $G$. Our construction always outputs a valid discrete gradient vector field, but it does not explicitly enforce $H_0$-faithfulness in all cases. Enforcing $H_0$-faithfulness is straightforward in principle (e.g., by pairing along a spanning forest), but we do not incorporate such enforcement in the current methodology.

\begin{definition}[Forested pairing (optional strengthening)]
\label{def:forested_pairing}
We say $\mathcal{V}$ is \emph{forested} if the undirected graph $P^-=(V,E^{\mathrm{pair}})$ is cycle-free (i.e., a forest).
\end{definition}

\begin{remark}[Why we separate $H_0$-faithfulness from being a forest]
\label{rem:H0_vs_forest}
Acyclicity of $\mathcal{V}$ (no closed $\mathcal{V}$-paths) implies $\Gamma(\mathcal{V})$ has no \emph{directed} cycles (Lemma~\ref{lem:acyclic_implies_DAG}), but it does \emph{not} by itself rule out \emph{undirected} cycles in the paired-edge subgraph $P^-$. Thus, to conclude that $P^-$ is a spanning forest one must assume forestedness (Definition~\ref{def:forested_pairing}) in addition to $H_0$-faithfulness.
\end{remark}

\begin{proposition}[Critical edges carry cycle rank under component-faithful forested pairing]
\label{thm:cycle_rank_semantic}
Let $G=(V,E)$ be a finite graph, and let $\mathcal{V}$ be an acyclic discrete gradient vector field on $G$. Assume:
\begin{enumerate}
\item $\mathcal{V}$ is $H_0$-faithful (Definition~\ref{def:H0_faithful}), and
\item $\mathcal{V}$ is forested (Definition~\ref{def:forested_pairing}).
\end{enumerate}
Then:
\begin{enumerate}[label=(\alph*)]
\item $P^-=(V,E^{\mathrm{pair}})$ is a spanning forest of $G$ (one tree per connected component);
\item consequently,
\[
\beta_1(G)=|E^{\mathrm{crit}}|
\qquad \text{where } E^{\mathrm{crit}}:=E\setminus E^{\mathrm{pair}}.
\]
\end{enumerate}
\end{proposition}

\begin{proof}
Let $c$ be the number of connected components of $G$. By $H_0$-faithfulness, $P^-$ has the same component partition as $G$, hence $P^-$ also has $c$ components. By forestedness, each component of $P^-$ is a tree; therefore $|E^{\mathrm{pair}}|=|V|-c$. Now use the standard cycle-rank identity for finite graphs:
\[
\beta_1(G)=|E|-|V|+c.
\]
Since $|E|=|E^{\mathrm{pair}}|+|E^{\mathrm{crit}}|$ and $|E^{\mathrm{pair}}|=|V|-c$, we obtain
\[
\beta_1(G)=(|E^{\mathrm{pair}}|+|E^{\mathrm{crit}}|)-|V|+c  \]\[ = (|V|-c+|E^{\mathrm{crit}}|)-|V|+c =|E^{\mathrm{crit}}|.
\]
\end{proof}

\begin{remark}[Interpretation for MoMoE routing]
Under the optional strengthened pairing regime of Proposition ~\ref{thm:cycle_rank_semantic}, $A_{\alpha'}$ (paired edges) is a cycle-free connectivity backbone within each component, while $A_{\alpha}$ (critical edges) accounts exactly for the cycle rank ($H_1$ content) of the graph. This supplies a clean semantics for the MoMoE split when one enforces forestedness.
\end{remark}

\section{Proofs and Additional Details on Stability of Morse-Skeleton Guidance}
\label{sec:stability_skeleton_guidance}

\paragraph{Core design principle.}
MG-Diff conditions the reverse diffusion dynamics on a \emph{Morse skeleton} extracted from $\mathcal{V}$, rather than on all edges/cells of the input. Accordingly, our stability guarantees are statements about invariance/robustness of the \emph{conditioning mechanism} (guidance), not global stability of discrete Morse theory itself.

\subsection{Abstract Setup}

Let $\mathbb{S}$ be the space of \emph{states} read by the extractor (graph structure plus any scores/flags it consumes), and let $s\in\mathbb{S}$ denote a state. Let us assume:
\begin{itemize}
\item a deterministic skeleton extractor $\mathrm{Skel}:\mathbb{S}\to Cl(G)$ (all tie-breaking fixed),
\item a deterministic guidance map $\Psi:Cl(G)\times\mathcal{X}\to\mathcal{H}$,
\item a fixed historical input $X^h\in\mathcal{X}$, and define $\bar H(s):=\Psi(\mathrm{Skel}(s),X^h)$,
\item a denoiser $\epsilon_\theta(X_k,\bar H,k)$ and reverse-step kernel $p_\theta(\cdot\mid X_k,\bar H,k)$ with covariance independent of $s$.
\end{itemize}
\begin{remark}
    We assume the reverse kernel $p_\theta(\cdot \mid X_k, \bar{H}, k)$ is Gaussian with covariance $\Sigma_k$ that depends only on the timestep $k$, not on the state $X_k$ or guidance $\bar{H}$. This is standard for DDPM-style diffusion models where $\Sigma_k$ follows a fixed schedule~\cite{ho2020denoising}, but we state it explicitly.

\end{remark}
\begin{definition}[Skeleton-preserving perturbation]
\label{def:skeleton_preserving}
A perturbation $s\mapsto s'$ is \emph{skeleton-preserving} if $\mathrm{Skel}(s')=\mathrm{Skel}(s)$.
\end{definition}

\begin{lemma}[Exact invariance under skeleton-preserving perturbations]
\label{lemma:exact_invariance_guidance}
Fix $X^h \in \mathcal{X}$. If $s' \in \mathbb{S}$ is skeleton-preserving relative to $s \in \mathbb{S}$, then $\bar{H}(s') = \bar{H}(s)$. Consequently, for every timestep $k$ and every noised state $X_k$,
\[
    p_\theta(\cdot \mid X_k, \bar{H}(s'), k) = p_\theta(\cdot \mid X_k, \bar{H}(s), k).
\]
\end{lemma}

\begin{proof}
We establish the result via the composition chain
\[
    s' \;\xrightarrow{\;\mathrm{Skel}\;}\; \mathrm{Skel}(s') = \mathrm{Skel}(s) \;\xrightarrow{\;\Psi(\cdot, X^h)\;}\; \bar{H}(s') \] \[= \bar{H}(s) \;\xrightarrow{\;p_\theta\;}\; \text{identical distributions}.
\]

By Definition~\ref{def:skeleton_preserving}, the perturbation $s \mapsto s'$ is skeleton-preserving, so $\mathrm{Skel}(s') = \mathrm{Skel}(s)$. Since $\Psi$ is deterministic and the historical input $X^h$ is fixed, it follows that
\[
    \bar{H}(s') = \Psi(\mathrm{Skel}(s'), X^h) = \Psi(\mathrm{Skel}(s), X^h) = \bar{H}(s).
\]
The reverse-step kernel $p_\theta(\cdot \mid X_k, \bar{H}, k)$ is fully determined by the triple $(X_k, \bar{H}, k)$. Since $X_k$ and $k$ are identical in both cases and $\bar{H}(s') = \bar{H}(s)$, the two conditional distributions coincide.
\end{proof}

%-----------------------------
\subsubsection{Interpretation and implications}

\begin{remark}[The skeleton as a sufficient characteristic]
\label{rem:sufficient_statistic}
Lemma ~\ref{lemma:exact_invariance_guidance} establishes that the Morse skeleton $\mathrm{Skel}(s)$ is a \emph{sufficient characteristic} for the guidance mechanism: any information in the state $s$ that does not affect the skeleton is invisible to the reverse diffusion process. Formally, the guidance  $\bar{H}: \mathbb{S} \to \mathcal{H}$ factors as $\bar{H} = \Psi(\cdot, X^h) \circ \mathrm{Skel}$.
and consequently $\bar{H}(s) = \bar{H}(s')$ whenever $\mathrm{Skel}(s) = \mathrm{Skel}(s')$.
\end{remark}

\begin{remark}[Practical consequences]
\label{rem:practical_consequences}
We see that under these assumptions that perturbations to edge weights, node features, or auxiliary scores that do not alter the extracted skeleton leave the model's generative distribution unchanged, providing a form of noise robustness. Further, the skeleton acts as an information bottleneck by discarding structural details irrelevant to generation while preserving the topological features that guide the diffusion process. 
\end{remark}

\begin{remark}[Scope and limitations]
\label{rem:scope_limitations}
Theorem~\ref{lemma:exact_invariance_guidance} is, by design, a conditional statement: \emph{if} a perturbation preserves the skeleton, \emph{then} the generative distribution is invariant. The theorem does not address \emph{when} perturbations are skeleton-preserving, nor \emph{how much} the distribution changes when the skeleton does change. In other words, Theorem \ref{lemma:exact_invariance_guidance} does not characterize \emph{which} perturbations are skeleton-preserving. We now develop a quantitative sufficient condition: if a perturbation is small relative to a computable \emph{margin}, the skeleton is guaranteed unchanged.
\end{remark}

%========================================================
% Margin Stability: Sufficient Conditions for Skeleton Preservation
%========================================================

\subsection{A sufficient condition for skeleton preservation: margin stability}
\label{sec:margin_stability}

Our condition follows from the observation that any deterministic extractor (sequential matching, sorting, thresholding, or similar, etc.) ultimately reduces to a finite collection of strict comparisons, and if none of these comparisons flip, the output cannot change.

%-----------------------------
\subsubsection{Decision representation}
\label{sec:decision_representation}

We formalize the structure of comparison-based extractors through the following assumptions:

\begin{assumption}[Decision representation]
\label{ass:decision_representation}
There exist:
\begin{itemize}
    \item a \emph{feature map} $\varphi: \mathbb{S} \to \mathbb{R}^d$ that extracts all scalar quantities examined by the extractor (scores, weights, tie-breaking indices, etc.),
    \item a finite \emph{decision index set} $\mathcal{J}$ enumerating the comparisons performed by the extractor,
    \item for each $j \in \mathcal{J}$, functions $u_j, v_j: \mathbb{R}^d \to \mathbb{R}$ representing the left and right sides of comparison $j$,
\end{itemize}
such that the extractor output $\mathrm{Skel}(s)$ is fully determined by the sign pattern of the \emph{decision gaps}
\[
    \Delta_j(s) := u_j(\varphi(s)) - v_j(\varphi(s)), \qquad j \in \mathcal{J},
\]
through a deterministic map $F: \{-1, +1\}^{\mathcal{J}} \to Cl(G)$, i.e.,
\[
    \mathrm{Skel}(s) = F\bigl((\mathrm{sign}(\Delta_j(s)))_{j \in \mathcal{J}}\bigr).
\]
We restrict attention to states $s \in \mathcal{S}$ satisfying $\Delta_j(s) \neq 0$ for all $j \in \mathcal{J}$ (the no-ties condition).
\end{assumption}

\begin{remark}[Interpretation of the decision representation]
\label{rem:decision_interpretation}
The feature map $\varphi$ collects all numerical inputs to the extractor into a single vector while the decision index set $\mathcal{J}$ enumerates every comparison the extractor performs. For a sequential pairing algorithm that processes edges in a fixed order and checks Flag variables, this comprises $O(m)$ explicit comparisons: for each edge, one comparison determines endpoint ordering (which vertex has larger score), and the Flag lookup determines the branching outcome. The sequential dependence through Flag variables means that later decisions depend implicitly on earlier ones, but the total number of comparisons remains finite. The functions $u_j$ and $v_j$ extract the quantities being compared: if decision $j$ asks whether $f(v_1) > f(v_2)$, then $u_j(\varphi(s)) = f(v_1)$ and $v_j(\varphi(s)) = f(v_2)$ under state $s$.

The deterministic map $F$ encodes the extractor's logic: once all comparison outcomes (signs) are known, the skeleton is determined. This holds for any algorithm without internal randomness.
\end{remark}

\textbf{Why Is This Construction Valid?}
As shown by~\cite{kannan2019persistent}, the procedure yields a valid discrete Morse function in the sense of Forman's discrete Morse theory~\cite{forman1998morse}. Each vertex participates in at most one gradient pairing due to the use of
flag variables, and each edge is assigned a function value that ensures pairing with at most one endpoint.
By construction, no simplex admits more than one coface with lower or equal
function value or more than one face with higher or equal function value,
and the induced matching on the Hasse diagram is acyclic.

\begin{remark}[Scope of the decision representation]
\label{rem:decision_scope}
Assumption~\ref{ass:decision_representation} is satisfied by extractors based on sequential matching, sorting, thresholding, and argmax/argmin operations, since these reduce entirely to pairwise comparisons. The assumption fails for extractors involving continuous optimization (e.g., gradient descent to a local minimum) or soft differentiable relaxations where the output varies smoothly with inputs. 

For the discrete Morse extractor in MG-Diff, vertex scores are assigned based on degree via $g(v) = \deg_{\max} - \deg(v) + \epsilon(v)$ (where $\epsilon(v)$ is a small tie-breaking perturbation) and edges are processed sequentially with a Flag-based construction that determines pairing. Each edge makes a binary decision based on whether its higher-score endpoint is already paired, yielding a finite sequence of comparisons. Assumption~\ref{ass:decision_representation} therefore holds provided tie-breaking is deterministic (e.g., lexicographic ordering on vertex indices). If randomized tie-breakers are used, the analysis applies conditionally on a fixed realization of the random seed.
\end{remark}

\begin{remark}[The no-ties condition]
\label{rem:no_ties}
The restriction $\Delta_j(s) \neq 0$ for all $j$ excludes states where exact ties occur. This is a genericity assumption: the set of tie-inducing states is small in any reasonable parameterization. In practice, when vertex scores are degree-based, ties occur precisely when two adjacent vertices have identical degree. Such ties can be broken deterministically, which ensures strict inequalities without affecting the topological content of the construction. We state this condition explicitly to avoid degenerate edge cases in the analysis.
\end{remark}

%-----------------------------
\subsubsection{Decision margin and Lipschitz regularity}

We now introduce the central quantity governing stability.

\begin{definition}[Decision margin]
\label{def:decision_margin}
For a state $s \in \mathbb{S}$ satisfying the no-ties condition, the \emph{decision margin} is
\[
    m(s) := \min_{j \in \mathcal{J}} |\Delta_j(s)|.
\]
\end{definition}
Intuitively, we can think of the decision margin as the smallest gap across all comparisons in the extractor. If one wants to change the skeleton, at least one comparison must flip, which requires pushing some $\Delta_j$ across zero. The margin $m(s)$ quantifies the minimum perturbation required to reach the nearest decision boundary, so a large margin indicates that all decisions are comfortable, while a small margin indicates that at least one comparison is nearly tied and could flip under small perturbations.

Crucially (if one were given access to the extractor's internal comparisons) one can evaluate each $\Delta_j(s)$ and take the minimum, obtaining a concrete numerical certificate of local stability. For the MG-Diff extractor, this reduces to computing the minimum score gap $|g(\gamma_0(e)) - g(\gamma_1(e))|$ across all edges $e \in E$, where $\gamma_0(e)$ and $\gamma_1(e)$ denote the higher- and lower-scoring endpoints of $e$.

\begin{assumption}[Lipschitz decision gaps]
\label{ass:lipschitz_gaps}
We assume there exists a constant $L_\Delta < \infty$ such that for all $j \in \mathcal{J}$ and all $s, s' \in \mathbb{S}$ if $||\cdot||$ is $l_\infty$,
\[
    |\Delta_j(s') - \Delta_j(s)| \leq L_\Delta \, \|\varphi(s') - \varphi(s)\|.
\]
\end{assumption}

\begin{remark}[Verification of Lipschitz regularity]
\label{rem:lipschitz_verification}
in Assumption ~\ref{ass:lipschitz_gaps} we assert that decision gaps cannot change faster than $L_\Delta$ times the feature displacement, which holds automatically when $u_j$ and $v_j$ are Lipschitz functions of the feature vector.

\paragraph{Why Assumption~\ref{ass:lipschitz_gaps} is reasonable.}
Our extractor is implemented by a finite sequence of comparisons between scalar quantities derived from the input
(e.g., vertex/edge scores, weights, and tie-breaking indices). Each decision gap $\Delta_j$ is a difference of two
such scalars. In our setting these scalars are constructed from bounded-degree algebraic operations
(addition/subtraction, $\min/\max$ over finite sets, and fixed-order sorting/tie-breaking after stable perturbation),
all of which are Lipschitz in $\ell_\infty$ on bounded domains. Consequently each $\Delta_j$ is Lipschitz with some
finite constant, and we may take $L_\Delta:=\max_{j\in\mathcal{J}} L_j$.
When the pipeline includes non-smooth steps (e.g., $\min$), Lipschitz continuity still holds (though not differentiability),
which is exactly what the margin argument requires.

For the MG-Diff extractor in particular, the relevant comparisons are vertex score differences. If we recall that $g(v) = \deg_{\max} - \deg(v) + \varepsilon_v$, the decision gap for an edge $e = \{v_1, v_2\}$ is
\[
    \Delta_j = g(v_1) - g(v_2) = \bigl(\deg(v_2) - \deg(v_1)\bigr) + \bigl(\varepsilon_{v_1} - \varepsilon_{v_2}\bigr),
\]
where the $\deg_{\max}$ terms cancel. So under a perturbation $s \mapsto s'$ that modifies tie-breaker values (degrees are fixed by graph structure), we have
\[
    \Delta_j(s') - \Delta_j(s) = \bigl(\varepsilon'_{v_1} - \varepsilon_{v_1}\bigr) - \bigl(\varepsilon'_{v_2} - \varepsilon_{v_2}\bigr).
\]
Followed by the triangle inequality,
\[
    |\Delta_j(s') - \Delta_j(s)| \leq |\varepsilon'_{v_1} - \varepsilon_{v_1}| + |\varepsilon'_{v_2} - \varepsilon_{v_2}| \leq 2\|\varphi(s') - \] \[ \varphi(s)\|_\infty \leq 2\|\varphi(s') - \varphi(s)\|,
\]
we find that $L_\Delta \leq 2$. This small Lipschitz constant therefore reflects the linear structure of degree-based scoring with additive tie-breakers.
\end{remark}

\subsubsection{Proof of Theorem~\ref{lem:margin_preserves_skel}: the margin stability guarantee}
\label{app:proof_marstab}

We are now ready to proof Theorem~\ref{lem:margin_preserves_skel}, noting that the assumptions on decision representations and Lipschitz conditions on the decision gaps, stated in Theorem~\ref{lem:margin_preserves_skel}, are restated as~(\ref{ass:decision_representation}) and~(\ref{ass:lipschitz_gaps}), respectively.
%For simplicity we restated it in terms of  Assumptions~\ref{ass:decision_representation} and~\ref{ass:lipschitz_gaps}.
%build toward, state, and prove Lemma \ref{margin_lemma}.

%\begin{lemma*}[Margin implies skeleton preservation]
%\label{lem:margin_preserves_skel}
%Suppose Assumptions~\ref{ass:decision_representation} and~\ref{ass:lipschitz_gaps} hold. Fix a state $s \in \mathbb{S}$ with $m(s) > 0$. If $s' \in \mathbb{S}$ satisfies
%\[    \|\varphi(s') - \varphi(s)\| <\frac{m(s)}{L_\Delta},
%\]
%then $\mathrm{Skel}(s') = \mathrm{Skel}(s)$.
%\end{lemma*}

\begin{proof}[Proof of Theorem~\ref{lem:margin_preserves_skel}]
We show that every decision $j \in \mathcal{J}$ has the same sign at $s'$ as at $s$, from which the conclusion follows by determinism of $F$.

Fix an arbitrary $j \in \mathcal{J}$. In view of
the assumptions on decision representations~(\ref{ass:decision_representation}) and Lipschitz conditions~(\ref{ass:lipschitz_gaps}), we obtain
\[
    |\Delta_j(s') - \Delta_j(s)| \leq L_\Delta \, \|\varphi(s') - \varphi(s)\| < L_\Delta \cdot \frac{m(s)}{L_\Delta} = m(s).
\]
By definition of the decision margin, $m(s) = \min_{j' \in \mathcal{J}} |\Delta_{j'}(s)| \leq |\Delta_j(s)|$, so
\[
    |\Delta_j(s') - \Delta_j(s)| < m(s) \leq |\Delta_j(s)|.
\]
This inequality implies that $\Delta_j(s')$ lies strictly within distance $|\Delta_j(s)|$ of $\Delta_j(s)$. Since $\Delta_j(s) \neq 0$ by the no-ties condition, the open interval $(\Delta_j(s) - |\Delta_j(s)|, \Delta_j(s) + |\Delta_j(s)|)$ does not contain zero: if $\Delta_j(s) > 0$, the interval is $(0, 2\Delta_j(s))$, and if $\Delta_j(s) < 0$, the interval is $(2\Delta_j(s), 0)$. In either case, $\Delta_j(s')$ lies in this interval and therefore has the same sign as $\Delta_j(s)$.

Since $j$ was arbitrary, the sign pattern $(\mathrm{sign}(\Delta_j(s')))_{j \in \mathcal{J}}$ equals $(\mathrm{sign}(\Delta_j(s)))_{j \in \mathcal{J}}$. By Assumption~\ref{ass:decision_representation}, $\mathrm{Skel}(s') = F((\mathrm{sign}(\Delta_j(s')))_j) = F((\mathrm{sign}(\Delta_j(s)))_j) = \mathrm{Skel}(s)$.
\end{proof}

%-----------------------------
\subsubsection{Interpretation and practical implications}

\begin{remark}[Stability radius]
\label{rem:stability_radius}
Theorem~\ref{lem:margin_preserves_skel} defines a \emph{stability radius} around each state:
\[
    r(s) := \frac{m(s)}{L_\Delta}.
\]
Within the open ball $\{s' \in \mathbb{S} : \|\varphi(s') - \varphi(s)\| < r(s)\}$, the skeleton is guaranteed unchanged, and so the generative distribution should remain invariant to local perturbation in this manner.

For the MG-Diff extractor with $L_\Delta \leq 2$, the stability radius is at least $m(s)/2$. If the minimum score gap between adjacent vertices is $\delta > 0$, then perturbations to vertex scores of magnitude less than $\delta/2$ are guaranteed to preserve the skeleton.
\end{remark}

\begin{remark}[Dependence on the margin]
\label{rem:margin_dependence}
The stability guarantee is only as strong as the margin $m(s)$. If the extractor frequently produces near-ties---for instance, when many vertices share the same degree---the margin is small, the stability radius $r(s)$ is correspondingly small, and the guarantee becomes weak. Conversely, if all decisions are made with comfortable gaps, the margin is large and meaningful robustness is assured.

This is not a flaw in the theory but rather an accurate reflection of reality: an extractor operating near decision boundaries \emph{is} sensitive to small perturbations, and the analysis correctly captures this. For degree-based scoring, the margin depends on the degree distribution of the input graph. Graphs with heterogeneous degree sequences yield larger margins than regular or near-regular graphs where many vertices have identical degree.
\end{remark}

\begin{remark}[Role of the Lipschitz constant]
\label{rem:lipschitz_role}
The Lipschitz constant $L_\Delta$ appears in the denominator of the stability radius, so a large $L_\Delta$ shrinks the guaranteed stability region. For the MG-Diff extractor based on degree scoring, $L_\Delta \leq 2$ is a small constant arising from the linear structure of coordinate differences. More complex scoring functions (e.g., involving products or nonlinear transformations of features) would yield larger Lipschitz constants and correspondingly weaker stability guarantees.
\end{remark}

\begin{remark}[Computability and certification]
\label{rem:computability}
Unlike many stability results in machine learning that provide only existential guarantees, the margin stability framework yields a \emph{computable certificate}. Given a state $s$ and access to the extractor's decision structure, one can enumerate the decision gaps $\{\Delta_j(s)\}_{j \in \mathcal{J}}$, compute $m(s) = \min_j |\Delta_j(s)|$, and obtain an explicit stability radius $r(s) = m(s)/L_\Delta$. For MG-Diff, this amounts to computing the minimum absolute difference $|g(u) - g(v)|$ over all edges $\{u,v\} \in E$, a linear-time operation. This enables post-hoc verification of robustness for individual inputs, a property that may be valuable in safety-critical applications.
\end{remark}

\section{Worked Examples}
\label{sec:worked_ex}

\subsection{Worked Example: Discrete Morse Flow on a Two-Cycle Graph}
\label{sec:worked_example}

We illustrate the construction of a discrete Morse function as follows.
Consider the graph $G = (V, E)$ shown in Fig.~\ref{fig:example_graph}, consisting of two 4-cycles 
connected by a bridge edge, with two pendant vertices.

\begin{figure}[t]
\centering
\begin{tikzpicture}[
    vertex/.style={circle, fill=black, minimum size=4pt, inner sep=0pt},
    edge/.style={thick},
    every label/.style={font=\scriptsize}
]

% Left cycle
\node[vertex] (0) at (0, 1) {};
\node[vertex] (1) at (0, 0) {};
\node[vertex] (2) at (1, 0) {};
\node[vertex] (3) at (1, 1) {};

% Right cycle
\node[vertex] (4) at (2.5, 0) {};
\node[vertex] (5) at (3.5, 0) {};
\node[vertex] (6) at (3.5, 1) {};
\node[vertex] (7) at (2.5, 1) {};

% Pendant vertices
\node[vertex] (8) at (-0.7, -0.5) {};
\node[vertex] (9) at (4.2, -0.5) {};

% Labels
\node[above, font=\scriptsize] at (0) {$v_0$};
\node[below, font=\scriptsize] at (1) {$v_1$};
\node[below, font=\scriptsize] at (2) {$v_2$};
\node[above, font=\scriptsize] at (3) {$v_3$};
\node[below, font=\scriptsize] at (4) {$v_4$};
\node[below, font=\scriptsize] at (5) {$v_5$};
\node[above, font=\scriptsize] at (6) {$v_6$};
\node[above, font=\scriptsize] at (7) {$v_7$};
\node[below, font=\scriptsize] at (8) {$v_8$};
\node[below, font=\scriptsize] at (9) {$v_9$};

% Left cycle edges
\draw[edge] (0) -- (1);
\draw[edge] (1) -- (2);
\draw[edge] (2) -- (3);
\draw[edge] (3) -- (0);

% Bridge
\draw[edge] (2) -- (4);

% Right cycle edges
\draw[edge] (4) -- (5);
\draw[edge] (5) -- (6);
\draw[edge] (6) -- (7);
\draw[edge] (7) -- (4);

% Pendant edges
\draw[edge] (1) -- (8);
\draw[edge] (5) -- (9);

\end{tikzpicture}

\vspace{1em}

\begin{tikzpicture}[
    vertex/.style={circle, fill=black, minimum size=4pt, inner sep=0pt},
    critical vertex/.style={circle, fill=red, minimum size=4pt, inner sep=0pt},
    edge/.style={thick},
    critical edge/.style={thick, red}
]

% Left cycle
\node[vertex] (0) at (0, 1) {};
\node[vertex] (1) at (0, 0) {};
\node[vertex] (2) at (1, 0) {};
\node[vertex] (3) at (1, 1) {};

% Right cycle
\node[vertex] (4) at (2.5, 0) {};
\node[critical vertex] (5) at (3.5, 0) {};
\node[vertex] (6) at (3.5, 1) {};
\node[vertex] (7) at (2.5, 1) {};

% Pendant vertices
\node[vertex] (8) at (-0.7, -0.5) {};
\node[vertex] (9) at (4.2, -0.5) {};

% Labels
\node[above, font=\scriptsize] at (0) {$v_0$};
\node[below, font=\scriptsize] at (1) {$v_1$};
\node[below, font=\scriptsize] at (2) {$v_2$};
\node[above, font=\scriptsize] at (3) {$v_3$};
\node[below, font=\scriptsize] at (4) {$v_4$};
\node[below, font=\scriptsize] at (5) {$v_5$};
\node[above, font=\scriptsize] at (6) {$v_6$};
\node[above, font=\scriptsize] at (7) {$v_7$};
\node[below, font=\scriptsize] at (8) {$v_8$};
\node[below, font=\scriptsize] at (9) {$v_9$};

% Left cycle edges
\draw[edge] (0) -- (1);
\draw[edge] (1) -- (2);
\draw[critical edge] (2) -- (3);
\draw[edge] (3) -- (0);

% Bridge
\draw[edge] (2) -- (4);

% Right cycle edges
\draw[edge] (4) -- (5);
\draw[edge] (5) -- (6);
\draw[critical edge] (6) -- (7);
\draw[edge] (7) -- (4);

% Pendant edges
\draw[edge] (1) -- (8);
\draw[edge] (5) -- (9);

\end{tikzpicture}

\vspace{1em}

\begin{tikzpicture}[
    vertex/.style={circle, fill=black, minimum size=4pt, inner sep=0pt},
    critical vertex/.style={circle, fill=red, minimum size=4pt, inner sep=0pt},
    path vertex/.style={circle, fill=green!70!black, minimum size=4pt, inner sep=0pt},
    edge/.style={thick},
    critical edge/.style={thick, red},
    path edge/.style={thick, green!70!black}
]

% Left cycle
\node[vertex] (0) at (0, 1) {};
\node[vertex] (1) at (0, 0) {};
\node[path vertex] (2) at (1, 0) {};
\node[vertex] (3) at (1, 1) {};

% Right cycle
\node[path vertex] (4) at (2.5, 0) {};
\node[critical vertex] (5) at (3.5, 0) {};
\node[path vertex] (6) at (3.5, 1) {};
\node[vertex] (7) at (2.5, 1) {};

% Pendant vertices
\node[vertex] (8) at (-0.7, -0.5) {};
\node[vertex] (9) at (4.2, -0.5) {};

% Labels
\node[above, font=\scriptsize] at (0) {$v_0$};
\node[below, font=\scriptsize] at (1) {$v_1$};
\node[below, font=\scriptsize] at (2) {$v_2$};
\node[above, font=\scriptsize] at (3) {$v_3$};
\node[below, font=\scriptsize] at (4) {$v_4$};
\node[below, font=\scriptsize] at (5) {$v_5$};
\node[above, font=\scriptsize] at (6) {$v_6$};
\node[above, font=\scriptsize] at (7) {$v_7$};
\node[below, font=\scriptsize] at (8) {$v_8$};
\node[below, font=\scriptsize] at (9) {$v_9$};

% Left cycle edges
\draw[edge] (0) -- (1);
\draw[edge] (1) -- (2);
\draw[critical edge] (2) -- (3);
\draw[edge] (3) -- (0);

% Bridge - part of descending path
\draw[path edge] (2) -- (4);

% Right cycle edges
\draw[path edge] (4) -- (5);
\draw[path edge] (5) -- (6);
\draw[critical edge] (6) -- (7);
\draw[edge] (7) -- (4);

% Pendant edges
\draw[edge] (1) -- (8);
\draw[edge] (5) -- (9);

\end{tikzpicture}
\caption{Example graph with two 4-cycles connected by a bridge. Top: the input graph. Middle: critical cells in red. Bottom: descending paths in green, flowing from critical edges to the critical vertex.}
\label{fig:example_graph}
\end{figure}
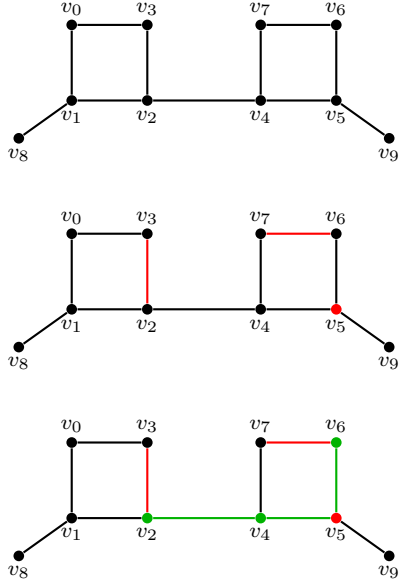

\paragraph{Graph structure.}
The vertex set is $V = \{0, 1, \ldots, 9\}$ with edges
\begin{align*}
E = \{&(0,1), (1,2), (2,3), (3,0)\} && \text{(left cycle)} \\
\cup\; \{&(2,4)\} && \text{(bridge)} \\
\cup\; \{&(4,5), (5,6), (6,7), (7,4)\} && \text{(right cycle)} \\
\cup\; \{&(1,8), (5,9)\}. && \text{(pendant edges)}
\end{align*}

\paragraph{Filtration values.}
We assign filtration values $f: V \cup E \to \mathbb{R}$ from the function $f(v) = deg_{max} -deg(v) + \epsilon(v)$ as discussed in  \ref{sec:method}.
For vertices:
\begin{center}
\begin{tabular}{@{}lccccc@{}}
\toprule
$v$ & 0 & 1 & 2 & 3 & 4 \\
$f(v)$ & 1.187 & 0.475 & 0.366 & 1.299 & 0.078 \\
\midrule
$v$ & 5 & 6 & 7 & 8 & 9 \\
$f(v)$ & 0.078 & 1.029 & 1.433 & 2.301 & 2.354 \\
\bottomrule
\end{tabular}
\end{center}
For edges, recall that $f(e)$ is set to the midpoint of endpoint values: $f(\{u,w\}) = \tfrac{1}{2}(f(u) + f(w))$.
This scoring rule guarantees face-monotonicity and enables a canonical pairing construction.

\paragraph{Discrete gradient vector field.}
We construct a filtration-compatible discrete gradient vector field $\mathcal{V}$ using the \emph{upper-link pairing rule}:
for each edge $e = \{u, w\}$, if vertex $v \in e$ satisfies $f(v) = \max\{f(u), f(w)\}$ and $v$ is not yet paired, 
we add $(v, e)$ to $\mathcal{V}$.
Processing edges in increasing $f$-order yields the pairings shown in Table~\ref{tab:pairings}.

\begin{table}[H]
\centering
\caption{Vertex--edge pairings in the discrete gradient vector field. Each paired edge $e$ is matched with 
its higher-valued endpoint. Unpaired cells are critical.}
\label{tab:pairings}
\vspace{0.5em}
\begin{tabular}{@{}lcc@{}}
\toprule
Edge $e$ & $f(e)$ & Pairing $(v, e) \in \mathcal{V}$ \\
\midrule
$(4,5)$ & 0.078 & $(4, (4,5))$ \\
$(2,4)$ & 0.222 & $(2, (2,4))$ \\
$(1,2)$ & 0.421 & $(1, (1,2))$ \\
$(5,6)$ & 0.554 & $(6, (5,6))$ \\
$(4,7)$ & 0.756 & $(7, (4,7))$ \\
$(0,1)$ & 0.831 & $(0, (0,1))$ \\
$(5,9)$ & 1.216 & $(9, (5,9))$ \\
$(0,3)$ & 1.243 & $(3, (0,3))$ \\
$(1,8)$ & 1.388 & $(8, (1,8))$ \\
\midrule
$(2,3)$ & 1.310 & \textit{unpaired (critical)} \\
$(6,7)$ & 1.918 & \textit{unpaired (critical)} \\
\bottomrule
\end{tabular}
\end{table}

\paragraph{Critical cells.}
Cells not appearing in any pair are \emph{critical}:
\[
\mathcal{C}_0 = \{5\}, \qquad \mathcal{C}_1 = \{(2,3), (6,7)\}.
\]
The unique critical vertex $5$ has the global minimum filtration value among vertices, 
while the two critical edges $(2,3)$ and $(6,7)$ are precisely the cycle-closing edges 
that cannot be paired without violating acyclicity.

\paragraph{Descending paths and anchor computation.}
For each critical edge $e \in \mathcal{C}_1$, we compute descending paths by following the 
gradient flow toward lower filtration values.
Starting from $e = \{u, w\}$, we alternate between:
\begin{enumerate}[label=(\roman*), nosep]
    \item moving from edge to vertex along the boundary (choosing the unpaired incidence), and
    \item moving from paired vertex to its matched edge.
\end{enumerate}
The path terminates at a critical vertex.

\textbf{Critical edge $(2,3)$:} 
Both endpoints are reachable; taking the branch through vertex $2$:
\[
(2,3) \;\longrightarrow\; 2 \;\longrightarrow\; (2,4) \;\longrightarrow\; 4 \;\longrightarrow\; (4,5) \;\longrightarrow\; 5.
\]
The anchor is $r^-(u) = r^-(w) = 5$ for both endpoints, computed using only edges with $f < f(2,3)$.

\textbf{Critical edge $(6,7)$:}
Taking the branch through vertex $6$:
\[
(6,7) \;\longrightarrow\; 6 \;\longrightarrow\; (5,6) \;\longrightarrow\; 5.
\]
Again, both endpoints have anchor $5$.

\section{Counterexamples: Persistence Does Not Determine Descending Manifolds}
\label{sec:ph_counterexamples}

\begin{figure}[t]
\centering
\begin{tikzpicture}[scale=0.90]

% === LEFT PANEL: V_1 ===
\begin{scope}[shift={(-2.8,0)}]
    \coordinate (v0) at (30:1.6);
    \coordinate (v1) at (90:1.6);
    \coordinate (v2) at (150:1.6);
    \coordinate (v3) at (210:1.6);
    \coordinate (v4) at (270:1.6);
    \coordinate (v5) at (330:1.6);
    \coordinate (c) at (0,0);
    
    % Triangle fills
    \fill[blue!10] (c) -- (v1) -- (v2) -- cycle;
    \fill[blue!10] (c) -- (v2) -- (v3) -- cycle;
    \fill[blue!10] (c) -- (v3) -- (v4) -- cycle;
    \fill[blue!10] (c) -- (v4) -- (v5) -- cycle;
    \fill[blue!10] (c) -- (v5) -- (v0) -- cycle;
    \fill[blue!40] (c) -- (v0) -- (v1) -- cycle; % Critical
    
    % Edges
    \draw[thick] (v0) -- (v1) -- (v2) -- (v3) -- (v4) -- (v5) -- cycle;
    \draw[blue!70!black, semithick] (c) -- (v0);
    \draw[blue!70!black, semithick] (c) -- (v1);
    \draw[blue!70!black, semithick] (c) -- (v2);
    \draw[blue!70!black, semithick] (c) -- (v3);
    \draw[blue!70!black, semithick] (c) -- (v4);
    \draw[gray!60, dashed, semithick] (c) -- (v5); % Excluded
    \draw[blue!70!black, very thick] (c) -- (v0) -- (v1) -- cycle;
    
    % Vertices
    \foreach \v in {v0,v1,v2,v3,v4,v5} \fill[black] (\v) circle (1.8pt);
    \fill[black] (c) circle (1.8pt);
    
    % Labels
    \node[right, font=\tiny] at ($(v0)+(0.08,0)$) {0};
    \node[above, font=\tiny] at ($(v1)+(0,0.08)$) {1};
    \node[left, font=\tiny] at ($(v2)+(-0.08,0)$) {2};
    \node[left, font=\tiny] at ($(v3)+(-0.08,0)$) {3};
    \node[below, font=\tiny] at ($(v4)+(0,-0.08)$) {4};
    \node[right, font=\tiny] at ($(v5)+(0.08,0)$) {5};
    \node[below, font=\tiny] at ($(c)+(0,-0.15)$) {$c$};
\end{scope}

% === RIGHT PANEL: V_2 ===
\begin{scope}[shift={(2.8,0)}]
    \coordinate (v0) at (30:1.6);
    \coordinate (v1) at (90:1.6);
    \coordinate (v2) at (150:1.6);
    \coordinate (v3) at (210:1.6);
    \coordinate (v4) at (270:1.6);
    \coordinate (v5) at (330:1.6);
    \coordinate (c) at (0,0);
    
    % Triangle fills
    \fill[red!10] (c) -- (v1) -- (v2) -- cycle;
    \fill[red!10] (c) -- (v2) -- (v3) -- cycle;
    \fill[red!10] (c) -- (v3) -- (v4) -- cycle;
    \fill[red!10] (c) -- (v4) -- (v5) -- cycle;
    \fill[red!10] (c) -- (v5) -- (v0) -- cycle;
    \fill[red!40] (c) -- (v0) -- (v1) -- cycle; % Critical (same)
    
    % Edges
    \draw[thick] (v0) -- (v1) -- (v2) -- (v3) -- (v4) -- (v5) -- cycle;
    \draw[red!70!black, semithick] (c) -- (v0);
    \draw[red!70!black, semithick] (c) -- (v1);
    \draw[gray!60, dashed, semithick] (c) -- (v2); % Excluded
    \draw[red!70!black, semithick] (c) -- (v3);
    \draw[red!70!black, semithick] (c) -- (v4);
    \draw[red!70!black, semithick] (c) -- (v5);
    \draw[red!70!black, very thick] (c) -- (v0) -- (v1) -- cycle;
    
    % Vertices
    \foreach \v in {v0,v1,v2,v3,v4,v5} \fill[black] (\v) circle (1.8pt);
    \fill[black] (c) circle (1.8pt);
    
    % Labels
    \node[right, font=\tiny] at ($(v0)+(0.08,0)$) {0};
    \node[above, font=\tiny] at ($(v1)+(0,0.08)$) {1};
    \node[left, font=\tiny] at ($(v2)+(-0.08,0)$) {2};
    \node[left, font=\tiny] at ($(v3)+(-0.08,0)$) {3};
    \node[below, font=\tiny] at ($(v4)+(0,-0.08)$) {4};
    \node[right, font=\tiny] at ($(v5)+(0.08,0)$) {5};
    \node[below, font=\tiny] at ($(c)+(0,-0.15)$) {$c$};
\end{scope}

% === BARCODE ===
\begin{scope}[shift={(0,-1.7)}]
    \draw[-{Stealth[length=2mm]}] (-1.8,0) -- (2.2,0) node[right, font=\tiny] {$t$};
    \draw[very thick, black!70] (-0.8,0.12) -- (1.2,0.12);
    \fill[black!70] (-0.8,0.12) circle (2pt);
    \draw[black!70, thick] (1.2,0.12) circle (2pt);
    \node[below, font=\tiny] at (-0.8,-0.05) {1};
    \node[below, font=\tiny] at (1.2,-0.05) {2};
    \node[left, font=\tiny] at (-1.3,0.12) {$H_1$:};
\end{scope}

% === Critical cell annotation ===
\node[font=\tiny, align=center] at (0,1.0) {Same critical 2-cell:\\ $\Delta_0 = (c,0,1)$};

\end{tikzpicture}
\caption{\textbf{Same persistence, different gradient routing.} 
Both panels show the same filtered hexagon-fan complex with identical $H_1$ persistence $[1,2)$. 
The critical 2-cell $\Delta_0 = (c,0,1)$ (darkly shaded) is identical in both matchings. 
However, V-path reachability differs: in $V_1$, gradient paths reach $\Delta_0$ counterclockwise (support excludes $s_5$); 
in $V_2$, clockwise (support excludes $s_2$). 
This directed incidence information-invisible to persistence-determines which edges are downstream of a cycle-filling event in topology-guided generation.}
\label{fig:ph_vs_flow}
\end{figure}
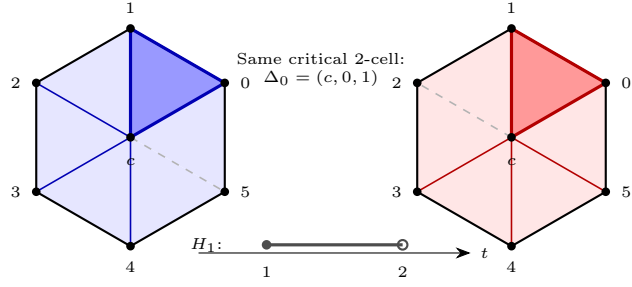 

\paragraph{Scope and compatibility.}
Appendix~J proves a \emph{non-identifiability} statement: persistence diagrams do not determine descending-manifold or flow structure.
To make this obstruction independent of implementation details, the counterexamples in this appendix are stated for \emph{externally specified filtrations} as well as (in Section \ref{sec:maximality_counterexample}) for the \emph{entry-time compatible} and \emph{maximal} notions of matchings (Definition \ref{def:max_filtr}).
By contrast, our implementation constructs a discrete Morse function $f$ and uses the weaker \emph{$f$-admissible} condition (Definition \ref{def:compatibility_notion}), which only guarantees compatibility with the \emph{sublevel filtration induced by $f$}.
The impossibility result does not require the counterexamples to arise from Alg.~\ref{alg:discrete_morse_function}.

We now provide explicit counterexamples demonstrating that persistent homology cannot recover gradient-flow geometry.

\begin{proposition}[Non-identifiability of descending manifolds from persistence]
\label{thm:nonidentifiability}
There exists no well-defined map
\[
\Phi:\mathrm{Dgm}_*(\Lambda)\longrightarrow \mathrm{DescMfd}(\Lambda)
\]
from persistence diagrams of filtrations of a finite cell complex $\Lambda$ to descending-manifold structures on $\Lambda$ such that, for every filtration $F$ of $\Lambda$ and every discrete gradient vector field $V$ compatible with $F$, one has
\[
\Phi\big(\mathrm{PH}(F)\big)
\;=\;
\{\mathcal{D}_V(\sigma)\}_{\sigma\in\mathrm{Crit}(V)}.
\]
Equivalently, the fiber over a single persistence diagram may contain filtration-compatible discrete gradient vector fields inducing distinct descending-manifold structures.
\end{proposition}

\begin{proof}
The proof follows from the explicit counterexamples constructed below. We exhibit filtrations with genuine finite birth--death intervals and two distinct filtration-compatible gradient vector fields that yield identical persistence diagrams but different descending manifolds.
\end{proof}

\subsection{A Minimal Counterexample: Path Graph on Four Vertices}
\label{sec:minimal_counterexample}

Let $\Lambda$ be the path graph on four vertices (a $1$-dimensional CW complex) $V = \{v_0, v_1, v_2, v_3\}$ with edges $E=\{e_{01},e_{12},e_{23}\}$:
\[
v_0 - v_1 - v_2 - v_3.
\]
Define a filtration by subcomplexes $\Lambda_0 \subset \Lambda_1 \subset \Lambda_2 \subset \Lambda_3$ via:
\begin{align*}
\Lambda_0 &:= \{v_0,v_3\}, \\
\Lambda_1 &:= \Lambda_0 \cup \{v_1, e_{01}\}, \\
\Lambda_2 &:= \Lambda_1 \cup \{v_2, e_{23}\}, \\
\Lambda_3 &:= \Lambda_2 \cup \{e_{12}\}.
\end{align*}
This is a valid filtration with a merge event at time $t = 3$.

\begin{figure}[ht]
  \centering
  \includegraphics[width=1.0\linewidth]{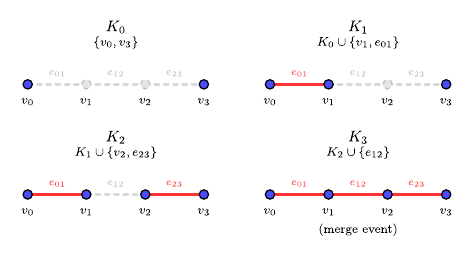}
  \caption{A filtration on $\Lambda$: the path graph enters in stages, with the final edge $e_{12}$ merging two components at time $t=3$.}
  \label{fig:filtration_path}
\end{figure}

\paragraph{The $0$-dimensional persistence diagram.}
At $t=0$ there are two connected components, so $H_0(\Lambda_0)\cong\mathbb{Z}^2$. At $t=1$ and $t=2$ there are still two components, so $H_0(\Lambda_1)\cong H_0(\Lambda_2)\cong \mathbb{Z}^2$. At $t=3$ the edge $e_{12}$ connects the two components, so $H_0(\Lambda_3)\cong \mathbb{Z}$. Hence there is exactly one death event in $H_0$ at time $3$.

Under the standard persistence pairing convention (``elder rule''), one of the two components born at time $0$ receives the essential interval $(0,\infty)$ and the other receives the finite interval $(0,3)$. Thus as a multiset:
\[
D_0(F)=\{(0,\infty),(0,3)\}, \qquad D_1(F)=\varnothing.
\]
The diagram contains a genuine finite birth--death interval.

\paragraph{Two gradient fields compatible with the same filtration.}
Define two discrete gradient vector fields on $\Lambda_3$:
\[
V_1 := \{(v_0,e_{01}),\ (v_2,e_{23})\},
\qquad
V_2 := \{(v_1,e_{01})\}.
\]
On a tree, any matching is acyclic (there are no directed cycles in the directed Hasse diagram), so both $V_1$ and $V_2$ are valid discrete gradient vector fields that are filtration-compatible.

\begin{figure}[ht]
  \centering
  \includegraphics[width=1.\linewidth]{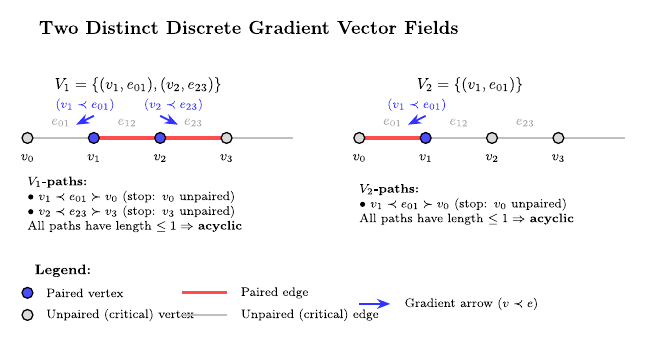}
  \caption{Two distinct gradient vector fields $V_1$ (left) and $V_2$ (right) on the path graph.}
  \label{fig:two_gradient_fields}
\end{figure}

\paragraph{Same persistence, different descending manifolds.}
The death of the $0$-dimensional class at $t=3$ is caused by the attachment of edge $e_{12}$. In both cases, $e_{12}$ is critical (unpaired). However, its descending manifold depends on the chosen gradient field.

To compute descending manifolds, we form the directed graph $G_\mathcal{V}$ on cells as follows:
\begin{itemize}
\item If an edge $e$ is \emph{unpaired}, orient its incidences downward to endpoints: $e\to v$ for each endpoint $v\prec e$.
\item If a vertex $v$ is paired with an edge $e$ in $\mathcal{V}$, orient upward along the pair: $v\to e$.
\item For a paired edge $e=(v,w)$ with $v\to e$ (because $(v,e)\in \mathcal{V}$), keep the \emph{other} incidence oriented downward: $e\to w$.
\end{itemize}

\medskip
\noindent\underline{Case $V_1 = \{(v_0,e_{01}), (v_2,e_{23})\}$:}

Since $e_{12}$ is unpaired, its incidences orient downward: $e_{12}\to v_1$ and $e_{12}\to v_2$.

Vertex $v_1$ is unpaired in $V_1$, so there is no outgoing arrow; the branch through $v_1$ stops. Vertex $v_2$ is paired with $e_{23}$, so $v_2\to e_{23}$ and $e_{23}\to v_3$. Hence $e_{12}\to v_2\to e_{23}\to v_3$ is a directed path. Therefore,
\[
\mathcal{D}_{V_1}(e_{12}) = \{e_{12},\, v_1,\, v_2,\, e_{23},\, v_3\}.
\]

\medskip
\noindent\underline{Case $V_2 = \{(v_1,e_{01})\}$:}

Again $e_{12}$ is unpaired, so $e_{12}\to v_1$ and $e_{12}\to v_2$.

Now $v_1$ is paired with $e_{01}$, hence $v_1\to e_{01}$ and $e_{01}\to v_0$. Thus $e_{12}\to v_1\to e_{01}\to v_0$ is a directed path. Meanwhile $v_2$ is unpaired in $V_2$, so the branch through $v_2$ stops. Therefore,
\[
\mathcal{D}_{V_2}(e_{12}) = \{e_{12},\, v_2,\, v_1,\, e_{01},\, v_0\}.
\]

\medskip
\noindent\textbf{Conclusion.} We have the same filtration $F$ (hence identical persistence diagram $D_0(F)=\{(0,\infty),(0,3)\}$) and two gradient fields $V_1\neq V_2$ such that:
\[
\mathcal{D}_{V_1}(e_{12}) \neq \mathcal{D}_{V_2}(e_{12}).
\]
Thus any purported rule $\Phi$ that outputs a unique descending-manifold structure from the diagram alone cannot be well-defined.

\subsection{Counterexample with Maximality Constraint}
\label{sec:maximality_counterexample}

\subsection{J.2. Robustness under stricter compatibility: maximal entry-time matchings}
One might object that the non-identifiability in J.1 is an artifact of using non-maximal matchings or a permissive compatibility notion.
We therefore strengthen the obstruction: even restricting to \emph{maximal} matchings that are \emph{entry-time compatible} with an external filtration (Definition \ref{def:max_filtr}), which is a strictly stronger constraint than the $f$-admissibility used in our implementation; persistence still does not determine descending-manifold connectivity.

% \begin{definition}[Maximality]
% Fix a filtration $F:K_0\subset \cdots \subset K_T$ and let $t(\sigma)$ denote the first index at which a cell $\sigma$ appears. A discrete gradient field $\mathcal{V}$ is \emph{filtration-compatible} if every matched pair $(\sigma,\tau)\in \mathcal{V}$ satisfies $t(\sigma)=t(\tau)$. Among filtration-compatible matchings, we call $\mathcal{V}$ \emph{maximal} if there is no strictly larger filtration-compatible matching $\mathcal{V}\subsetneq \widetilde{\mathcal{V}}$.
% \end{definition}

\begin{definition}[Maximality under entry-time compatibility]
Fix an external filtration $F:\Lambda_0\subset\cdots\subset \Lambda_T$ and let $t_F(\sigma)$ denote the first index at which a cell $\sigma$ appears.
A discrete gradient field $\mathcal{V}$ is \emph{entry-time compatible} (with respect to $F$) if every matched pair $(\sigma,\tau)\in\mathcal{V}$ satisfies $t_F(\sigma)=t_F(\tau)$.
Among entry-time compatible matchings, we call $\mathcal{V}$ \emph{maximal} if there is no strictly larger entry-time compatible matching $\mathcal{V}\subsetneq \widetilde{\mathcal{V}}$.
\label{def:max_filtr}
\end{definition}

\paragraph{Construction: Star graph with simultaneous edge arrivals.}
Let $T$ be the tree on vertices $\{u,a,b,c\}$ with edges $E=\{e_{ua}, e_{ub}, e_{uc}\}$ (a 3-arm star centered at $u$). Define a filtration:
\begin{align*}
T_0 &= \{a,b,c\}, \\
T_1 &= T_0\cup\{u,e_{ua}\}, \\
T_2 &= T_1\cup\{e_{ub},e_{uc}\}.
\end{align*}

\begin{figure}[ht]
  \centering
  \includegraphics[width=0.9\linewidth]{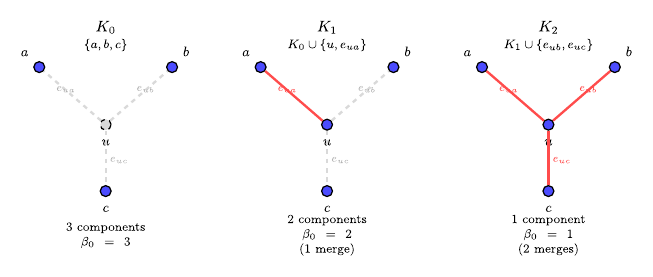}
  \caption{A filtration on the star graph $T$: three leaves appear at $t=0$, the center and one edge at $t=1$, and two more edges simultaneously at $t=2$.}
  \label{fig:star_filtration}
\end{figure}

At time $t=0$ there are three components; at $t=1$ vertex $u$ and edge $e_{ua}$ attach, reducing components by one; at $t=2$ edges $e_{ub}$ and $e_{uc}$ attach, producing two further merge events.

\paragraph{Two distinct maximal filtration-compatible matchings.}
At filtration level $t=2$, the new edges $e_{ub}$ and $e_{uc}$ \emph{share} vertex $u$. A filtration-compatible matching can pair $u$ with at most one of them at time $t=2$, yielding two distinct choices:
\begin{align*}
\mathcal{V}^{(b)}_{\max} &= \{(u,e_{ub})\}\ \cup\ \{\text{compatible pairs from earlier levels}\}, \\
\mathcal{V}^{(c)}_{\max} &= \{(u,e_{uc})\}\ \cup\ \{\text{compatible pairs from earlier levels}\}.
\end{align*}

Both matchings are filtration-compatible and maximal: once $u$ is paired with one of $e_{ub}$ or $e_{uc}$, the other edge cannot be paired with $u$ (matching constraint), and it cannot be paired with $b$ or $c$ because those vertices were born at $t=0$ while the edges were born at $t=2$ (violating compatibility).

The \emph{descending-manifold connectivity} differs between these two maximal matchings---the directed reachability from the critical edge responsible for the merge at time $2$ flows through $u$ toward $b$ in one matching and toward $c$ in the other---while the persistence diagram of the filtration is identical in both cases.

\begin{remark}[Filtration-compatible gradient fields]
Filtration-compatible gradient fields are useful because they yield a smaller filtered Morse complex that computes the same persistence. However, the diagram itself is a property of the filtration, not of the gradient field.
\end{remark}

Essentially, Proposition \ref{thm:nonidentifiability} follows from the combination of two well-established facts: 

First, persistent homology is a coarse invariant: distinct filtrations of the
same complex may induce identical persistence diagrams while differing in
geometric or combinatorial structure beyond homological birth--death data
\citep{cohen2007stability,bauer2021measuring}.

Second, discrete Morse theory admits substantial non-uniqueness.
For a fixed filtration $F$, there generally exist multiple discrete gradient
vector fields compatible with $F$, all yielding Morse reductions that preserve
persistent homology but induce different gradient-flow structures and descending
manifolds \citep{forman1998morse,mischaikow2013morse}.

In particular, there exist filtrations $F$ and pairs of compatible discrete
gradient vector fields $V_1,V_2$ such that $\mathrm{PH}(F) \text{ is identical for both,}$ yet the associated collections of descending manifolds
$\{\mathcal{D}_{V_1}(\sigma)\}$ and $\{\mathcal{D}_{V_2}(\sigma)\}$ differ as
combinatorial subcomplexes.
Any map $\Phi$ depending only on the persistence diagram would necessarily assign
the same output to both cases, contradicting the requirement that it recover the
descending-manifold structure for all compatible gradient fields.

Therefore, no such reconstruction map $\Phi$ can exist.

\section{Algorithms}
\label{sec:algorithms_appendix}

\begin{algorithm}[t]
\caption{Discrete Morse function on a $d$-dimensional simplicial complex $\Lambda$}
\label{alg:discrete_morse_function}
\begin{algorithmic}[1]
\REQUIRE complex $\Lambda$ given as a dictionary $\texttt{complex\_dict}[p]$ of $p$-simplices up to dimension $d$; vertex weights $g$
\ENSURE function $f$ on all simplices in $\Lambda$; face-flag map $\mathrm{Flag}$
\STATE Initialize $\mathrm{Flag}[\alpha]\gets 0$ for all simplices $\alpha\in \Lambda$
\FORALL{vertices $v\in \texttt{complex\_dict}[0]$}
    \STATE $f(v)\gets g(v)$
\ENDFOR
\FOR{$p=1$ \textbf{to} $d$}
    \FORALL{$p$-simplices $\alpha\in \texttt{complex\_dict}[p]$}
        \STATE $\mathrm{Faces}\gets$ list of $(p-1)$-faces of $\alpha$
        \STATE sort $\mathrm{Faces}$ so $f(\mathrm{Faces}[0])\ge f(\mathrm{Faces}[1])\ge \cdots$
        \STATE $\gamma_0\gets \mathrm{Faces}[0]$
        \STATE $\gamma_1\gets \mathrm{Faces}[1]$ if $|\mathrm{Faces}|\ge 2$, else $\gamma_1\gets \gamma_0$
        \IF{$\mathrm{Flag}[\gamma_0]=0$ \AND $f(\gamma_0)>f(\gamma_1)$}
            \STATE $f(\alpha)\gets \frac{f(\gamma_0)+f(\gamma_1)}{2}$
            \STATE $\mathrm{Flag}[\gamma_0]\gets 1$
        \ELSE
            \STATE sample $\varepsilon\sim \mathrm{Unif}(0,0.5)$
            \STATE $f(\alpha)\gets f(\gamma_0)+\varepsilon$
        \ENDIF
    \ENDFOR
\ENDFOR
\STATE \textbf{return} $(f,\mathrm{Flag})$
\end{algorithmic}
\end{algorithm}

% ============================================================
% Algorithm 2: construct_gradient_vector_field (Cell 8)
% NOTE: Flag is NOT used by the code (even though passed in).
% ============================================================
\begin{algorithm}[t]
\caption{Gradient-respecting neighbors in the incidence graph}
\label{alg:neighbors}
\begin{algorithmic}[1]
\REQUIRE $\texttt{complex\_dict}$ up to dimension $d$; function $f$ on simplices
\ENSURE pairing map $\mathsf{pair}$ and list of gradient pairs $\mathcal{V}$
\STATE Initialize $\mathsf{pair}(\alpha)\gets \varnothing$ for all simplices $\alpha\in \Lambda$
\STATE $\mathcal{V}\gets \emptyset$
\FOR{$p=0$ \textbf{to} $d-1$}
    \FORALL{$(p+1)$-simplices $\tau\in \texttt{complex\_dict}[p+1]$}
        \IF{$\mathsf{pair}(\tau)\neq \varnothing$}
            \STATE \textbf{continue}
        \ENDIF
        \STATE $\mathrm{Faces}\gets$ all $p$-faces $\sigma\prec \tau$
        \STATE $\mathrm{Eligible}\gets \{\sigma\in \mathrm{Faces}:\ \mathsf{pair}(\sigma)=\varnothing\ \textbf{and}\ f(\tau)\le f(\sigma)\}$
        \IF{$\mathrm{Eligible}\neq \emptyset$}
            \STATE $\sigma^\star \gets \arg\min_{\sigma\in \mathrm{Eligible}} f(\sigma)$
            \STATE $\mathcal{V}\gets \mathcal{V}\cup\{(\sigma^\star,\tau)\}$
            \STATE $\mathsf{pair}(\sigma^\star)\gets \tau$,\;\; $\mathsf{pair}(\tau)\gets \sigma^\star$
        \ENDIF
    \ENDFOR
\ENDFOR
\STATE \textbf{return} $(\mathcal{V}, \mathsf{pair})$
\end{algorithmic}
\end{algorithm}

% ============================================================
% Algorithm 3A: downhill-only neighbor oracle used inside find_gradient_vpath (Cell 9)
% ============================================================
\begin{algorithm}[t]
\caption{BFS path in the gradient-respecting incidence graph}
\label{alg:bfs}
\begin{algorithmic}[1]
\REQUIRE simplex $\sigma$; function $f$; pairing map $\mathsf{pair}$
\ENSURE neighbor set $\mathcal{N}_{\downarrow}(\sigma)$
\STATE $\mathcal{N}_{\downarrow}(\sigma)\gets \emptyset$
\STATE $\beta\gets \mathsf{pair}(\sigma)$
\IF{$\beta\neq \varnothing$ \AND $\dim(\beta)<\dim(\sigma)$}
    \STATE $\mathcal{N}_{\downarrow}(\sigma)\gets \mathcal{N}_{\downarrow}(\sigma)\cup\{\beta\}$
\ENDIF
\IF{$\dim(\sigma)>0$}
    \FORALL{faces $\gamma\prec\sigma$ with $\dim(\gamma)=\dim(\sigma)-1$}
        \IF{$f(\gamma)< f(\sigma)$}
            \STATE $\mathcal{N}_{\downarrow}(\sigma)\gets \mathcal{N}_{\downarrow}(\sigma)\cup\{\gamma\}$
        \ENDIF
    \ENDFOR
\ENDIF
\STATE \textbf{return} $\mathcal{N}_{\downarrow}(\sigma)$
\end{algorithmic}
\end{algorithm}

% ============================================================
% Algorithm 4: find_gradient_vpath (Cell 9) uses BFS with downhill neighbors
% ============================================================
\begin{algorithm}[t]
\caption{Path between two critical cells (dimension-order wrapper)}
\label{alg:between}
\begin{algorithmic}[1]
\REQUIRE start simplex $s$, target simplex $t$, downhill neighbor oracle $\mathcal{N}_{\downarrow}$ from Algorithm~\ref{alg:bfs}
\ENSURE path $P$ and length $|P|-1$, or $(\varnothing,\infty)$ on failure
\STATE Initialize queue $Q\gets \{(s,[s])\}$ and visited set $\mathsf{Visited}\gets\{s\}$
\WHILE{$Q\neq\emptyset$}
    \STATE Pop-left $(\sigma,P)$ from $Q$
    \IF{$\sigma=t$}
        \STATE \textbf{return} $(P,|P|-1)$
    \ENDIF
    \FORALL{$\tau\in \mathcal{N}_{\downarrow}(\sigma)$}
        \IF{$\tau\notin \mathsf{Visited}$}
            \STATE $\mathsf{Visited}\gets \mathsf{Visited}\cup\{\tau\}$
            \STATE Push-right $(\tau, P\mathbin{\|}[\tau])$ into $Q$
        \ENDIF
    \ENDFOR
\ENDWHILE
\STATE \textbf{return} $(\varnothing,\infty)$
\end{algorithmic}
\end{algorithm}

%============================================================

\section{Overview of the Related Work}
\label{sec:lit_appendix}

{\bf Diffusion on graphs.} Denoising diffusion models (DDPMs) and score-based generative models learn to sample from complex data distributions by reversing a gradual noising process.
%, originally developed for continuous data such as images \citep{sohl2015deep, ho2020ddpm, song2021scorebased_sde}. A central design choice is the state space and corruption kernel. Namely, in continuous settings, Gaussian perturbations yield tractable reverse-time dynamics, whereas in discrete settings (including graphs), noising is commonly defined through categorical or other discrete Markov kernels and trained with classification-style denoising objectives. 
Graph diffusion methods adapt this framework to permutation-invariant, combinatorial objects by diffusing over an explicit graph representation such as an adjacency matrix (or edge indicators), optionally together with node/edge attributes~\cite{liu2023generative, cao2024survey, shou2026graph}.  A widely used approach is  \emph{categorical diffusion} on
edges such as DiGress~\citep{vignac2022digress}, EDGE~\citep{chen2023efficient} and SID~\citep{boget2025simple}, where the forward process progressively randomizes edge types and the reverse model predicts either the clean edges or the noise distribution. %A canonical example is DiGress~\citep{vignac2022digress} and which demonstrates strong performance for generic graph generation under a discrete denoising objective, and more recent advances 
Complementary to discrete-time categorical diffusion, continuous-time \emph{score-based} formulations define a stochastic differential equation (SDE) and Ordinary Differential Equation (ODE) over a (possibly relaxed) graph representation and learn a score network for reverse-time sampling. A broader class of such score-based graph generative models (SGGMs) models include such approaches as  GDSS~\citep{jo2022gdss}, CDGS~\citep{huang2023conditional} and NVDiff~\citep{chen2022nvdiff}, while~\cite{su2025non} offers a comprehensive treatment of the SGMS convergent analysis. Other alternatives includes interpretation of generation through stochastic optimal control as well as parametrizations based on sequential or normalizing-flow-like constructions, e.g., GGDiff~\citep{tenorio2025graph},
GraphDF~\citep{huang2022graphdf} and DeFoG~\citep{yimingdefog}. Additional line of research explores alternative parameterizations and scalability tradeoffs, including diffusion models that target molecular graphs and graph-structured data with domain constraints%~\citep{hoogeboom2022edm, su2025sbgd} 
and methods designed to improve efficiency for larger graphs via sparsity-aware mechanisms \citep{hoogeboom2022edm, qin2024sparsetrainingdiscretediffusion, su2025sbgd}. Beyond unconditional generation, diffusion has become a promising new alternative for \emph{probabilistic} spatiotemporal graph learning, where the principled  uncertainty quantification plays a critical role in the domain application, ranging from traffic forecasting to air quality monitoring~\citep{yang2024survey_ts_st_diffusion}. Some recent models in this direction include  non-autoregressive diffusion of DiffSTG~\citep{wen2024diffstgprobabilisticspatiotemporalgraph}, conditional diffusion for spatiotemporal imputation PriSTI~\citep{liu2023pristiconditionaldiffusionframework}, Unified Spatio-Temporal Diffusion Models (USTD) \citep{hu2024unifyingdiffusionmodelsprobabilistic}, topological zigzag-guided diffusion ZS-DM~\citep{chen2025topological}.

%Despite these advances, most graph diffusion pipelines primarily learn only \emph{local} denoising rules and do not explicitly represent when global %topological  events such as critical information on the creation of independent cycles should occur during generation, motivating structured intermediate guidance. In turn, 

Despite these advances, most graph diffusion pipelines primarily learn only \emph{local} denoising rules and do not explicitly represent when global events such as, for example, the appearance of fundamental cycles should occur during generation. However, these often overlooked fundamental (or independent) cycles convey a critical information on the underlying latent organization and functionality of the complex system, playing a particularly important role in spatiotemporal graph learning and diffusion. For example, in transportation systems fundamental cycles describe independent alternative traffic loops, define the overall connectivity of the traffic system and determine the robustness of traffic flow against disruptions.
In turn, fundamental cycles in power grid networks describe how the electricity can be transmitted along the power lines, such that specific cycle does not inherently force a specific power flow on another cycle. This, in turn, provides a critical information on potentially redundant power lines and sheds light onto the resilience of the power networks to attacks and the associated cascading failures.
We propose to address these fundamental gaps by invoking the machinery of law-dimensional topology and cobordism theory, in particular.

{\bf Cobordisms in Machine Learning} While topological methods get increasingly popular in ML, their range is largely limited to topological data analysis (TDA) and simplicial topology~\citep{carlson, papamarkou2024position}. In contrast, low-dimensional topology and, cobordism theory, in particular, remains virtually unexplored in ML. We advocate that the cobordism theory offers mathematical machinery that is remarkably well-suited 
to generative modeling yet remains largely untapped in applications~\cite{milnor1963morse,tqft}. 
Classical cobordism provides: (i) \emph{handle decompositions} that express 
topology-changing trajectories as compositions of elementary moves, giving a 
natural ``grammar'' for generation \cite{matsumoto2002introduction, milnor1963morse}; (ii) \emph{Cerf theory}, which classifies 
how critical points can be born, annihilated, or exchanged along generic 
one-parameter families---precisely the structure needed to understand how 
diffusion trajectories can vary while preserving essential features \cite{cerf1964cerf,weber2022multiparameter}; and 
(iii) \emph{topological quantum field theory} (TQFT), which assigns vector 
spaces to states and linear maps to cobordisms in a compositional, functorial 
manner \cite{aitchison2022quantum,marciano2024deepneuralnetworkssemiclassical}.
In this paper we introduce the power tools of the cobordism theory to the diffusion models of spatio-temporal graphs.

\begin{table}[H]
\scriptsize
\centering
\caption{Summary statistics of the datasets.}
\label{tab:dataset_stats}
\resizebox{0.8\linewidth}{!}{
\begin{tabular}{lcccl}
\hline
{\bf Dataset}   & \#{\bf Node} & \#{\bf Time Step} & {\bf Granularity} & {\bf Attribute} \\
\hline
PEMS-03   & 358    & 26,208      & 5 min       & Flow      \\
PEMS-BAY  & 325    & 52,116      & 5 min       & Speed     \\
AIR-BJ    & 36     & 8,760       & 1 hour      & PM2.5     \\
AIR-GZ    & 42     & 8,760       & 1 hour      & PM2.5     \\
\hline
\end{tabular}
}
\end{table}

\section{Datasets, baselines and experimental setup}\label{app_data}
We perform extensive experiments on real-world datasets, and all datasets
are obtained from open sources including (i) PEMS-03~\cite{chen2001freeway} which is sourced from the Performance Measurement System (PeMS) maintained by the California Department of Transportation (Caltrans), and all measurements are recorded at a uniform temporal resolution of 5 minutes; (ii) PEMS-BAY~\cite{chen2001freeway,li2018diffusion} is also collected by CalTrans's PeMS and contains 6 months of traffic speed information on 325 sensors in the Bay area; (iii) AIR-BJ which contains PM2.5 concentration measurements spanning 1 year collected from distributed air-quality monitoring stations in Beijing, China~\cite{yi2018deep}; (iv) AIR-GZ which contains PM2.5 concentration measurements spanning 1 year collected from distributed air-quality monitoring stations in Guangzhou, China~\cite{yi2018deep}. All the tests adopt 60 minutes (i.e., 12 time steps) as the history time window. To evaluate MG-Diff, we compare it with 8 state-of-the-art baselines, i.e., (i) TimeGrad~\cite{rasul2021autoregressive}, (ii) MC Dropout~\cite{wu2021quantifying}, (iii) CSDI~\cite{tashiro2021csdi}, (iv) DiffSTG~\cite{wen2024diffstgprobabilisticspatiotemporalgraph}, (v) PriSTI~\cite{liu2023pristiconditionaldiffusionframework}, (vi) USTD~\cite{hu2024unifyingdiffusionmodelsprobabilistic} , (vii) Decoupled spatio-temporal diffusion model (DSTD) (anonymous ICLR2026 submission), and (viii) ZS-DM~\cite{chen2025topological}. The model is trained using Adam optimizer. We set the batch size of 32, and set the learning rate as $1e^{-3}$ via searching set of $\{1e^{-4}, 3e^{-4}, 3e^{-3}, 1e^{-2}\}$. For the hidden dimension, we conduct a grid search over $\{8, 16, 32, 64, 128, 256\}$. Code is available in~\url{https://anonymous.4open.science/r/MD-Diff-B705/}.
%https://www.dropbox.com/scl/fo/z0szcf77buk59zp6gic15/AP4M-vCEYGcLGBY_RnEC3BI?rlkey=zwxl4r4zcd1tbpsw0r90saj3x&st=p08zlxxo&dl=0

\section{Ablation Studies and More Details on Computational Costs}\label{app_abl}

\begin{table}[H]
\centering
\caption{Ablation studies on MG-Diff variants. 
\label{ablation_analysis}}
  \begin{tabular}{lcccc}
    \toprule
    \multirow{2}{*}{\bf Architecture} &
      \multicolumn{2}{c}{MAE} &
      \multicolumn{2}{c}{CRPS}\\
      & \textbf{{PEMS-03}} & \textbf{{AIR-BJ}} & \textbf{{PEMS-03}} & \textbf{{AIR-BJ}}\\
      \midrule
    {\bf MG-Diff} &{\bf 15.67} & {\bf 29.91}& {\bf 0.077}  & {\bf 0.290}  \\
    MG-Diff  W/o MoMoE & 16.20 & 30.22 & 0.082 & 0.348 \\
    MG-Diff  W/o MDDN & 16.56 &30.83 & 0.234 & 0.534\\
    \bottomrule
  \end{tabular}
\end{table}

\begin{table}[h]
\centering
\caption{Ablation studies with different Morse function values on PEMS-03 and AIR-BJ.}
\label{tab:ablation_different_morse_values}
\begin{tabular}{lcc}
\hline
{\bf Method} & {\bf PEMS-03} & {\bf AIR-BJ}\\
\hline
MG-Diff & {\bf 15.67} & {\bf 29.91} \\
MG-Diff w. deg. & 16.20 & 30.22 \\
MG-Diff w. bet. & {16.56} & {30.83} \\
\hline
\end{tabular}
\end{table}
% 9.99 15.41 0.235
% 9.80 15.07 0.254
\begin{table*}[t]
\centering
\caption{Robustness analysis (noise ratio $=5\%$) on AIR-GZ.\label{robustness_analysis}}
%\vspace{-1.em}
  \scalebox{0.6}{
  \centering
  %\small
  \begin{tabular}[width=1.\linewidth]{l|cc|cc}
  \toprule
      \multirow{2}*{Method} & \multicolumn{2}{c}{Clean AIR-GZ} & \multicolumn{2}{c}{Noisy AIR-GZ}\\
      \cline{2-5}
      & MAE & RMSE & MAE & RMSE\\
      %\hline
      \hline
      USTD & 9.99 &15.41 & 10.72 & 15.71  \\ %& 29.91 & 46.43 & 0.31
      %ZS-DM & & \\
    \hline
    {\bf MG-Diff (Ours)} & 9.80 &15.07  & 10.25 & 15.26 % 30.50 & 48.04 & 0.34
    \\
     %{\bf MG-Diff - morse\_v2 + CRPS} & &  \\
  \bottomrule
\end{tabular}
}
\label{tab:forecasting}
\end{table*}

\paragraph{Scalability} Three optimizations improve scalability: replace finding cliques with triangle enumeration over common neighborhoods ($\mathcal{O}(M\Delta)$); precompute edge--triangle coface indices to cut identification to $\mathcal{O}(M+T_\triangle)$; and restrict path search to local critical pairs, reducing cost from $\mathcal{O}(K^2)$ to $\mathcal{O}(Kk)$ for small $k$. Each iteration runs the GWaveNet encoder and denoiser with backprop. Graph convolution dominates at $\mathcal{O}(B C N^2 L)$ per hop and per layer, giving total $\mathcal{O}((\ell+\Lambda)BCN^2L)$, where $N^2$ arises from dense adjacency and quadratic attention. Inference performs $S$ denoising steps after a single encoding pass, costing $\mathcal{O}(\ell BCN^2L + S\Lambda BN^2C)$. For probabilistic evaluation with $Q$ stochastic samples, inference scales as $\mathcal{O}(\ell BCN^2L + QS\Lambda BN^2C)$, roughly $QS$ times slower than one denoiser call per training step. Table~\ref{runtime} shows the training running time/epoch (sec) of Morse skeleton extraction and topology-guided diffusion.

\begin{table}[H]
\caption{Running time in seconds (s) (average per epoch).}
\label{runtime}
\centering
\begin{tabular}{lcccc}
\hline
 & PEMS-03 & PEMS-BAY & AIR-BJ & AIR-GZ \\
\hline
MG-Diff & 136  & 67.5  & 1  & 1 \\
\hline
\end{tabular}
\end{table}

Table \ref{tab:morse_controls} pairs Table \ref{tab:ablation_different_morse_values} with two comparisons on AIR-GZ on 5 random seeds. First, the Morse skeleton is replaced by a random skeleton obtained by sampling the same numbers of nodes and edges from the original graph; the random skeleton replaces the Morse-based conditioning while the model architecture, training configuration, and pretrained checkpoint are unchanged. Since the two models differ only in the conditioning skeleton, we see that the resulting $6.0\%, 3.5\%, 7,2\%, \text{ and }8.4\%$ relative improvements in MAE, RMSE, MAPE, and CRPS respectively. We note that this demonstrates that the observed gains arise from the topological structure encoded by the Morse skeleton rather than from the model capacity or the presence of a conditioning branch.

\textbf{Testing Various Morse Functions. }Further, we compare three representative scoring functions, and note that all variants achieve comparable performance (implying that MG-Diff is not overly sensitive to changes in the scoring rule). In our formulation $g(v)=\deg_{\max}-\deg(v)+\varepsilon_v$ with $\varepsilon_v\sim\mathrm{Unif}(0,0.5)$ achieves the overall performance with the lowest MAE, RMSE, CRPS and a competitive MAPE). On the other hand, $g(v)=\deg(v)+\varepsilon_v$ and $g(v)=\deg_{\mathrm{med}}+\deg(v)+\varepsilon_v$ result in only modest changes. In other words, the skeleton is robust as long as the scoring function preserves the underlying degree-based ordering, which also empirically supports the stability analysis of Section~\ref{sec:mainstablemma}. More generally, we see that the choice of $g$ depends on the underlying physical structure, and for applications governed by different processes, the scoring function can be adapted to emphasize alternative node attributes while preserving the overall framework and its theoretical guarantees.

\begin{table}[H]
\centering
\caption{Isolating the Morse skeleton on AIR-GZ (mean $\pm$ std over 5 seeds): a random skeleton of matched size vs.\ Morse skeletons under different vertex scoring rules $g$, with identical architecture, training configuration, and checkpoint.}
\label{tab:morse_controls}
\resizebox{0.95\textwidth}{!}{
\begin{tabular}{lcccc}
\toprule
Conditioning skeleton & MAE & RMSE & MAPE & CRPS \\
\midrule
Random skeleton (matched size) & 10.19767$\pm$0.00002 & 15.46738$\pm$0.00003 & 0.41979$\pm$0.00000 & 0.25223$\pm$0.00105 \\
Morse, $g(v)=\deg(v)+\varepsilon_v$ & 9.72684$\pm$0.00000 & 15.20191$\pm$0.00003 & {\bf 0.38107$\pm$0.00000} & 0.23853$\pm$0.00084 \\
Morse, $g(v)=\deg_{\mathrm{med}}+\deg(v)+\varepsilon_v$ & 10.44039$\pm$0.00004 & 15.30176$\pm$0.00003 & 0.46580$\pm$0.00000 & 0.24502$\pm$0.00112 \\
{\bf Morse, $g(v)=\deg_{\max}-\deg(v)+\varepsilon_v$ (Ours)} & {\bf 9.58643$\pm$0.00001} & {\bf 14.92188$\pm$0.00002} & 0.38942$\pm$0.00000 & {\bf 0.23099$\pm$0.00108} \\
\bottomrule
\end{tabular}}
\end{table}

\section{Additional Experiments}
\label{app_addexp}

Table~\ref{tree_gen} demonstrates that MG-Diff achieves the lowest MMDs for the key graph summaries including degree, wavelet, spectrum, and orbits, with relative gains from 3.2\% to 49\%. (MMD stands for Maximum mean discrepancy which quantifies distribution difference of the graph summary statistics of generated and test graphs.) Note that MG-Diff yields outstanding 1.7 and 2 times lower MMD for degree and orbits, respectively. This phenomenon can be explained directly by the topological skeleton conditioning, since tree-structured graphs have no independent cycles, so $E^{\mathrm{crit}} = \emptyset$ and the entire edge set coincides with the spanning forest $E^{\mathrm{pair}}$. The Morse decomposition hence gives the model complete structural guidance, non-critical expert  $f_{\alpha'}$ processes the full graph as a tree backbone, with gradient flow directed toward high-degree hub vertices by the degree-based scoring function. This means the skeleton explicitly encodes the branching hierarchy of the tree: which nodes are hubs (critical nodes that anchor the spanning tree) and which are peripheral leaves. The degree distribution is a readout of this hierarchy, where a node's degree in a tree is precisely its number of children plus one (for its parent edge), which the non-critical expert's message passing along $E^{\mathrm{pair}}$ captures by construction. Without this conditioning, baseline models must discover the correct degree hierarchy implicitly from data, which is precisely where they lose fidelity. The orbit (graphlet) improvement is more pronounced because orbit statistics count higher-order subgraph patterns such as stars, paths, and their combinations that are composed of exactly the local branching motifs the spanning-tree expert specializes in. In a tree, every orbit pattern is determined by the local branching structure of the spanning forest, and the Morse skeleton provides the denoiser with a direct encoding of that structure at every denoising step.

Performance of all models on clus. is virtually indistinguishable, which is due to the fact that the clustering coefficient %$c(v) = 2T(v) / (d(v)(d(v)-1))$ 
measures triangle density around each node, and trees contain no triangles by definition ($\beta_1 = 0$ implies no cycles, no triangles). The clustering coefficient is hence identically 0 for every node, making the clustering MMD trivially zero for any model that produces valid trees. This metric provides no discriminative power on tree-structured data, and the fact that all models match on it is an artifact of the data's acyclic topology rather than a reflection of model quality. 

In summary, the Morse skeleton's advantage is precisely in capturing the global structural organization, for degree hierarchy and higher-order neighborhood patterns, where the tree backbone conditioning provides genuine information that baselines lack.

Finally, Table~\ref{robustness_analysis} presents a robustness study for graph diffusion in application to spatio-temporal forecasting, comparing MG-Diff to the next best model on AIR-GZ (USTD). We add Gaussian noise with $\mu=0$ and $\sigma=1$ to features.
We observe that the conditional loss in MAE of MG-Diff s 4.61\% vs. 7.83\% for for USTD, while the conditional loss in MAE of MG-Diff is 3.20\% vs. 4.71\% for USTD.
This underlines the theoretical stability guarantees of Morse-guided diffusion with respect to noisy conditions (see Section~\ref{sec:mainstablemma}).

\begin{table}[H]
\centering
\caption{
Graph generation performance on tree data}
\label{tree_gen}
\begin{tabular}{lccccc}
\hline
Method$\downarrow$ & Degree $\downarrow$  & Wavelet $\downarrow$  & Spec.$\downarrow$  & Clus. & Orbit $\downarrow$ \\
\hline
DiGress & 0.000374 &  0.00567 & 0.0175 & 0 & 0.000611 \\
DeFoG & \underline{0.000349} &  \underline{0.00517} &  \underline{0.0149} &  0 & \underline{0.000335} \\
BWFlow & {0.003180} &  0.01920 &  0.0300 &  0 & 0.000840
\\ \hline
{\bf MG-Diff (Ours)} & {\bf 0.000204} & {\bf 0.00500} &  {\bf 0.0112} & 0  &{\bf 0.000175}
\\
\hline
\end{tabular}
\end{table}

\textbf{Empirical scalability of Morse construction. } To assess the scalability beyond the sensor graphs of Table \ref{tab:dataset_stats}, we construct the Morse skeleton on BSF induced subgraphs of ogbn-arxiv (created by taking subgraphs of 1k, 10, 25, 50, and all 169k nodes). Beginning from the highest-degree node, we retain all edges among the selected nodes, remove self-loops, and treat the citation graph as undirected. We track the wall-clock time for each graph size and process the RSS memory required to construct the vertex scores, the Morse function, the critical cells, and the gradient pairings (Table~\ref{tab:scalability}) We preform the Morse construction only once during model initialization, and therefore the resulting Morse structure is cached within the model and introduces no additional computation during training iterations. We note that after testing the practical overhead is minimal, and on ogbn-arxiv the Morse construction requires $\sim 2.2$ seconds and $\sim25$MB of additional memory (it incurs no per-epoch and per-iteration cost during optimization).

\begin{table}[H]
\centering
\caption{Runtime and memory consumption of the one-time Morse construction on BFS-induced subgraphs of ogbn-arxiv. Memory is the process RSS before and after construction.}
\label{tab:scalability}
\begin{tabular}{rrrrr}
\toprule
\#Nodes & \#Edges & Construction time (s) & Start memory (MB) & End memory (MB) \\
\midrule
1,000   & 1,333     & 0.0049 & 666.00 & 666.00 \\
5,000   & 13,530    & 0.0333 & 666.64 & 666.64 \\
10,000  & 44,175    & 0.0912 & 666.89 & 666.89 \\
25,000  & 353,582   & 0.6185 & 666.90 & 666.90 \\
50,000  & 669,889   & 1.1644 & 670.53 & 680.37 \\
169,343 & 1,157,799 & 2.2019 & 681.35 & 706.65 \\
\bottomrule
\end{tabular}
\end{table}

\textbf{Robustness to missing events. } In Table \ref{tab:scalability} we further evaluate MG-Diff under further increasing missing-event rates on AIR-GZ. To do so, we randomly remove $5\%, 10\%, \text{ or }20\%$ of the originally observed entries from the historical input of each test window, and instead replace the removed values with 0 in the standardized space. We see that MG-Diff consistently outperforms USTD across all corruption levels on MAE, RMSE, and CRPS; more importantly, it exhibits substantially stronger robustness as the corruption severity increases. Indeed from the clean setting to 20\% missing events, MG-Diff experiences only 18.3\%, 11.5\%, and 20.5\% relative degradations in MAE, RMSE, and CRPS respectively, whereas USTD suffers larger desegregations of of 27.1\%, 15.0\%, and 35.2\%. We see that similar trends tend to hold under both 5\% and 10\% missing rates. In other words, the gracefulness of this degradation follows directly from the model conditioning, as the Morse skeleton is derived from the graph structure and is much more stable than the raw signal. Missing entries corrupt the node observations, but the skeleton remains a compact summary of the persistent cycles and bottlenecks along which the signal propagates, so MG-Diff retains an intact structural prior even at 20\% missingness and can propagate information along the preserved skeleton to compensate for locally missing entries, whereas a topology-blind baseline such as USTD has only the increasingly corrupted observations to rely on. Again, we note this to be the empirical realization of Theorem \ref{lem:margin_preserves_skel}, where the conditioning signal is stable under input perturbations below the radius, so the topological prior is a source of robustness rather than a new sensitivity. Furthermore it explains why the advantage is most seen in CRPS (20.5\% vs.\ 35.2\% degradation at 20\%), as the skeleton encodes where alternative routes, and in that way predictive uncertainty live so the calibration of the predictive distribution benefits most from the preserved structure.

\begin{table}[H]
\centering
\caption{Robustness to missing events on AIR-GZ. Top: absolute performance under increasing missing-event rates. Bottom: relative degradation with respect to the clean setting.}
\label{tab:missing_events}
\resizebox{0.95\textwidth}{!}{
\begin{tabular}{lc|ccc|ccc}
\toprule
\multirow{2}{*}{Corruption} & \multirow{2}{*}{Rate} & \multicolumn{3}{c|}{MG-Diff (Ours)} & \multicolumn{3}{c}{USTD} \\
 & & MAE & RMSE & CRPS & MAE & RMSE & CRPS \\
\midrule
Clean          & 0\%  & {\bf 10.259} & {\bf 15.201} & {\bf 0.2452} & 10.725 & 15.713 & 0.2750 \\
Missing events & 5\%  & {\bf 10.449} & {\bf 15.458} & {\bf 0.2508} & 11.271 & 16.196 & 0.2938 \\
Missing events & 10\% & {\bf 10.815} & {\bf 15.837} & {\bf 0.2612} & 11.936 & 16.725 & 0.3174 \\
Missing events & 20\% & {\bf 12.132} & {\bf 16.945} & {\bf 0.2953} & 13.628 & 18.073 & 0.3718 \\
\midrule
\multicolumn{2}{l|}{Degradation at 5\%}  & $-1.85\%$  & $-1.69\%$  & $-2.28\%$  & $-5.09\%$  & $-3.07\%$  & $-6.84\%$ \\
\multicolumn{2}{l|}{Degradation at 10\%} & $-5.42\%$  & $-4.18\%$  & $-6.55\%$  & $-11.29\%$ & $-6.44\%$  & $-15.42\%$ \\
\multicolumn{2}{l|}{Degradation at 20\%} & $-18.26\%$ & $-11.47\%$ & $-20.46\%$ & $-27.07\%$ & $-15.02\%$ & $-35.22\%$ \\
\bottomrule
\end{tabular}}
\end{table}

\subsection{Morse guidance, a discussion. } We would like to discuss that MG-Diff targets two information flow over a fixed graph (forecasting) and graph regeneration (the process). For forecasting, signal propagates over a fixed topology, and the discrete Morse skeleton identifies the bottlenecks and topological events. Specifically, cycles (alternative routes), and junctions where flow concentrates are exactly what a forecaster must reconstruct to capture the signal distribution and its uncertainty. Therefore, MG-Diff is most helpful when the dynamics are shaped by these cycles and bottlenecks.  This can be observed through AIR-BJ (containing ring roads with high $\beta_1$) which case the largest CRPS gain, AIR-GZ (containing a dense, opaque skeleton) the smallest, and PEMS in between.

In the graph regeneration case, topology itself changes as the graph is built and so the Morse structure marks the topological events of that process (namely cycle birth/death, component merging/splinting). Conditioning helps when valid samples require finding proper events (e.g., molecular rings, meshes, or utility networks) and local edge statistics alone drift away from them. MG-Diff is least helpful for forecasting on tree-like graphs (where $\beta_1 \approx 0$ see Table \ref{tree_gen}), where $E^\text{crit}$ is empty and MoMoE collapses to one expert and MG-Diff reduces to standard conditional diffusion. In that case, a pure hierarchy is our weakest case, as hierarchy lives in the standard paired-edge forest rather than in $\beta_1$, and the gain appears hen the hierarchy coexists with cycles. Furthermore, the salience of cycle structure is quantified by $\beta_1$ and can be controlled by fixing a spanning tree on $N$ nodes ($\beta_1 = 0)$ and adding $r$ random chord edges yields $\beta_1 = r$ (since each cord adds exactly one independent cycle). 

\begin{table}[H]
\centering
\caption{
Graph generation performance on tree data}
\label{tree_gen}
\begin{tabular}{lccccc}
\hline
Method$\downarrow$ & Degree $\downarrow$  & Wavelet $\downarrow$  & Spec.$\downarrow$  & Clus. & Orbit $\downarrow$ \\
\hline
DiGress & 0.000374 &  0.00567 & 0.0175 & 0 & 0.000611 \\
DeFoG & \underline{0.000349} &  \underline{0.00517} &  \underline{0.0149} &  0 & \underline{0.000335} \\
BWFlow & {0.003180} &  0.01920 &  0.0300 &  0 & 0.000840
\\ \hline
{\bf MG-Diff (Ours)} & {\bf 0.000204} & {\bf 0.00500} &  {\bf 0.0112} & 0  &{\bf 0.000175}
\\
\hline
\end{tabular}
\end{table}

\section{Limitations}
\label{app: limitations}

We find that one of the main current limitations of the proposed ideas on the Morse skeleton is the lack of formal statistical inference. However, this is a fundamental question at the interface of low dimensional topology and mathematical statistics, and we leave it a standalone future research direction.

\section{Broader Impacts}\label{app: broader_impacts}
We believe that the proposed methodology will open a way for more principled topological uncertainty quantification, thereby leading to positive societal impacts in application requiring risk assessment. We do not envision any negative societal impacts. 

\section{Usage of LLM}\label{app: use_llm}
Large language models (LLMs) were only used for rephrasing sentences and correcting grammar, and were not used for designing experiments, generating results, or contributing to the scientific content of the paper.

\end{document}